\documentclass{article}

\usepackage{microtype}
\usepackage{graphicx}
\usepackage{subfigure}
\usepackage{booktabs} %

\usepackage{hyperref}

\usepackage[accepted]{icml2023}

\usepackage{amsmath}
\usepackage{amssymb}
\usepackage{mathtools}
\usepackage{amsthm}

\usepackage[capitalize,noabbrev]{cleveref}

\usepackage[all]{xy}

\theoremstyle{plain}
\newtheorem{theorem}{Theorem}[section]
\newtheorem{proposition}[theorem]{Proposition}
\newtheorem{lemma}[theorem]{Lemma}
\newtheorem{corollary}[theorem]{Corollary}
\theoremstyle{definition}
\newtheorem{definition}[theorem]{Definition}
\newtheorem{assumption}[theorem]{Assumption}
\theoremstyle{remark}
\newtheorem{remark}[theorem]{Remark}
\newtheorem{example}[theorem]{Example}

\newcommand{\er}{\hat{\mathcal{R}}}
\newcommand{\upperbox}{\overline{\dim}_B}
\newcommand{\lowerbox}{\underline{\dim}_{B}}

\newcommand{\dimph}{\dim_{\text{PH}^0}}
\newcommand{\rad}{\text{\normalfont \textbf{Rad}}}
\newcommand{\closed}{\text{\normalfont \textbf{CL}}}

\newcommand{\risk}{\mathcal{R}}
\newcommand{\der}{\text{d}}
\newcommand{\kl}[2]{\textbf{KL}(#1 || #2)}
\newcommand{\rips}{\textbf{Rips}}

\usepackage{dsfont}
\usepackage{enumitem}

\usepackage[textsize=tiny]{todonotes}

\icmltitlerunning{Generalization Bounds using Data-Dependent Fractal Dimensions}

\begin{document}

\twocolumn[
\icmltitle{Generalization Bounds using Data-Dependent Fractal Dimensions}

\icmlsetsymbol{equal}{*}

\begin{icmlauthorlist}
\icmlauthor{Benjamin Dupuis}{inria,ens,psl}
\icmlauthor{George Deligiannidis}{oxford,alanturing}
\icmlauthor{Umut \c{S}im\c{s}ekli}{inria,ens,psl,cnrs}
\end{icmlauthorlist}

\icmlaffiliation{inria}{Inria}
\icmlaffiliation{ens}{Ecole Normale Supérieure, Paris, France}
\icmlaffiliation{cnrs}{CNRS}
\icmlaffiliation{psl}{PSL Research University, Paris, France}
\icmlaffiliation{alanturing}{The Alan Turing Institute, London, UK}
\icmlaffiliation{oxford}{Department of Statistics, 
       University of Oxford, Oxford, UK}

\icmlcorrespondingauthor{Benjamin Dupuis}{benjamin.dupuis@inria.fr}
\icmlcorrespondingauthor{Umut \c{S}im\c{s}ekli}{umut.simsekli@inria.fr}

\icmlkeywords{Machine Learning, ICML}

\vskip 0.3in
]

\printAffiliationsAndNotice{\icmlEqualContribution} %

\begin{abstract}

Providing generalization guarantees for modern neural networks has been a crucial task in statistical learning. Recently, several studies have attempted to analyze the generalization error in such settings by using tools from fractal geometry. While these works have successfully introduced new mathematical tools to apprehend generalization, they heavily rely on a Lipschitz continuity assumption, which in general does not hold for neural networks and might make the bounds vacuous. In this work, we address this issue and prove fractal geometry-based generalization bounds \emph{without} requiring any Lipschitz assumption. To achieve this goal, we build up on a classical covering argument in learning theory and introduce a \emph{data-dependent fractal dimension}. Despite introducing a significant amount of technical complications, this new notion lets us control the generalization error (over either fixed or random hypothesis spaces) along with certain mutual information (MI) terms. To provide a clearer interpretation to the newly introduced MI terms, as a next step, we introduce a notion of `geometric stability' and link our bounds to the prior art. Finally, we make a rigorous connection between the proposed data-dependent dimension and topological data analysis tools, which then enables us to compute the dimension in a numerically efficient way. We support our theory with experiments conducted on various settings.

\end{abstract}

\section{Introduction}

Understanding the generalization properties of modern neural networks has been one of the major challenges in statistical learning theory over the last decade. 
In a classical supervised learning setting, this task boils down to understanding the so-called \emph{generalization error}, which arises from the population risk minimization problem, given as follows: 
\begin{align*}
    \min_{w\in \mathds{R}^d} \Bigl\{ \mathcal{R}(w) := \hspace{-2pt} \underset{z\sim \mu_z}{\mathds{E}}  [\ell(w,z)] := \hspace{-2pt}\underset{(x,y)\sim \mu_z}{\mathds{E}}[\mathcal{L}(h_w(x),y)] \Bigr\},
\end{align*}
where $x \in \mathcal{X}$ denotes the features, $y \in \mathcal{Y}$ denotes the labels, $\mathcal{Z} = \mathcal{X} \times \mathcal{Y}$ denotes the data space endowed with an unknown probability measure $\mu_z$, referred to as the data distribution, $h_w : \mathcal{X} \longrightarrow \mathcal{Y}$ denotes a parametric predictor with $w \in \mathds{R}^d$ being its parameter vector, $\mathcal{L}: \mathcal{Y} \times \mathcal{Y} \longrightarrow \mathds{R}$ denotes the loss function, and $\ell$ is the composition of the loss and the predictor, i.e. $\ell(w,z) = \ell(w,(x,y)) = \mathcal{L}(h_w(x),y)$, which will also be referred to as `loss', with a slight abuse of notation.
As $\mu_z$ is unknown, in practice one resorts to the minimization of the empirical risk, given as follows: 
\begin{equation}
\label{eq:empirical_risk}
\er_S (w) := \frac{1}{n} \sum_{i=1}^n \ell(w, z_i),
\end{equation}
where $S := (z_i)_{1\leq i \leq n} \sim \mu_z^{\otimes n}$ is a set of independent and identically distributed (i.i.d.) data points. 
Then, our goal is to bound the worst-case generalization error that is defined as the gap between the population and empirical risk over a (potentially random) hypothesis set $\mathcal{W} \subset \mathds{R}^d$:
\begin{equation}
\label{eqn:gen_err_det}
\mathcal{G}(S) := \sup_{w \in \mathcal{W}}  \bigl( \mathcal{R}(w) - \er_S(w) \bigr).
\end{equation}
In the context of neural networks, one peculiar observation has been that, even when a network contains millions of parameters (i.e., $d \gg 1$), it might still generalize well \cite{zhang_understanding_2017}, despite accepted wisdom suggesting that typically $\mathcal{G} \approx \sqrt{d/n}$ \cite{anthony_neural_1999}.

To provide a theoretical understanding for this behavior, several directions have been explored, such as compression-based approaches \cite{arora_stronger_2018, suzuki_compression_2020, barsbey_heavy_2021} and the approaches focusing on the double-descent phenomenon \cite{belkin_reconciling_2019, nakkiran_deep_2019}. Recently,   
there has been an increasing interest in examining the role of `algorithm dynamics' on this phenomenon. In particular, it has been illustrated that, in the case where a stochastic optimization algorithm is used for minimizing \eqref{eq:empirical_risk}, the optimization trajectories can exhibit a fractal structure \cite{simsekli_hausdorff_2021, camuto_fractal_2021, birdal_intrinsic_2021, hodgkinson_generalization_2022}.
Under the assumption that $\ell$ is uniformly bounded by some $B$ and uniformly $L$-Lipschitz with respect to $w$, their results informally implies the following: with probability $1-\zeta$, we have that 
\begin{equation}
\label{eq:informal_previous_bounds}
        \mathcal{G} \lesssim LB \sqrt{\frac{\bar{d}(\mathcal{W}) + I_\infty(\mathcal{W}, S) + \log(1/\zeta)}{n}},
\end{equation}
where $\mathcal{W}$ is a \emph{data-dependent hypothesis set}, which is provided by the learning algorithm, $\bar{d}(\mathcal{W})$ is a notion of \emph{fractal dimension} of $\mathcal{W}$, and $I_\infty(\mathcal{W}, S)$ denotes the \emph{total mutual information} between the data $S$ and the hypothesis set $\mathcal{W}$. These notions will be formally defined in Section~\ref{sec:tb}\footnote{In \cite{simsekli_hausdorff_2021, camuto_fractal_2021} the bound is logarithmic in $L$. \cite{simsekli_hausdorff_2021} only requires sub-gaussian losses while \cite{camuto_fractal_2021} requires sub-exponential losses. Their common points is to require a Lipschitz assumption.}. In the case where the intrinsic dimension $\bar{d}(\mathcal{W})$ is significantly smaller than the ambient dimension $d$ (which has been empirically illustrated in \citep{simsekli_hausdorff_2021, birdal_intrinsic_2021}), the bound in \eqref{eq:informal_previous_bounds} provides an explanation on why overparameterized networks might not overfit in practice.

While these bounds have brought a new perspective on understanding generalization, they also possess an important drawback, that is they all rely on a \emph{uniform Lipschitz continuity} assumption on $\ell$ (with respect to the parameters), which is too strict to hold for deep learning models. While it is clear that we cannot expect Lipschitz continuity of a neural network when the parameter space is unbounded,  \citet{herrera_estimating_2020} showed that, even for the bounded domains, the Lipschitz constants of fully connected networks are typically polynomial in the width, exponential in depth 
which may be excessively large in practical settings; hence might make the bounds vacuous.

The Lipschitz assumption is required in \citep{simsekli_hausdorff_2021,birdal_intrinsic_2021,camuto_fractal_2021} as it enables the use of a fractal dimension defined through \emph{the Euclidean distance} on the hypothesis set $\mathcal{W}$ (which is independent of the data). Hence, another downside of the Lipschitz assumption is that the Euclidean distance-based dimension  unfortunately ignores certain important components of the learning problem, such as the how the loss $\ell$ behaves over $\mathcal{W}$.
As shown in \cite{jiang_fantastic_2019} in the case sharpness measures \cite{keskar_large-batch_2017}, which measure the sensitivity of the empirical risk around local minima and correlate well with generalization, the data-dependence may improve the ability of a complexity measure to explain generalization.

\subsection{Contributions}
\label{subsection:contributions}

In this study, our main goal is to address the aforementioned issues by proving fractal geometric generalization bounds without requiring any Lipschitz assumptions.  Inspired by a classical approach for bounding the Rademacher complexity (defined formally in Appendix~\ref{sec:radcpx}), we achieve this goal by making use of a \emph{data-dependent} pseudo-metric on the hypothesis set $\mathcal{W}$.  
Our contributions are as follows:
\begin{itemize}[noitemsep,topsep=0pt,leftmargin=.11in]
    \item We prove bounds (Theorems \ref{main_result_fixed_hypothesis_space} and \ref{main_result_HP_bound_convering_MI}) on the worst-case generalization error of the following form:
    \begin{equation}
         \mathcal{G} \lesssim B \sqrt{\frac{\bar{d}_S(\mathcal{W}) + I + \log(1/\zeta)}{n}},
    \end{equation}
    where $\bar{d}_S$ denotes a notion of \emph{data-dependent} fractal dimension and $I$ is a  (total) mutual information term (see Section \ref{subsection:information_theory}). As opposed to prior work, this bound does not require any Lipschitz assumption and therefore applies to more general settings. However, this improvement comes with the expense of having a more complicated mutual information term compared to the one in \eqref{eq:informal_previous_bounds}.
    \item To provide more understanding about the newly introduced mutual information term $I$ and highlight its links to prior work, we introduce a notion of `geometric stability' and without requiring Lipschitz continuity, we prove an almost identical bound to the one in Equation \eqref{eq:informal_previous_bounds} (with a potentially slightly worse rate in $n$). 
    \item In order to be able to compute the data-dependent fractal dimension, we build on \cite{birdal_intrinsic_2021} and prove that our dimension can also be computed by using numerically efficient topological data analysis tools \cite{carlsson_topological_2014, perez-fernandez_characterizing_2021}.
\end{itemize}

Finally, we illustrate our bounds on experiments using various neural networks. In addition to not requiring Lipschitz continuity, we show that our data-dependent dimension provides improved correlations with the actual generalization error. 
All the proofs are provided in the Appendix. Python code for numerical experiments is available at \url{https://github.com/benjiDupuis/data_dependent_dimensions}.

\section{Technical Background}
\label{sec:tb}
\subsection{Learning framework}
\label{subsection:learning_algorithm}

We formalize the learning algorithm as follows. The data (probability) space is denoted by $(\mathcal{Z}, \mathcal{F}, \mu_z)$\footnote{For technical measure-theoretic reasons (see Section \ref{subsection:technical_lemmas}), it is best to assume $\mathcal{Z} \subseteq \mathds{R}^N$ for some $N$.}. A learning algorithm $\mathcal{A}$ is a map generating a random closed set $\mathcal{W}_{S,U}$ (see \citep[Definition $1.1.1$]{molchanov_theory_2017}) from the data $S$ and an external random variable $U$ accounting for the randomness of the learning algorithm. The external randomness $U$ takes values in some probability space $(\Omega_U, \mathcal{F}_U, \mu_u)$, which means that $U$ is $\mathcal{F}_U$-measurable and has distribution $\mu_u$. Moreover, we assume that $U$ is independent of $S$. Therefore if we write $\closed(\mathds{R}^d)$ for the class of closed sets of $\mathds{R}^d$ endowed with the Effrös $\sigma$-algebra, as in \cite{molchanov_theory_2017}, the algorithm will be thought as a measurable map:
\begin{equation}
\label{learning_algorithm}
\mathcal{A} :   \bigcup_{n=0}^\infty \mathcal{Z}^n \times \Omega_U \to \closed(\mathds{R}^d) \ni \mathcal{W}_{S,U}.
\end{equation}
This formulation encompasses several settings, such as the following two examples.
\begin{example}
    Given a continuous time process of the form $\der W_t = -\nabla f(W_t) \der t + \Sigma(W_t) \der X_t$ where $X_t$ is typically a Brownian motion or a L\'{e}vy process, as considered in various studies \cite{mandt_variational_2016, chaudhari_stochastic_2018, hu_diffusion_2018, jastrzebski_three_2018,  simsekli_hausdorff_2021}, we can view $\mathcal{W}_{S,U}$ as the set of points of the trajectory $\{ W_t, ~ t \in [0,T] \}$, where $U$ accounts for randomness coming from quantities defining the model like $X_t$.
\end{example}

\begin{example}
    Consider a neural network $h_w (\cdot)$ and denote the output of the stochastic gradient descent (SGD) iterates by $A(x_0, S, U)$, where $U$ accounts for random batch indices and $x_0$ is the initialization. This induces a learning algorithm $\mathcal{W}_{S,U} = \bigcup_{x_0 \in X_0}\{ A(x_0, S, U) \} $, which is closed if $X_0$ is compact under a continuity assumption on $A$. 
\end{example}

\subsection{Information theoretic quantities}
\label{subsection:information_theory}
Recently, one popular approach to prove generalization bounds has been based on information theory. In this context, \citet{xu_information-theoretic_2017, russo_how_2019} proved particularly interesting generalization bounds in terms of the \emph{mutual information} between input and output of the model. Other authors refined this argument in various settings \cite{pensia_generalization_2018, negrea_information-theoretic_2019, steinke_reasoning_2020, harutyunyan_information-theoretic_2021} while \citet{asadi_chaining_2019} combined mutual information and chaining to tighten the bounds. In our work we will use the total mutual information to specify the dependence between the data and the fractal properties of the hypothesis set.

The classic mutual information between two random elements $X$ and $Y$ is defined in terms of the Kullback-Leibler (KL) divergence $I(X,Y) := \text{KL}(\mathds{P}_{X,Y}||\mathds{P}_X \otimes \mathds{P}_Y)$. It is well known that mutual information can be used as a decoupling tool \cite{xu_information-theoretic_2017}; yet, in our setup, we will need to consider the \emph{total mutual information}, which is defined as follows, $\mathds{P}_X$ denoting the law of the variable $X$: 
\begin{equation}
\label{eq:total_mutual_information_def}
I_\infty(X,Y) := \log \bigg( \sup_{B} \frac{\mathds{P}_{X,Y}(B)}{\mathds{P}_X \otimes \mathds{P}_Y(B)} \bigg).
\end{equation}
\citet{hodgkinson_generalization_2022} used total mutual information to decouple the data and the optimization trajectory, they defined it as a limit of $\alpha$-mutual information, which is equivalent, see \citep[Theorem $6$]{van_erven_renyi_2014}.

\subsection{The upper box-counting dimension}
\label{sec:upperbox}

Fractal geometry \cite{falconer_fractal_2014}  and dimension theory have been successful tools in the study of dynamical systems and stochastic processes \cite{pesin_dimension_1997, xiao_random_2004}. 
In our setting, we will be interested in the \emph{upper box-counting dimension}, defined as follows. Given a (pseudo-)metric space $(X, \rho)$ and $\delta > 0$, we first define the closed $\delta$-ball centered in $x \in X$ by $B_\delta^\rho(x) = \{ y \in X,~ \rho(x,y) \leq \delta \}$ and a \emph{minimal covering} $N_\delta^\rho(X)$ as a minimal set of points of $X$ such that
$X\subset \bigcup_{y \in N_\delta^\rho(X)}B_\delta^\rho(y)$. We can then define the upper box-counting dimension as follows:
\begin{equation}
\label{upperbox_dimension}
\upperbox^\rho(X) := \limsup_{\delta \to 0} \frac{\log |N_\delta^\rho(X)| }{\log(1/\delta)},
\end{equation}
where $|A|$ denotes the cardinality of a set $A$.

Under the Lipschitz loss assumption, \citet{simsekli_hausdorff_2021, birdal_intrinsic_2021, camuto_fractal_2021, hodgkinson_generalization_2022} related different kinds of fractal dimensions, computed with \emph{the Euclidean distance} $\rho(w,w') = \text{Eucl}(w,w') := \Vert w - w' \Vert_2$, to the generalization error. Our approach in this study will be based on using a \emph{data-dependent} pseudo-metric $\rho$, which will enable us to remove the Lipschitz assumption.

\section{Main Results}
\label{section:main_results}

In this section we present our main theoretical results; our aim is to relate the worst-case generalization error of \eqref{learning_algorithm} with the upper box-counting dimension computed based on the following random pseudo-metric:
\begin{equation}
        \label{losses_based_pseudo_metric}
        \rho_S(w,w') := \frac{1}{n} \sum_{i=1}^n | \ell(w,z_i) - \ell(w', z_i)|.
\end{equation}

We insist on the fact that it is only a pseudo-metric because in practice we can have $\rho_S(w,w') = 0$ while $w \neq w'$, for example due to the internal symmetries of a neural network.

\subsection{Main assumptions}

A key component of our work is that we do not use any Lipschitz assumption on $\ell$ as for example in \cite{simsekli_hausdorff_2021, hodgkinson_generalization_2022}. The only regularity assumption we impose is the following:

\begin{assumption}
    \label{bounded_continuous_assumption}
    The loss $\ell : \mathds{R}^d \times \mathcal{Z} \longrightarrow \mathds{R}$ is continuous in both variables and uniformly bounded by some $B > 0$.
\end{assumption}

We note that the box-counting dimension with respect to the pseudo-metric \eqref{losses_based_pseudo_metric} involves minimal coverings, which we denote $N_\delta^{\rho_S}(A)$ for some set $A$. The boundedness assumption is essential to ensure that minimal coverings are finite and $\upperbox^{\rho_S}$ is also finite. Therefore our boundedness assumption cannot be replaced with a subgaussian assumption, as opposed to \cite{simsekli_hausdorff_2021}. 

We also assume that we can construct minimal coverings which are random closed (finite) sets in the sense of \citep[Definition 1.1.1]{molchanov_theory_2017}; this is made precise with the following assumption:

\begin{assumption}
    \label{coverings_measurability_assumption}
    Let $C \subset \mathds{R}^d$ be any closed set, $\delta>0$, $S \in \mathcal{Z}^n$ and $S' \in \mathcal{Z}^m$. We can construct minimal $\delta$-coverings $N_\delta^{\rho_{S'}}(C \cap \mathcal{W}_{S,U})$ which are random finite sets with respect to the product $\sigma$-algebra $\mathcal{F}^{\otimes n} \otimes \mathcal{F}^{\otimes n} \otimes \mathcal{F}_U$ (measurability with respect to $S,S',U$). 
    We denote by $\mathcal{N}_{\delta}(C \cap \mathcal{W}_{S,U})$ the family of all those random minimal coverings.
\end{assumption}

\begin{remark}
    Assumption \ref{coverings_measurability_assumption} essentially enables us to avoid technical measurability complications.  
    The main message is that we assume that we are able to construct ``measurable coverings". This assumption can be cast as a \emph{selection} property; indeed for each realization of $(S,S',U)$ there may be a wide range of possible minimal coverings: what we assume is that we can select one of them for each $(S,S',U)$ so that the obtained random set is measurable.
\end{remark}

Assumption \ref{coverings_measurability_assumption} is actually much stronger than what is needed to make our results valid. Indeed, we are able to show that, under the assumption that the mapping $\mathcal{A}$ of \eqref{learning_algorithm} is measurable with respect to the Effrös $\sigma$-algebra, we can construct coverings $(S,U) \longmapsto N_\delta^{\rho_S}(\mathcal{W}_{S,U})$ which are measurable and, if not minimal, yield the same upper box-counting dimension as minimal coverings, when computing the limit \eqref{upperbox_dimension}. To avoid too much technical considerations, this discussion is presented in Appendix \ref{subsection:technical_lemmas}, as an additional technical contribution.

As the upper box-counting dimension \eqref{upperbox_dimension} may be written as a countable limit, the measurability assumption \ref{coverings_measurability_assumption} also implies that $\upperbox^{\rho_S}(\mathcal{W}_{S,U})$ is a random variable.
Continuity of the loss in Assumption~\ref{bounded_continuous_assumption} is there for technical purposes, e.g., to make quantities of the form $\sup_{w \in \mathcal{W}_{S,U}} \big( \mathcal{R}(w) - \er_S(w) \big)$ well-defined random variables (see \citep[Theorem 
1.3.28]{molchanov_theory_2017} and Section~\ref{subsection:technical_lemmas} in the Appendix).

\subsection{Warm-up: fixed hypothesis spaces}
\label{sec:fixed_w}

In this subsection we fix a \emph{deterministic} closed set $\mathcal{W} \subset \mathds{R}^d$ and consider its upper box-counting dimension with respect to the data-dependent pseudo-metric \eqref{losses_based_pseudo_metric}, which we denote by $d(S) := \upperbox^{\rho_S} (\mathcal{W})$. Our goal is to bound the worst-case generalization error as defined in \eqref{eqn:gen_err_det}.
The next theorem is an extension of the classical covering bounds of Rademacher complexity \cite{barlett_rademacher_2002, rebeschini_algorithmic_2020}.
\begin{theorem}
\label{main_result_fixed_hypothesis_space}
For all $\epsilon, \gamma, \eta > 0$ and $n \in \mathds{N}_+$ there exists $\delta_{n,\gamma,\epsilon} > 0$ such that with probability at least $1 - 2\eta - \gamma$ under $\mu_z^{\otimes n}$, for all $\delta < \delta_{n,\gamma,\epsilon}$ we have:
$$
\mathcal{G}(S) \leq  2B \sqrt{\frac{4 (d(S) + \epsilon) \log(1/\delta) + 9 \log(1/\eta)}{n}} + 2\delta.
$$  
\end{theorem}

Theorem \ref{main_result_fixed_hypothesis_space} is therefore similar to \citep[Theorem 1]{simsekli_hausdorff_2021}, which used a fractal dimension based on the Euclidean distance on $\mathds{R}^d$, $\Vert w - w' \Vert_2$ and a fixed hypothesis space. The improvement here is in the absence of Lipschitz assumption. Moreover, as detailed in Appendix \ref{sec:Lipschitz_case}, in case of a Lipschitz $\ell$, we recover, from our proofs, bounds in term of the upper box-counting dimension based on the Euclidean distance on the hypothesis set, which is the one used in prior works \citep{simsekli_hausdorff_2021}. Therefore, our methods based on a data-dependent fractal dimension are more general than previous studies.

However, Theorem~\ref{main_result_fixed_hypothesis_space} might not be sufficiently satisfying. The proof involves techniques
that do not hold in the case of random hypothesis spaces, an issue which we address in the next subsection.

\subsection{Random hypothesis spaces}
\label{subsection:random_hypothesis_set}

Theorem \ref{main_result_fixed_hypothesis_space} is interesting because it gives a bound similar to \cite{simsekli_hausdorff_2021} in the case of a fixed hypothesis set but with a new notion of data dependent intrinsic dimension. Now we come to the case where the hypothesis set $\mathcal{W}_{S,U}$ generated by the learning algorithm \eqref{learning_algorithm} is a random set. 

For notational purposes let us denote the upper box-counting dimension of $\mathcal{W}_{S,U}$ induced by pseudo-metric \eqref{losses_based_pseudo_metric}  by $d(S, U) := \upperbox^{\rho_S} (\mathcal{W}_{S,U})$, and denote the worst-case generalization error by
\begin{equation}
    \label{worst_case_generalization_S_U}
    \mathcal{G}(S,U) := \sup_{w \in \mathcal{W}_{S,U}}(\mathcal{R}(w) - \er_S(w)).
\end{equation}
Here again, note that $d(S,U)$ can be written as a countable limit of random variables and therefore defines a random variable thanks to Assumption 
\ref{coverings_measurability_assumption}.

The main difficulty here is that classical arguments based on the Rademacher complexity cannot be applied in this case as $\mathcal{W}_{S,U}$ depends on the data sample $S$. Hence, to be able to develop a covering argument, we first cover the set $\mathcal{W}_{S,U}$ by using the pseudo-metric $\rho_S$ (cf.\ Section~\ref{sec:upperbox}) and rely on the following decomposition: for any $\delta >0$ and $w' \in N_\delta^{\rho_S}(\mathcal{W}_{S,U})$ we have that 
\begin{align*}
    \mathcal{R}(w) - \hat{\mathcal{R}}_S(w) &\leq \mathcal{R}\left(w^{\prime}\right) - \hat{\mathcal{R}}_S\left(w^{\prime}\right) \\
    & \phantom{aaa}+ |\hat{\mathcal{R}}_S(w) 
     -\hat{\mathcal{R}}_S\left(w^{\prime}\right) | \\
     &\phantom{aaa}+ 
     \left|\mathcal{R}(w)-\mathcal{R}\left(w^{\prime}\right)\right|.
\end{align*}
In the above inequality, the first term can be controlled by standard techniques as $w'$ lives in a finite set $N_\delta^{\rho_S}(\mathcal{W}_{S,U})$ and the second term is trivially less than $\delta$ by the definition of coverings. However, the last term cannot be bounded in an obvious way.
To overcome this issue 
we introduce `approximate level-sets' of the population risk, defined as follows\footnote{As $U$ is independent of $S$, we drop the dependence on it to ease the notation.} for some $K \in \mathds{N}_+$:
   \begin{equation}
   \label{approximate_level_sets}
   R_S^j := \mathcal{W}_{S,U} \cap \mathcal{R}^{-1} \bigg( \bigg[\frac{jB}{K}, \frac{(j+1)B}{K} \bigg] \bigg),
   \end{equation}
   where $j = 0,\dots, K-1$ and $\mathcal{R}^{-1}$ denotes the inverse image of $\mathcal{R}$. 
Let $N_{\delta,j}$ collect the centers of a minimal $\delta$-cover of $R_S^j$ 
relatively to $\rho_S$\footnote{Assumption \ref{coverings_measurability_assumption} extends to the randomness of those sets $N_{\delta,j}$.}. The next theorem provides a generalization bound for random hypothesis sets.
\begin{theorem}
\label{main_result_HP_bound_convering_MI}
Let us set $K = \lfloor \sqrt{n} \rfloor$ and define $I_{n,\delta} := \max_{0 \leq j \leq \lfloor \sqrt{n} \rfloor} I_\infty (S, N_{\delta,j})$.
Then, for all $\epsilon, \gamma, \eta > 0$, there exists $\delta_{n,\gamma,\epsilon} > 0$ such that with probability at least $1 - \eta - \gamma$ under $\mu_z^{\otimes n} \otimes \mu_u$, for all $\delta < \delta_{n,\gamma,\epsilon}$ we have: 
$$
\begin{aligned}
    \mathcal{G}(S,U) \leq & \frac{B}{\sqrt{n} - 1}  +  \bigg\{ \frac{2B^2}{n} \bigg( (d(S,U) + \epsilon  ) \log(2/\delta) \\ &+\log(\sqrt{n}/\eta) + I_{n,\delta} 
\bigg) \bigg\}^{\frac{1}{2}} + \delta.
\end{aligned}
$$

\end{theorem}

This theorem gives us a bound in the general case similar to \citep[Theorem 2]{simsekli_hausdorff_2021}, yet without requiring Lipschitz continuity.

Moreover, also similar to \cite{simsekli_hausdorff_2021,hodgkinson_generalization_2022},   Theorem~\ref{main_result_HP_bound_convering_MI} introduces a mutual information term
$I_{n,\delta}$, which intuitively measures the local mutual dependence between the data and the coverings. This can be seen as how the data influences the `local fractal behavior' of the the hypothesis set. On the other hand, despite the similarity to prior work, $I_{n,\delta}$ might be more complex because the dependence of $N_{\delta,j}$ on $S$ comes both from the pseudo-metric $\rho_S$ and the hypothesis set $\mathcal{W}_{S,U}$. In the next subsection, we show that we can modify our theory in a way that it involves the simpler mutual information term proposed in \citep{hodgkinson_generalization_2022}.

\subsection{Geometric stability and mutual information}
\label{subsection:covering_stability}

The intricate dependence between $N_{\delta, j}$ and $S$ makes it hard to express the term $I_{n,\delta}$ in Theorem \ref{main_result_HP_bound_convering_MI} or bound it with standard methods (e.g. data-processing inequality). In this subsection, we introduce a notion of `geometric stability' to obtain a more interpretable bound.

Algorithmic stability is a key notion in learning theory and has been shown to imply good generalization properties \cite{bousquet_stability_2002, bousquet_sharper_2020, chandramoorthy_generalization_2022}. Recently, \citet{foster_hypothesis_2020} extended this notion to the stability of \emph{hypothesis sets}, and proposed a notion of stability as a bound on the Hausdorff distance between the hypothesis sets generated by neighboring datasets. In our setting this would mean that there exists some $\bar{\beta} > 0$ such that for all $S,S' \in \mathcal{Z}^n$ differing only by one element, for all $u \in \mathcal{U}$, we have:
    \begin{equation}
    \label{eq:foster_uniform_stability}
    \begin{aligned}  
        \forall w \in \mathcal{W}_{S,U},& ~\exists w' \in \mathcal{W}_{S',U}, ~\forall z \in \mathcal{Z} , \\ &~\vert \ell(w,z) - \ell(w',z) \vert \leq \bar{\beta}.
    \end{aligned}
    \end{equation}
\citet{foster_hypothesis_2020} argue that in many situations $\bar{\beta} = \mathcal{O}(1/n)$.

Inspired by \citep{foster_hypothesis_2020}, we introduce a stability notion, coined \emph{geometric stability}, on the minimal coverings that will allow us to reduce the statistical dependence between the dataset $S \sim \mu_z^{\otimes n}$ and those coverings. 

To state our stability notion, we need to refine our definition of coverings. Let $A \subset \mathds{R}^d$ be some closed set, potentially random. For any $\delta > 0$ we define $N_\delta(A, S)$ to be the random minimal coverings of $A$ by closed $\delta$-balls under pseudo-metric $\rho_S$ \eqref{losses_based_pseudo_metric} with centers in $A$. Note that the dependence in $S$ in $N_\delta(A,S)$ only refers to the \emph{pseudo-metric} used. In addition to Assumption \ref{coverings_measurability_assumption} which states that we can make such a selection of $N_\delta(A,S)$, making it a well-defined random set, we add the fact that this selection can be made regular enough in the following sense.

\begin{definition}
    \label{coverings_stability}
    We say that a set $A$ is geometrically stable if there exist some $\beta > 0$ and $\alpha > 0$ such that for $\delta$ small enough we can find a random covering $S \mapsto N_\delta (A,S)$ such that for all $S\in \mathcal{Z}^n$ and $S' \in \mathcal{Z}^{n-1}$ such that $S' = S \setminus \{z_i\}$ for some $i$, then $N_\delta(A,S)$ and $N_\delta(A,S')$ are within $\beta/n^{\alpha}$ data-dependent Hausdorff distance, by which we mean: 
    \begin{equation}
    \label{eq:coverings_stability}
    \begin{aligned}
    \forall w \in N_\delta(S, A), &~\exists w' \in N_\delta(S', A), \\ & ~\sup_{z\in \mathcal{Z}} |\ell(w,z) - \ell(w',z)| \leq \frac{\beta}{n^{\alpha}}.
    \end{aligned}
    \end{equation}
\end{definition}

Based on this definition, we assume the following condition. 

\begin{assumption}
    \label{local_covering_stability_assumption}
    Let $K \in \mathds{N}_+$. There exists $\alpha \in (0,3/2)$ and $ \beta > 0$ (potentially depending on $K$) such that all sets of the form $\mathcal{W}_{S,U}\cap \mathcal{R}^{-1} \big( \big[\frac{jB}{K}, \frac{(j+1)B}{K} \big] \big)$ 
    are geometrically stable with parameters $(\alpha, \beta)$.
\end{assumption}

Assumption \ref{local_covering_stability_assumption} essentially imposes a \emph{local} regularity condition on the fractal behavior of $\mathcal{W}_{S,U}$ with respect to the pseudo-metric $\rho_S$. Intuitively it means that we can select a regular enough covering among all coverings.
Note that the geometric stability is a condition on how the coverings vary with respect to the pseudo-metric, which is fundamentally different than \cite{foster_hypothesis_2020}.

The next theorem provides a generalization bound under the geometric stability condition. 
\begin{theorem}
\label{main_result_hp_bound_with_coverings_stability}
    Let $d(S,U)$ and $\mathcal{G}(S,U)$ be as in Theorem \ref{main_result_HP_bound_convering_MI} and further define $I := I_\infty (S, \mathcal{W}_{S,U})$. Suppose that Assumption \ref{local_covering_stability_assumption} holds. Then there exists a constant $n_\alpha, \delta_{\gamma, \epsilon, n} > 0$ such that for all $n\geq n_\alpha$, with probability $1 - \gamma - \eta$, and for all $\delta \leq \delta_{\gamma, \epsilon, n}$, the following inequality holds:
    $$
    \begin{aligned}
    \mathcal{G}(S,U) \leq & \frac{3B + 2\beta}{n^{\alpha/3}} + \bigg\{ \frac{B^2}{2n^{\frac{2\alpha}{3}}} \bigg(  \big( \epsilon + d(S, U) \big) \log(4/\delta) \\ &+ \log(1/\eta) + \log(n)+I \bigg) \bigg\}^{\frac{1}{2}}  + \delta.
 \end{aligned}
    $$
    Moreover, we have that $n_\alpha = \max \{ 2^{\frac{3}{2\alpha}}, 2^{1+\frac{3}{3 - 2\alpha}} \}$.
\end{theorem}

While Assumption \ref{local_covering_stability_assumption} might be restrictive, our goal here is to highlight how such geometric regularity can help us deal with the statistical dependence between the data and the hypothesis set.

Note that the mutual information term appearing in Theorem~\ref{main_result_hp_bound_with_coverings_stability}  is much more interpretable compared to the corresponding terms in Theorem~\ref{main_result_HP_bound_convering_MI}, and has the exact same form as the term presented in \cite{hodgkinson_generalization_2022}. 

We also note that this way of controlling the dependence between the data and the hypothesis set comes at the expense of potentially losing in the convergence rate of our bound. More precisely, for a stability index of $\alpha$, we get a convergence rate of $n^{-\alpha/3}$. By examining the value of constant $n_\alpha$ in Theorem~\ref{main_result_hp_bound_with_coverings_stability}, we observe that getting closer to an optimal rate ($\alpha \approx \frac{3}{2}$) implies a larger $n_\alpha$, rendering our bound asymptotic.

\subsection{One step toward lower bounds}

As an additional theoretical result, we present an attempt to prove a lower bound involving the introduced data-dependent fractal dimension. For this purpose, let us consider again the case of a fixed (non-random) closed hypothesis set $\mathcal{W} \subset \mathds{R}^d$. As in Section \ref{sec:fixed_w}, we make Assumption \ref{bounded_continuous_assumption}. We also introduce the \emph{data-dependent lower box-counting dimension} as:
\begin{equation}
    \label{eq:lower_box_counting}
    \underline{d}(S) = \lowerbox^{\rho_S}(\mathcal{W}) := \liminf_{\delta \to 0} \frac{\log |N_\delta^{\rho_S}(\mathcal{W})|}{\log(1/\delta)}.
\end{equation}

As proving lower bounds is not the main goal of this paper, we restrict ourselves to this specific setting. The next theorem is based on very classical arguments involving Gaussian complexity and Sudakov's theorem \citep[Theorem $7.4.1$]{vershynin_high-dimensional_2020}. This lower bound requires a slightly different definition of the worst-case generalization error:
\begin{equation}
\label{eq:absolute_worst_case_generalization}
    \overline{\mathcal{G}} (S) := \sup_{w \in \mathcal{W}} \big| \mathcal{R}(w) - \er_S(w) \big| .
\end{equation}

\begin{theorem}
    \label{thm:lower_bound}
    We further assume that $\underline{d}(S) > 0$ almost surely.
      Then, for all $\epsilon, \gamma, \eta > 0$, there is an absolute constant $c>0$ and some $\delta_{n,\gamma,\zeta} > 0$ such that, with probability at least $1 - \zeta - \gamma$, for all $\delta \leq \delta_{n,\gamma,\zeta}$ we have:
    $$
    \begin{aligned}
     \overline{\mathcal{G}} (S) \geq \frac{c}{4}\sqrt{\frac{\delta^2 \log(1/\delta)\underline{d}(S)}{2n \log(n)}} - B \sqrt{\frac{\log(2) + 9 \log(1/\zeta)}{n}}
    \end{aligned}
    $$
\end{theorem}
Theorem \ref{thm:lower_bound} gives a lower bound that is probably less tight compared to the upper bounds we presented in this work. One could even note that the right hand side of the bound may be negative in some contexts. However, we believe that the techniques used to derive this bound are classical and may be useful for future research.

\section{Computational Aspects}
In this section, we will illustrate how the proposed data-dependent dimension can be numerically computed, by making a rigorous connection to topological data analysis (TDA) tools \cite{boissonat_geometrical_2018}.  
\subsection{Persistent homology}

Persistent homology (PH) is a well known notion in TDA typically used for point cloud analysis \cite{edelsbrunner_computational_2010, carlsson_topological_2014}. Previous works have linked neural networks and algebraic topology \cite{rieck_neural_2019, perez-fernandez_characterizing_2021}, especially in \cite{corneanu_computing_2020} who established experimental evidence of a link between homology and generalization. Important progress was made in \cite{birdal_intrinsic_2021}, who used PH tools to estimate the upper-box counting dimension induced by the Euclidean distance on $\mathcal{W}_{S,U}$. Here we extend their approach to the case of data-dependent pseudo-metrics, which lays the ground for our experimental analysis.

The formal definition of PH is rather technical and is not essential to our problematic; hence, we only provide a high-level description here, and provide a more detailed description in Section~\ref{subsection:persistent_homology_dimension} (for a formal introduction, see \citep{boissonat_geometrical_2018,memoli_primer_2019}. In essence, given a point cloud $W \subset \mathds{R}^d$, `PH of degree $0$', denoted by $\text{PH}^0$ keeps track of the \emph{connected components} in $W$, as we examine $W$ at a gradually decreasing resolution.

Given a bounded (pseudo-)metric space $(X,\rho)$, by using $\text{PH}^0$, one can introduce another notion of fractal dimension, called the \emph{persistent homology dimension}, which we denote by $\dimph^\rho(X)$ (see Section~\ref{subsection:persistent_homology_dimension} and \citep[Definition $4$]{schweinhart_persistent_2019}).

Our particular interest in $\dimph^\rho(X)$ in the case where $\rho$ is a proper metric comes from an important result \cite{kozma_minimal_2005, schweinhart_fractal_2020} stating that for any bounded metric space $(X,\rho)$ we have the following identity. 
\begin{equation}
    \label{eq:dim_equality_metrics_spaces}
    \upperbox^\rho(X) = \dimph^\rho(X).
\end{equation}
Several studies used this property to numerically evaluate the upper box-counting dimension \cite{adams_fractal_2020, birdal_intrinsic_2021}. In particular \citet{birdal_intrinsic_2021} combined it with the results from \cite{simsekli_hausdorff_2021} and showed that $\dimph^\text{Eucl}(X)$ associated with the Euclidean metric on the parameter space, can be linked to the generalization error under the Lipschitz loss condition.

\subsection{PH dimension in pseudo-metric spaces}

In order to extend the aforementioned analysis to our data-dependent dimension, we must first prove that the equality \eqref{eq:dim_equality_metrics_spaces} extends to pseudo-metric spaces, which is established in the following theorem: 

\begin{theorem}
    \label{dim_equality_pseudo_metric_spaces}
    Let $(X,\rho)$ be a bounded pseudo-metric space, we have: $\upperbox^\rho(X) = \dimph^{\rho}(X)$.
\end{theorem}
This theorem shows that, similar to $\dimph^\text{Eucl}(X)$, our proposed dimension $\dimph^{\rho_S}(X)$ can also be computed by using numerically efficient TDA tools.
Moreover, Theorem \ref{main_result_fixed_hypothesis_space} now (informally) implies that with probability $1 - \zeta$:
\begin{equation}
\label{eq:informal_ph_bound}
\mathcal{G}(S) \lesssim  \sqrt{\frac{\dimph^{\rho_S}(W) \log(1/\delta) + \log(1/\zeta)}{n}} + \delta.
\end{equation}
Theorems \ref{main_result_HP_bound_convering_MI} and \ref{main_result_hp_bound_with_coverings_stability} can be adapted similarly.

\section{Experiments}
\label{section:experiments}

\textbf{Experimental setup. }
In our experiments, we closely follow the setting used in \citep{birdal_intrinsic_2021}. In particular, we consider learning a neural network by using SGD, and choose the hypothesis set $\mathcal{W}_{S,U}$ as the \emph{optimization trajectory} near the local minimum found by SGD\footnote{Note that as the trajectories collected by SGD will only contain finitely many points, its dimension will be trivially $0$. However, as in \cite{birdal_intrinsic_2021}, we treat this finite set an approximation to the full trajectory. This is justified since even for infinite $X$, $\dimph^{\rho_S}(X)$ is computed based on \emph{finite} subsets of $X$.}. Then, we numerically estimate $\dimph^{\rho_S}(\mathcal{W}_{S,U})$ by using the PH software provided in \cite{perez_giotto-ph_2021}. The main difference between our approach and \cite{birdal_intrinsic_2021} is that we replace the Euclidean metric with the pseudo-metric $\rho_S$ to compute the PH dimension.

\begin{figure}[!t]
    \begin{center}
    \centering
    \centerline{\includegraphics[trim={0 2cm 0 2.5cm}, width=0.9\columnwidth, clip]{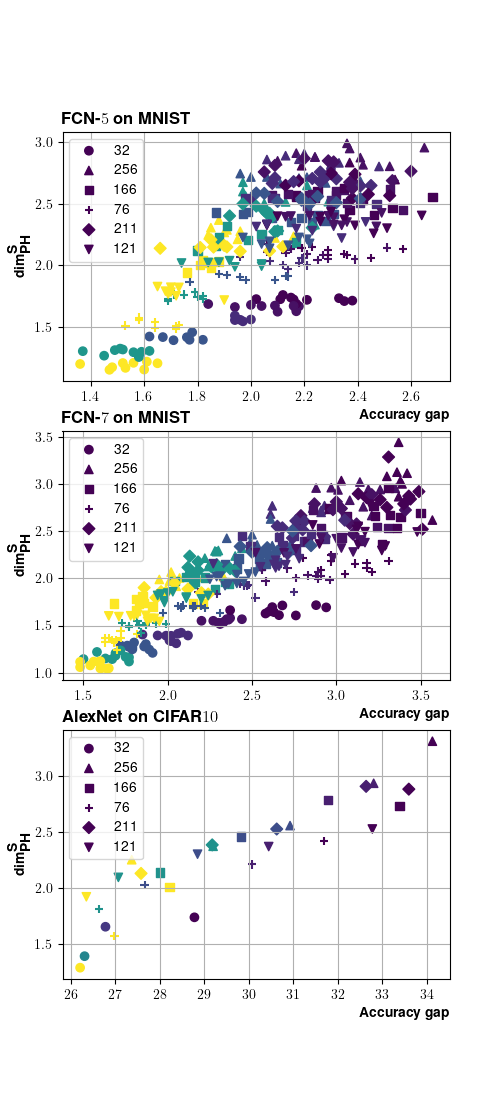}}
    \caption{$\dimph^{\rho_S}$ (denoted $\dimph^S$ in the figure) versus accuracy gap for FCN-$5$ (\textit{top}), FCN-$7$ (\textit{middle}) on MNIST and AlexNet (\textit{bottom}) on CIFAR-$10$ Different colors indicate different learning rates and different markers indicate different batch sizes.}
    \label{fig:classification}
    \end{center}
    \vskip -0.3in
\end{figure}

Here is a brief description of the method: given a neural network, its loss $\ell(w,z)$, and a dataset $S = (z_1,\dots,z_n)$, we compute the iterations of SGD for $K^\star$ iterations, $(w_k)_{k=0}^{K^\star}$, such that $w_{K^\star}$ reaches near a local minimum. We then run SGD for $5000$ more iterations and set $\mathcal{W}_{S,U}$ to $\{w_{K^\star+1}, \dots, w_{K^\star+5000} \}$. 
We then approximate $\dimph^{\rho_S}(\mathcal{W}{S,U})$ by
using the algorithm proposed in \cite{birdal_intrinsic_2021} by replacing the Euclidean distance with $\rho_S$. 

We experimentally evaluate $\dimph^{\rho_S}(\mathcal{W}_{S,U})$ in different settings: (i) regression experiment with Fully Connected Networks of $5$ (FCN-$5$) and $7$ (FCN-$7$) layers trained on the California Housing Dataset (CHD) \cite{kelley_pace_sparse_1997}, (ii) training FCN-$5$ and FCN-$7$ networks on the MNIST dataset \cite{lecun_gradient-based_1998} and (iii) training AlexNet \cite{krizhevsky_imagenet_2017} on the CIFAR-$10$ dataset \cite{krizhevsky_cifar-10_2014}. More experiments are shown in the appendix Section \ref{section:additional experimental reults}. All the experiments use standard ReLU activation and vanilla SGD with constant step-size. We made both learning rate and batch size vary across a $6 \times 6$ grid. For experiments on CHD and MNIST we also used $10$ different random seeds. 
All hyperparameter configurations are available in  Section~\ref{setcion:Additional experimental details}.

Note that in the case of a classification experiment, one could not compute $\dimph^{\rho_S}$ using a zero-one loss in  \eqref{losses_based_pseudo_metric}. Indeed, it would be equivalent to computing PH on the \emph{finite} set $\{0, 1\}^n \subset \mathds{R}^n$, which trivially gives an upper box-counting dimension of $0$. 
To overcome this issue, we compute $\dimph^{\rho_S}$ using the surrogate loss (cross entropy in our case) and illustrate that it is still a good predictor of the gap between the training and testing accuracies. For the sake of completeness, we provide how $\dimph^{\rho_S}$ behaves with respect to the the actual \emph{loss gap} in Section~\ref{section:additional experimental reults}.

\begin{figure}[!t]
    \vspace{5pt}
    \begin{center}
    \centering
    \centerline{\includegraphics[trim={0 0.5cm 0 1.5cm}, width=0.9\columnwidth, clip]{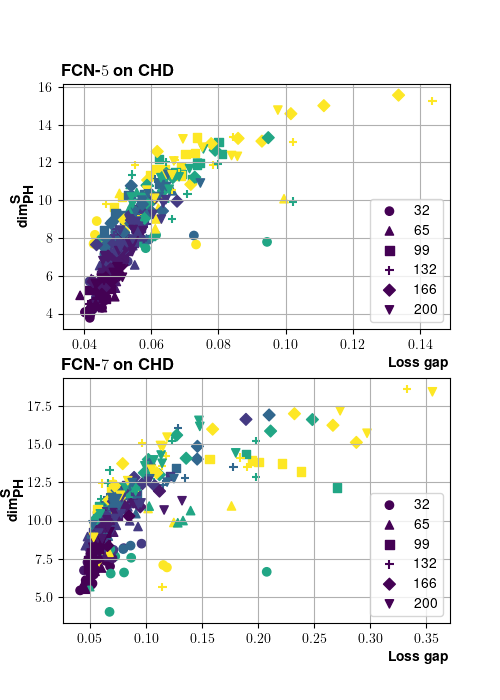}}
    \caption{$\dimph^{\rho_S}$ (denoted $\dimph^S$ in the figure) versus generalization gap for FCN-$5$ (\textit{top}) and FCN-$7$ (\textit{bottom}) trained on CHD. Different colors indicate different learning rates and different markers indicate different batch sizes.}
    \label{fig:regression}
    \end{center}
    \vskip -0.3in
\end{figure}

\textbf{Results. }
In order to compare our data-dependent intrinsic dimension with the one introduced in \cite{birdal_intrinsic_2021}, which is the PH dimension induced by the Euclidean distance on the trajectory and denoted $\dimph^{\text{Eucl}}$, we compute various correlation statistics, namely the Spearman's rank correlation coefficient $\rho$ \cite{kendall_advanced_1973} and Kendall's coefficient $\tau$ \cite{kendall_new_1938}. We also use the \emph{mean Granulated Kendall's Coefficient} $\boldsymbol{\Psi}$ introduced in \cite{jiang_fantastic_2019}, which aims at isolating the influence of each hyperparameter and according to the authors could better capture the causal relationships between the generalization and the proposed complexity metric (the intrinsic dimension in our case).
For more details on the exact computation of these coefficients, please refer to Section~\ref{subsection:kendall_coefficients}. Therefore $(\rho, \boldsymbol{\Psi}, \tau)$ are our main indicators of performance. The values of each granulated Kendall's coefficient are reported in Section \ref{section:additional experimental reults}\footnote{All those coefficients are between $-1$ and $1$, where the value of $1$ indicating a perfect positive correlation.}. 

Figures \ref{fig:classification} and \ref{fig:regression} depict the data-dependent dimension versus the generalization gap, as computed in different settings. We observe that, in all cases, we have a strong correlation between $\dimph^{\rho_S}(\mathcal{W}_{S,U})$ and the generalization gap, for a wide range of hyperparameters.  
We also observe that the highest learning rates and lowest batch sizes seem to give less correlation, which is similar to what was observed in \cite{birdal_intrinsic_2021} as well. This might be caused by the increased noise as we suspect that the point clouds in those settings show more complex fractal structures and hence require more points for a precise computation of the PH dimension. 

Next, we report the correlation coefficients for the same experiments in Tables \ref{table:kendall_chd}, \ref{table:kendall_mnist} and \ref{table:kendall_alexnet}. The results show that 
on average our proposed dimension always yields improved metrics compared to the dimension introduced in \cite{birdal_intrinsic_2021}.
The improvement is particularly better in the regression experiment we performed (as the classification task yields larger variations in the metrics, see Table \ref{table:kendall_mnist}). This may indicate that the proposed dimension may be particularly pertinent in specific settings. Moreover, increasing the size of the model, in all experiments, seems to have a positive impact on the correlation. We suspect that this might be due to the increasing local-Lipschitz constant of the network. We provide more experimental results in Section~\ref{section:additional experimental reults}.

\begin{table}[!t]
\caption{Correlation coefficients on CHD}
\label{table:kendall_chd}
\vspace{4pt}
\begin{center}
\begin{small}
\begin{sc}
\begin{tabular}{@{} l l l l l @{}} 
\toprule
{Model} & {Dim.} & {$\rho$}  & {$\boldsymbol{\Psi}$ } & {$\tau$} \\ 
\midrule
FCN-$5$ & $\dimph^{\text{eucl}}$ & $0.77 $\tiny{$\pm 0.08$}  & $0.54 $\tiny{$ \pm 0.11$} & $0.59 $\tiny{$ \pm 0.07$}    \\
FCN-$5$ & $\dimph^{\rho_S}$ & $\mathbf{0.87} $\tiny{$ \pm 0.05$}  & $\mathbf{0.68} $\tiny{$ \pm 0.10$}  & $\mathbf{0.71} $\tiny{$ \pm 0.09$} \\
\midrule
FCN-$7$ & $\dimph^{\text{eucl}}$ & $0.40 $\tiny{$ \pm 0.09$} &  $0.16 $\tiny{$ \pm 0.08$} & $0.28 $\tiny{$ \pm 0.07$}   \\
FCN-$7$ & $\dimph^{\rho_S}$ & $\mathbf{0.77} $\tiny{$ \pm 0.08$} & $\mathbf{0.62} $\tiny{$ \pm 0.06$}  & $\mathbf{0.77} $\tiny{$ \pm 0.08$} \\ 
\bottomrule
\end{tabular}
\end{sc}
\end{small}
\end{center}
\vskip -0.2in
\end{table}

\begin{table}[!t]
\caption{Correlation coefficients on MNIST}
\label{table:kendall_mnist}
\vspace{4pt}
\begin{center}
\begin{small}
\begin{sc}
\begin{tabular}{@{} l l l l l @{}} 
\toprule
{Model} & {Dim.} & {$\rho$}& {$\boldsymbol{\Psi}$ } & {$\tau$} \\ 
\midrule
FCN-$5$ & $\dimph^{\text{eucl}}$ & $0.62 $\tiny{$ \pm 0.10$} & $0.78$ \tiny{$\pm 0.08$} & $0.47 $\tiny{$ \pm 0.07$}   \\
FCN-$5$ & $\dimph^{\rho_S}$ & $\mathbf{0.73} $\tiny{$ \pm 0.07$} & $\mathbf{0.81} $\tiny{$ \pm 0.07$}  & $\mathbf{0.56} $\tiny{$ \pm 0.06$} \\ 
\midrule
FCN-$7$ & $\dimph^{\text{eucl}}$ & $0.80 $\tiny{$ \pm 0.04$}  & $0.88 $\tiny{$ \pm 0.04$ }& $0.62 $\tiny{$ \pm 0.04$ }   \\
FCN-$7$ & $\dimph^{\rho_S}$ & $\mathbf{0.89} $\tiny{$ \pm 0.02$} & $\mathbf{0.90} $\tiny{$ \pm 0.04$}  & $\mathbf{0.73} $\tiny{$ \pm 0.03$} 
\\
\bottomrule
\end{tabular}
\end{sc}
\end{small}
\end{center}
\vskip -0.2in
\end{table}

\begin{table}[!t]
\caption{Correlation coefficients with AlexNet on CIFAR-$10$}
\label{table:kendall_alexnet}
\begin{center}
\begin{small}
\begin{sc}
\begin{tabular}{@{} l l l l l@{}} 
\toprule
{Model} & {Dim.} & {$\rho$} & {$\boldsymbol{\Psi}$ } & {$\tau$} \\ 
\midrule
AlexNet & $\dimph^{\text{eucl}}$ & $0.86$ & $0.81$ & $0.68$    \\
AlexNet & $\dimph^{\rho_S}$ & $\mathbf{0.93}$ & $\mathbf{0.84}$ & $\mathbf{0.78}$ \\
\bottomrule
\end{tabular}
\end{sc}
\end{small}
\end{center}
\vskip -0.1in
\end{table}

\textbf{Robustness analysis. }
The computation of $\rho_S(w,w')$ requires the exact evaluation of the loss function on every data point $\{z_1, \dots, z_n\}$ for every $w,w' \in \mathcal{W}_{S,U}$. This introduces a computational bottleneck in case where $n$ is excessively large. To address this issue, in this section we will explore an approximate way of computing $\dimph^{\rho_S}$. Similar to the computation of a stochastic gradient, instead of computing the distance on every data point, we will first draw a random subset of data points $T \subset S$, with $|T| \ll n$ and use the following approximation $\rho_S(w,w') \approx \rho_T(w,w'):= \frac1{|T|} \sum_{z\in T} |\ell(w,z) - \ell(w',z)|$.

\begin{figure}[!t]
    \begin{center}
    \centering
    \centerline{\includegraphics[trim=30 3 50 4, width=\columnwidth, clip]{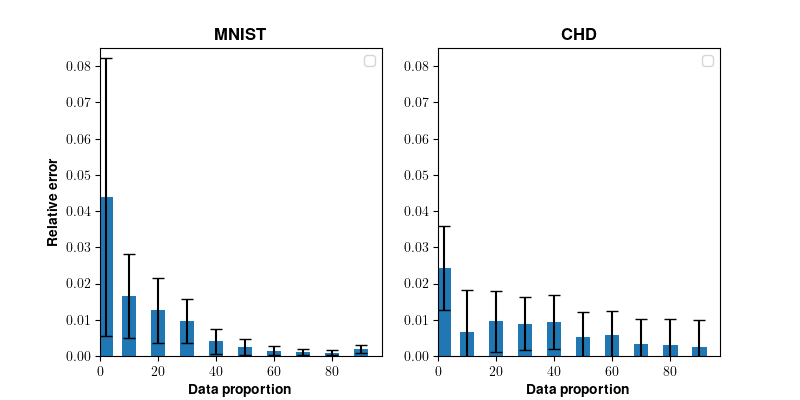}}
    \caption{Robustness experiment using a FCNN trained on MNIST (\textit{Left}) and CHD (\textit{Right}). $x$-axis represents the proportion of the data $T$ used to compute the metric, $y$-axis is the relative error with respect to the full dataset based dimension.}
    \label{fig:robustness}
    \end{center}
    \vskip -0.2in
\end{figure}

We now conduct experiments to analyze the robustness of the computation of $\dimph^{\rho_S}$ with respect to varying size of random subsets $T$. 
More precisely, we randomly select a subset $T \subset S$ whose size varies between $2\%$ and $99\%$ of the size dataset $S$ and compute the PH dimension using the approximate pseudo-metric. Note that the whole dataset $S$ is of course still used to produce the SGD iterates. Figure \ref{fig:robustness} presents results on the MNIST and CHD datasets in term of the relative error, i.e., $|\dimph^{\rho_T}-\dimph^{\rho_S}|/\dimph^{\rho_S}$. 
The results show that the proposed dimension is significantly robust to the approximation of the pseudo-metric: even with $40\%$ of the data, we achieve almost identical results as using the full dataset.

\section{Conclusion}

In this paper, we proved generalization bounds that do not require the Lipschitz continuity of the loss, which can be crucial in modern neural network settings. We linked the generalization error to a data-dependent fractal dimension of the random hypothesis set. We first extended some classical covering arguments to state a bound in the case of a fixed hypothesis set and then proved a result in a general learning setting. While some intricate mutual information terms between the geometry and the data appeared in this bound, we presented a possible workaround by the introduction of a stability property for the coverings of the hypothesis set. Finally, we made a connection to persistent homology, which 
allowed us to 
numerically approximate the intrinsic dimension and thus support our theory with experiments.

Certain points remain to be studied concerning our results. First the existence of differentiable persistent homology libraries \cite{hofer_deep_2018, hofer_connectivity-optimized_2019} open the door to the use of our intrinsic dimension as a regularization term as in \cite{birdal_intrinsic_2021}. Refining our proof techniques, for example using the chaining method \cite{ ledoux_probability_1991, clerico_chained_2022}, could help us improve our theoretical results or weaken the assumptions.

\section*{Acknowledgments}

 U.\c{S}. is partially supported by the French government under management of
Agence Nationale de la Recherche as part of the ``Investissements d'avenir'' program, reference
ANR-19-P3IA-0001 (PRAIRIE 3IA Institute). B.D. and U.\c{S}. are partially supported by the European Research Council Starting Grant
DYNASTY – 101039676.

\bibliography{main}
\bibliographystyle{icml2023}

\newpage
\appendix
\onecolumn

The outline of the appendix is as follows:
\begin{itemize}
    \item Section \ref{additional technical background}: Additional technical background related to information theory, Rademacher complexity, Egoroff's Theorem and persistent homology.
    \item Section \ref{Postponed proofs}: Postponed proofs of the theoretical results.
    \item Section \ref{setcion:Additional experimental details}: Additional experimental details.
    \item Section \ref{section:additional experimental reults}: Additional experimental results, including full statistic of experiments presented in the main part of the paper, as well as additional experiments on different datasets.
\end{itemize}

\section{Additional technical background}

\label{additional technical background}

\subsection{Information theoretic quantities}

\label{subsection:information_theoretic_quantities}

We recall there some basics concepts of information theory that we use throughout the paper. The absolute continuity of a probability measure with respect to another one will be denoted with symbol $\ll$.

\begin{definition}
Let us consider a probability space $(\Omega, \mathcal{F})$ and two probability distributions $\pi$ and $\rho$, with $\pi \ll \rho$. We define the \emph{Kullback-Leibler divergence} of those distributions as:
$$
\kl{\pi}{\rho} = \int \log\bigg( \frac{\der \pi}{\der \rho}  \bigg) \der \pi.
$$
For $\alpha > 1$, we define their \emph{$\alpha$-Renyi divergence} as:
$$
D_\alpha(\pi || \rho) = \frac{1}{\alpha - 1} \log \int\bigg( \frac{\der \pi}{\der \rho}  \bigg)^\alpha \der \rho.
$$
We set those two quantities to $+\infty$ if the absolute continuity condition is not verified.
\end{definition}

Note that by convention we often consider that $D_1 = \textbf{KL}$ and that Renyi divergences may also be defined for orders $\alpha < 1$  \cite{van_erven_renyi_2014}, but we won't need it here.

It is easy to prove that $D_\alpha$ is increasing in $\alpha$ and it is therefore natural to define:
$$
D_\infty(\pi || \rho) = \lim_{\alpha \to \infty} D_\alpha(\pi || \rho).
$$

The following property will be useful to perform decoupling of two random variables, a proof can be found in \citep{van_erven_renyi_2014}.

\begin{theorem}
    \label{essential_supremum]}
    With the same notations as above, we have:
    $$
    D_\infty(\pi || \rho) = \log \bigg(  \sup_{B \in \mathcal{F}} \frac{\pi(B)}{\rho(B)} \bigg).
    $$
\end{theorem}

We can then define the following notions of mutual information:

\begin{definition}
    Let $X,Y$ be two random variables on $\Omega$, we define for $\alpha \in [1,\infty]$:
    $$
    I_\alpha(X,Y) := D_\alpha(\mathds{P}_{X,Y} ||\mathds{P}_X \otimes \mathds{P}_{Y}),
    $$
    with in particular:
    $$
    I(X,Y) := I_1(X,Y) = \kl{\mathds{P}_{X,Y}}{\mathds{P}_X \otimes \mathds{P}_{Y}}.
    $$
\end{definition}

$I_\infty$ will be called the \textit{total mutual information}. Note that thanks to Theorem \ref{essential_supremum]}, we recover the definition of total mutual information that we wrote in Equation \eqref{eq:total_mutual_information_def}.

Those quantities satisfy the data processing inequality, given in the following proposition.

\begin{proposition}[Data-processing inequality]
    \label{prop:data_processing_inequality}
    If $X\longrightarrow Y \longrightarrow Z$ is a Markov chain and $\alpha \in [1,+\infty]$, then: 
    $$I_\alpha(X,Z) \leq I_\alpha(X,Y).$$ 
\end{proposition}

We are interested in those quantities because of their decoupling properties, summarized up in the following lemmas.:

\begin{lemma}[Lemma $1$ in \cite{xu_information-theoretic_2017}]
    \label{Xu-Raginsky_lemma}
    Let $X,Y$ be two random variables and $f(.,.)$ a measurable function. We consider $\bar{X}$ and $\bar{Y}$ two copies of $X$ and $Y$ which are independent. Then if $f(\bar{X}, \bar{Y}) - \mathds{E}[f(\bar{X}, \bar{Y})]$ is $\sigma^2$-subgaussian, we have:
    $$
    |\mathds{E}[f(X,Y)] - \mathds{E}[f(\bar{X}, \bar{Y})]| \leq \sqrt{2\sigma^2 I(X,Y)}.
    $$
\end{lemma}

We end this subsection by stating the decoupling in probability result that we will use several times in the proofs: Combining the definition of total mutual information with Theorem \ref{essential_supremum]}, we immediately obtain:

\begin{lemma}[Lemma $1$ in \cite{hodgkinson_generalization_2022}]
    \label{lemma:probability_decoupling_hodgkinson}
    For every measurable set $B$ we have:
    $$
    \mathds{P}_{X,Y} (B) \leq e^{I_\infty (X,Y)} \mathds{P}_X \otimes \mathds{P}_{Y} (B).
    $$
\end{lemma}

\subsection{Rademacher complexity}
\label{sec:radcpx}

We call Rademacher random variables a tuple $(\sigma_1\dots,\sigma_n)$ of mutually independent Bernoulli distributions with values in the set $\{ -1, 1\}$.

\begin{definition}
    \label{rademacher_complexity_def}
    Let us consider a fixed set $A \subset \mathds{R}^n$ and $\boldsymbol{\sigma} :=(\sigma_1\dots,\sigma_n)$ some Rademacher random variables, the Rademacher complexity of $A$ is defined as: 
    $$
    \rad(A) := \frac{1}{n} \mathds{E}_{\boldsymbol{\sigma} }\bigg[ \sup_{x \in A}  \sum_{i=1}^n \sigma_i x_i \bigg].
    $$
\end{definition}

Let us consider a fixed hypothesis space $\mathcal{W}$ and some dataset $S = (z_1,\dots, z_n) \sim \mu_z^{\otimes n}$, we will use the following notation:
\begin{equation}
    \ell(\mathcal{W}, S) = \{ (\ell(w,z_i)_{1\leq i \leq n} \in \mathds{R}^n, ~ w \in \mathcal{W}\}.
\end{equation}

\begin{remark}
One could legitimately inquire about the measurability of $\rad(\ell(\mathcal{W}, S))$ with respect to $\mathcal{F}^{\otimes n}$ (recall that the data space is denoted $(\mathcal{Z}, \mathcal{F}, \mu_z)$). Thanks to the closedness of $\mathcal{W} \subseteq\mathds{R}^d$ we can introduce a dense countable subset $\mathcal{C}$ of $\mathcal{W}$ and write that, thanks to the continuity of $\ell$,
$$
R(\boldsymbol{\sigma}, S) := \frac{1}{n} \sup_{w \in \mathcal{W}} \sum_{i=1}^n \sigma_i \ell(w,z_i) = \frac{1}{n} \sup_{w \in \mathcal{C}}  \sum_{i=1}^n  \sigma_i \ell(w,z_i),
$$
which is measurable as a countable supremum of random variables. As $\ell$ is bounded, so is $R(\boldsymbol{\sigma}, S)$; it is therefore integrable with respect to $(\boldsymbol{\sigma}, S)$. Thus $\rad(\ell(\mathcal{W}, S))$ is integrable (and measurable) thanks to the first part of Fubini's theorem.
\end{remark}

Rademacher complexity is linked to the worst case generalization error via the following proposition (see for example \cite{rebeschini_algorithmic_2020}):

\begin{proposition}
    \label{rademacher_generalization}
    Assume that the loss is uniformly bounded by $B$. Then, for all $\eta > 0$, we have with probability $1 - 2\eta$ that:
    $$
    \sup_{w \in \mathcal{W}} \big( \mathcal{R}(w) - \er_S(w) \big) \leq 2 \rad(\ell(\mathcal{W}, S))+ 3  \sqrt{\frac{2B^2}{n} \log(1/\eta)}.
    $$
\end{proposition}

We state the proof of this result for the sake of completeness. It is based on two classical arguments: symmetrization and Mc-Diarmid inequality.

\begin{proof}
Let us write:
$$
\mathcal{G}(S) := \sup_{w \in \mathcal{W}}  \big( \mathcal{R}(w) - \er_S(w) \big) .
$$
We introduce $\Tilde{S} = \{\Tilde{z}_1,\dots,\Tilde{z}_n\} \sim \mu_z^{\otimes n}$ an independent copy of $S$ and some Rademacher random variables $(\sigma_1,\dots,\sigma_n)$, using properties of conditional expectation and Fubini's theorem we have:
\begin{equation}
\label{rademacher_expected_bound}
\begin{aligned}
\mathds{E} [\mathcal{G}(S)]  &= \mathds{E} \bigg[  \sup_{w \in \mathcal{W}} \bigg(\frac{1}{n} \sum_{i=1}^n \mathcal{R}(w) -  \ell(w,z_i)  \bigg)\bigg] \\
&=  \mathds{E} \bigg[  \sup_{w \in \mathcal{W}} \frac{1}{n} \sum_{i=1}^n \mathds{E}[ \ell(w,\Tilde{z}_i) - \ell(w,z_i)  | \Tilde{S}]\bigg] \\
&\leq \mathds{E} \bigg[ \mathds{E} \bigg[  \sup_{w \in \mathcal{W}} \frac{1}{n} \sum_{i=1}^n (\ell(w,\Tilde{z}_i) - \ell(w,z_i)  ) \bigg| \Tilde{S} \bigg] \bigg] \\
&=  \mathds{E} \bigg[  \sup_{w \in \mathcal{W}} \frac{1}{n} \sum_{i=1}^n (\ell(w,\Tilde{z}_i) - \ell(w,z_i) )\bigg]  \\
&=  \mathds{E} \bigg[  \sup_{w \in \mathcal{W}} \frac{1}{n} \sum_{i=1}^n \sigma_i (\ell(w,z_i) - \ell(w,\Tilde{z}_i) )\bigg] \\
&\leq 2 \mathds{E} \bigg[  \sup_{w \in \mathcal{W}} \frac{1}{n} \sum_{i=1}^n \sigma_i \ell(w,z_i) \bigg] \\
&= 2 \mathds{E} [\rad(\ell(\mathcal{W}, S))].
\end{aligned}
\end{equation}

On the other hand if we denote $S^i = (z_1,\dots, z_{i-1}, \Tilde{z}_i, z_{i+1}, \dots z_n)$ we have that:
$$
|\mathcal{G}(S) - \mathcal{G}(S^i)| \leq \frac{2B}{n},
$$
And therefore by Mc-Diarmid inequality for any $\epsilon > 0$:
$$
\mathds{P} \bigg( \mathcal{G}(S) - \mathds{E} [\mathcal{G}(S)]  \geq \epsilon \bigg) \leq \exp \bigg\{ -\frac{n\epsilon^2}{2B^2} \bigg\}.
$$
By taking any $\eta \in (0,1)$ we can make a clever choice for $\epsilon$ and deduce that with probability at least $1 - \eta$ we have:
\begin{equation}
\label{Rademacher_HP1}
\mathcal{G}(S) \leq \mathds{E} [\mathcal{G}(S)] + \sqrt{\frac{2B^2}{n} \log(1/\eta)}.
\end{equation}
Moreover we can also write:
$$
    |\rad(\ell(\mathcal{W}, S)) - \rad(\ell(\mathcal{W}, S^i))| \leq \mathds{E}_{\boldsymbol{\sigma} }\bigg[  \sup_{w \in \mathcal{W}}  \frac{1}{n} \big| \sigma_i ( \ell(w,z_i)  - \ell(w,\Tilde{z_i})) \big|\bigg]  \leq \frac{2B}{n},
$$
so that by Mc-Diarmid and the exact same reasoning than above we have that with probability at least $1 - \eta$:
\begin{equation}
\label{rademacher_HP2}
     \mathds{E} [\rad(\ell(\mathcal{W}, S)) ]\leq\rad(\ell(\mathcal{W}, S)) + \sqrt{\frac{2B^2}{n} \log(1/\eta)}.
\end{equation}
Therefore combining equations \ref{rademacher_expected_bound}, \ref{Rademacher_HP1} and \ref{rademacher_HP2} gives us that with probability at least $1 - 2\eta$:
$$
\mathcal{G}(S)  \leq 2 \rad(\ell(\mathcal{W}, S))+ 3  \sqrt{\frac{2B^2}{n} \log(1/\eta)}.
$$
    
\end{proof}

Another important result for us is the well-known Massart's lemma, presented here in a slightly simplified version which is enough for our work:

\begin{lemma}[Massart's lemma]
    \label{lemma:Massart}
Let $T \subseteq \mathds{R}^n$ be a finite set, then:
    $$
    \rad(T) \leq \max_{t \in T}(\Vert t \Vert_2) \frac{\sqrt{2 \log(|T|)}}{n},
    $$
    Where $|T|$ denotes the cardinal of $T$ as usual.
\end{lemma}

\begin{example}
   Consider the setting where we have a fixed finite hypothesis set $\mathcal{W}$. In that case we have that $\max_{w\in \mathcal{W}} (\Vert (\ell(w,z_i))_i \Vert_2) \leq B\sqrt{n}$, thanks to the boundedness assumption. Thus Massart's lemma \ref{lemma:Massart} gives us
    \begin{equation}
    \label{eq:rademacher_bound_fixed_W}
    \rad(\ell(\mathcal{W}, S)) \leq B \sqrt{\frac{2 \log(|\mathcal{W}|)}{n}}.
    \end{equation}
\end{example}

\subsection{Egoroff's Theorem}

Egoroff's Theorem is an essential result in our theory which states that pointwise convergence in a probability space can be made uniform on measurable sets of arbitrary high probability. It was already used in \cite{simsekli_hausdorff_2021, camuto_fractal_2021} to make the convergence of the limit defining some fractal dimension uniform up  certain probability.

\begin{theorem}[Egoroff's Theorem \cite{bogachev_measure_2007}]
\label{egoroff}
     Let $(\Omega, \mathcal{F}, \mu)$ be a measurable space with $\mu$ a positive finite measure. Let $f_n, f : \Omega \longrightarrow (X,d)$ be functions with values in a separable metric space $X$ and such that $\mu$-almost everywhere $f_n(x) \to f(x)$.

    Then for all $\gamma > 0$ there exists $\Omega_\gamma \in \mathcal{F}$ such that $\mu(\Omega \backslash \Omega_\gamma ) \leq \gamma$ and on $\Omega_\gamma$ the convergence of $(f_n)$ to $f$ is uniform.
\end{theorem}

\subsection{Persistent Homology}
\label{subsection:persistent_homology_dimension}

Persistent homology (PH) is a well known notion in TDA typically used for point cloud analysis \cite{edelsbrunner_computational_2010, carlsson_topological_2014}. Previous works have linked neural networks and algebraic topology \cite{rieck_neural_2019, perez-fernandez_characterizing_2021}, especially in \cite{corneanu_computing_2020} who established experimental evidence of a link between homology and generalization. Important progress was made in \cite{birdal_intrinsic_2021}, who used PH tools to estimate the upper-box counting dimension induced by the Euclidean distance on $\mathcal{W}_{S,U}$. In this subsection, we introduce a few necessary PH tools to understand this approach.

Throughout this subsection we consider a finite set of point $W \subset \mathds{R}^m$. We will denote by $\mathds{K}$ the unique two elements field $\mathds{Z}/2\mathds{Z}$.

\begin{definition}[Abstract simplicial complex and filtrations]
\label{simplex_complex_filtration}
    Given a finite set $V$, an abstract simplicial complex (which we will often refer simply as complex) $K$ is a subset of $\mathcal{P}(V)$, the subsets of $V$, such that:
    \begin{itemize}
        \item $\forall v \in V,~ \{ v \} \in K$
        \item $\forall s \in K, \mathcal{P}(s) \subseteq K$
    \end{itemize}
    The elements of $K$ are called the simplices. For any non-empty simplex $s$, we call the number $|s| - 1$ its \emph{dimension}, denoted $\dim(s)$.
    Given a simplicial complex $K$, a filtration of $K$ is a sequence of sub-complexes increasing for the inclusion $\emptyset \subset K^0 \subset \dots \subset K^N = K$ such that every complex is obtained by adding one simplex to to the previous one: $K^{i+1} = K^i \cup \{ \sigma^{i+1} \}$. Thus a filtration of a complex induces an ordering on the simplices, which will be denoted $(s^i)_i$ by convention. 
\end{definition}

The filtration will be denoted by
$$
\emptyset \longrightarrow K^0 \longrightarrow \dots \longrightarrow K^N = K,
$$
and the corresponding simplices, in the order in which they are added to the filtration, will typically be denoted by $(s_0, \dots, s_N)$.

\begin{example}
    \label{ex:rips}
    The most important filtration that we shall encounter is the \emph{Vietoris-Rips filtration} (VR filtration) $\rips(W)$. For any $\delta > 0$ we first construct the Vietoris-Rips simplicial complex $\rips(W,\delta)$ by the following condition:
    \begin{equation}
    \label{vietoris_rips_condition}
    \forall k\{ w_1,\dots, w_k \} \in \rips(W,\delta) \iff \forall i,j,~d(p_i, p_j) \leq \delta.
    \end{equation}
    Then $\rips(W)$ is formed by adding the complexes in the increasing order of $\delta$ from $0$ to $+\infty$. Complexes with the same value of $\delta$ are ordered based on their dimension and ordered arbitrarily if they have the same dimension.
\end{example}

Intuitively, Persistent homology of degree $i$ keeps track of lifetimes of `holes of dimension $i$', it is built over the concept of chains, which are a sort of linearized version of sets of simplices. More precisely, the space of $k$-chains $C_k(K)$ over complex $K$ is defined as the set of formal linear combinations of the $k$-dimensional simplices of $k$:
\begin{equation}
    \label{k_chains}
    C_k(K) := \text{span} \big( \sum_i \epsilon_i s_i, ~\forall i, ~\dim(s_i) = k \big).
\end{equation}
We will denote a simplex by its points $s = [w_0,\dots,w_k]$ and use the notation $s_{\backslash i} := [w_0,\dots, w_{i-1}, w_{i+1}, \dots, w_k]$. The \emph{boundary operator} $\partial : C_k(K) \longrightarrow C_{k-1}(K)$ is the linear map induced by the relations on the simplices:
\begin{equation}
    \label{boundary_operator}
    \partial(s) = \sum_{i=0}^k s_{\backslash i}.
\end{equation}

It is easy to verify that $\partial^2 = 0$ and therefore we have an exact sequence, where $N = |W|$,
$$
\{ 0 \} \overset{\partial}{\longrightarrow} C_N(K)  \overset{\partial}{\longrightarrow} C_{N-1}(K)  \overset{\partial}{\longrightarrow} \dots  \overset{\partial}{\longrightarrow} C_0(K)  \overset{\partial}{\longrightarrow} \{ 0 \},
$$
from which it is natural to define:
\begin{definition}[Cycles and homology groups]
    With the same notations, we define:
    \begin{itemize}
        \item The $k$ cycles of $K$: $Z_k(K) := \ker(\partial: C_k(K) \longrightarrow C_{k-1}(K))$.
        \item The $k$-th boundary of $K$: $B_k(K) := \Im(\partial: C_{k+1}(K) \longrightarrow C_{k}(K))$.
        \item $k$-th homology group (it is actually a quotient vector space): $H_k(K) := Z_k / B_k$.
    \end{itemize}
    The $k$-th \emph{Betti} number of $K$ is defined as the dimension of the homology group: $\beta_k (K) = \dim(H_k(K))$.
\end{definition}

Those Betti numbers, $\beta_k$, correspond, in our analogy, to the number of holes of dimension $k$, i.e. the numbers of cycles whose `interior' is not in the complex, and therefore corresponds to a hole.

\begin{remark}
    In particular, $\beta_0$ corresponds to the number of connected components in the complex.
\end{remark}

Now that we defined the notion of homology, we go on with the definition of \emph{persistent homology} (PH). The intuition is the following: when we build the Vietoris-Rips filtration of the point cloud $W$, by increasing parameter $\delta$ in Example \ref{ex:rips}, we collect the `birth' and `death' of each hole, the multiset\footnote{By multiset, we mean that it can contain several time the same element, in our case the same persistence pair.} of those pairs (\textit{birth}, \textit{death}) will be the definition of persistent homology. 

\begin{remark}
    In all the following, the parameter $\delta$ used in the definition of the Vietoris-Rips filtration will be seen as a time parameter.
\end{remark}

While the concept of persistent homology can be extended to arbitrary orders (see \citep{boissonat_geometrical_2018}), here, for the sake of simplicity, we only define Persistent homology of degree $0$, which is much simpler and is the only one we need in our work. 

\textbf{Persistent homology of degree $0$:}

The persistent homology of degree $0$, denoted $\textbf{PH}^0$ is the multiset of the distances $\delta$ used to build the Vietoris-Rips filtration of $W$ for which a connected component is lost.

More formally, let us introduce a Vietoris-Rips filtration of $P$ denoted by:
$$
\emptyset \rightarrow K^{\delta_0, 1} \rightarrow \dots \rightarrow K^{\delta_0, \alpha_0} \rightarrow K^{\delta_1, 1} \rightarrow \dots \rightarrow K^{\delta_c, \alpha_C} = K,
$$
where $0\leq \delta_1 < \dots < \delta_C$ are the `time/distance' indices of the filtration and for the same value of $\delta$ the simplices are ordered by their dimension and arbitrarily if they also have the same dimension. Obviously $\delta_0 = 0$. With those notations, $\textbf{PH}^0$ is the multiset of all the $\delta_i$ corresponding to a complex $K^{\delta_i, j}$ which has one less connected component than the preceding complex in the above filtration. 

To stick with the usual notations, we actually define $\textbf{PH}^0$ as the multiset of the $(0,\delta_i)$, where the $0$ correspond to the `birth' of a connected component, while the $\delta_i$, as described above, corresponds to the `death' of this connected component.

\begin{definition}[Persistent homology dimension]
\label{PH-dim}
For any $\alpha \geq 0$ we define:
\begin{equation}
\label{eq:weighted_alpha_sum}
E_\alpha(W) := \sum_{(b,d) \in \text{PH}^0(\rips(W))} (d - b)^\alpha .
\end{equation}
The persistent homology dimension of degree $0$ (PH dimension) of any set bounded metric space $\mathcal{W}$ is then defined as:
$$
\dimph (\mathcal{W}) := \inf\{ \alpha > 0,~ \exists C > 0,~ \forall W \subset \mathcal{W} \text{ finite}, ~E_\alpha(W) < C  \}.
$$
Where the definition of VR filtration in finite subsets of metric spaces is naturally defined.
\end{definition}

The importance of this dimension for our work relies on the following result (see \cite{schweinhart_persistent_2019}, \cite{kozma_minimal_2005}):

\begin{proposition}
\label{dim_equality_metric}
    For any bounded metric space $X$, we have $\upperbox(X) = \dimph(X)$.
\end{proposition}

Proposition \ref{dim_equality_metric} opens the door to the numerical estimation of the upper-box dimension. Indeed, PH can be evaluated via several libraries \citep{bauer_ripser_2021, perez_giotto-ph_2021}, moreover, \citet{birdal_intrinsic_2021} noted that, while Definition \ref{PH-dim} is impossible to evaluate in practice, it can be approximated from $\text{PH}^0(\rips(W))$ computed on a finite number of finite subsets of the point cloud $\mathcal{W}$.

\subsection{Numerical estimation of the PH dimension}
\label{subsection:birdal_algorithm}

In this section we briefly discuss how we numerically estimate the persistent homology dimension, which is essentially the algorithm presented in \cite{birdal_intrinsic_2021} where we changed the distance, which implies that we must evaluate on all data points for the last iterates. See also \cite{adams_fractal_2020, schweinhart_fractal_2020} for similar ideas.

All persistent homology computation presented here have been made with the package presented in \cite{perez_giotto-ph_2021}, which allows us to use more points in our persistent homology computation, e.g. \citet{birdal_intrinsic_2021} was only using between $1000$ points prior to convergence for AlexNet and $200$ for the other experiments. In our work we use up to $8000$ points, which may allow us to better capture the fractal behavior.

The algorithm is based on the following result, proved by proposition $2$ of \cite{birdal_intrinsic_2021} and proposition $21$ of \cite{schweinhart_fractal_2020}: If $X$ is a bounded metric space with $\Delta = \dimph^d(X)$, then for all $\epsilon > 0$ and $\alpha \in (0, \Delta +\epsilon)$ there exists $D_{\alpha, \epsilon} > 0$ such that for all finite subset $X_n = \{x_1,\dots,x_n\}$ of $X$ we have:

\begin{equation}
    \label{eq:birdal_algorithm}
    \log E_\alpha(X_n) \leq \log D_{\alpha, \epsilon} + \bigg( 1 - \frac{\alpha}{\Delta + \epsilon} \bigg) \log(n).
\end{equation}

Then we can perform an affine regression of $\log E_\alpha(X_n)$ with respect to $\log n$ and get a slope $a$. Moreover it is argued in \cite{birdal_intrinsic_2021} that the slope has good chance to be approximately the one appearing in Equation \eqref{eq:birdal_algorithm}, which gives us $\Delta \simeq \frac{\alpha}{1 - a}$.

\begin{remark}
    The aforementioned algorithm \emph{works in pseudo metric spaces}. Indeed as we tried to explain formally in the proof of proposition \ref{pseudo_ph_dim_lemma}, $\text{PH}^0$ in a pseudo-metric space only add some zeros to the quantities $E_\alpha$ computed in its metric identifications. Therefore the above algorithm is approximating $\dimph^{\rho_S}(X/\sim)$ which is proven in lemma \ref{pseudo_ph_dim_lemma} to be equal to $\dimph^{\rho_S}(X)$. See those notations in the next subsection.
\end{remark}

\subsection{Measurable coverings and additional technical lemmas}
\label{subsection:technical_lemmas}

In this section we briefly discuss a few technical measure theoretic points that are worth mentioning. Essentially, we argue that our measurability assumptions ensure that the manipulations we make in our proofs on complicated random variables are valid and meaningful. We then show that it is possible to construct the measurable coverings that we need in our proofs.

\subsubsection{Some nice consequences of our measurability assumptions}

The worst-case generalization error takes the general form:
\begin{equation}
    \label{eq:worst_case_gen_general_case}
    \mathcal{G}(S,U) := \sup_{\mathcal{W}_{S,U}} \big( \mathcal{R}(w) - \er_S(w) \big).
\end{equation}

Here we require $\mathcal{W}_{S,U} \subset \mathds{R}^d$ to be a random closed set. In this subsection, we will make this precise by describing basic notions of random set theory and prove a few technical results which will lay the ground of a rigorous theoretical basis for our main results. The interested reader can consult \citep{kechris_classical_1995, molchanov_theory_2017}. Other works mentioned similar formulation of the problem \citep{hodgkinson_generalization_2022}, though with not much technical details. 

Let us fix a probability space $(\Omega, \mathcal{T}, \mathds{P})$ and denote $E = \mathds{R}^d$.

\begin{remark}
    As highlighted by \citep{molchanov_theory_2017}, we can develop the following theory in the more general case where $E$ is a locally compact Hausdorff second countable space, but we avoid those technical considerations.
\end{remark}

The definition of a random closed set is the following:

\begin{definition}[Random closed set]
\label{def:random_closed_set}
Consider a map $W: \Omega \longrightarrow \closed(E)$, $W$ is said to be a random closed set if for every compact set $K \subset E$ we have:
$$
\{ \omega, W(\omega) \cap K \neq \emptyset \} \in \mathcal{T}.
$$
\end{definition}

A natural question is to know whether we can cast it as a random variable defined in the usual way, the answer is yes and is formalized by the following definition.

\begin{definition}[Effrös $\sigma$-algebra and Fell topology]
    \label{def:effros}
    The Effrös $\sigma$-algebra is the one generated by the sets $\{ W \in \closed(E), W \cap K \neq \emptyset \} $ for $K$ going over all compact sets in $\mathds{R}^d$. 
    
    The \emph{Fell topology} on $\closed(E)$ is the one generated by open sets $\{ W \in \closed(E), W \cap K \neq \emptyset \} $ for $K$ going over all compact sets and $\{ W \in \closed(E), W \cap \mathcal{O} \neq \emptyset \} $ for $\mathcal{O}$ going over all open sets of $\mathds{R}^d$.

    One can show that the Effrös $\sigma$-algebra on $\closed(E)$ corresponds to the Borel $\sigma$-algebra induced by the Fell topology \citep[Chapter $1.1$]{molchanov_theory_2017}. The Effrös $\sigma$-algebra will be denoted by $\mathfrak{E}(E)$.

    It can be shown that Definition \ref{def:random_closed_set} is equivalent to asking the measurability of $W$ with respect to $\mathfrak{E}(E)$.
\end{definition}

The assumption that we made on our learning algorithm is the following:

\begin{assumption}
    \label{ass:w_random_set}
    We assume that $\mathcal{W}_{S,U}$ is a random closed set in the sense of the above definition. It means that the mapping defining the learning algorithm:
    $$
    \mathcal{A} : \bigcup_{n=0}^{+\infty} \mathcal{Z}^n \times \Omega_U \longrightarrow \closed(\mathds{R}^d),
    $$
    is measurable with respect to the Effrös $\sigma$-algebra.
\end{assumption}

Thanks to this definition, we can already state one particularly useful result:

\begin{proposition}[Theorem $1.3.28$ in \cite{molchanov_theory_2017}]
    Consider $(G_w)_{w \in E}$ a $\mathds{R}$-valued, almost surely continuous, stochastic process on $E = \mathds{R}^d$ and $W$ a random closed set in $E$. Then the mapping
    $$
    \Omega \ni \omega \longmapsto \sup_{w \in W(\omega)} G_w(\omega)
    $$
    is a random variable.
\end{proposition}

\begin{example}
    \label{ex:measurability_of_wg}
    If we define $\mathcal{R}(w) - \er_S(w) $ and $W = \mathcal{W}_{S,U}$, then thanks to the continuity of the loss (Assumption \ref{bounded_continuous_assumption}) we have that the worst case generalization error defined by Equation \eqref{eq:worst_case_gen_general_case} is a well-defined random variable.
\end{example}

While Example \ref{ex:measurability_of_wg} gives us useful information, it is actually not enough for some arguments of our proofs to hold. In particular, to deal with the statistical dependence between the data and the random hypothesis set, we want to be able to perform the following operation: Given a random closed set $W$ and $S \in \mathcal{Z}^n$ we want to apply the decoupling results and write:
\begin{equation}
    \label{eq:proba_decoupling_example}
    \mathds{P}_{W,S}\bigg(\sup_{w\in W} \mathcal{R}(w) - \er_S(w) \geq \epsilon \bigg) \leq e^{I_{\infty}(W,S)} \mathds{P}_W \otimes \mathds{P}_S \bigg(\sup_{w\in W} \mathcal{R}(w) - \er_S(w) \geq \epsilon \bigg) .
\end{equation}

In order for the decoupling lemmas to hold, we actually need the measurability of the mapping
$$
\closed(\mathds{R}^d) \times \mathcal{Z}^n \ni (W, S) \longmapsto \sup_{w \in W} \vert \mathcal{R}(w) - \er_S(w) \vert,
$$
with respect to $\mathfrak{E}(\mathds{R}^d) \otimes \mathcal{F}^{\otimes n}$.

We show two results in this direction, the first one assuming that the data space $\mathcal{Z}$ is countable\footnote{This countability assumption on the dataset is found in some other works, especially in \citep{simsekli_hausdorff_2021} who used it to leverage the local stability of Hausdorff dimension}.

\begin{lemma}
    As before, let $(\closed(\mathds{R}^d), \mathfrak{E}(\mathds{R}^d))$ denotes the closed sets of $\mathds{R}^d$ endowed with the Effrös $\sigma$-algebra, $(\Omega, \mathcal{T})$ be a countable measurable space (with $\mathcal{T} = \mathcal{P}(\Omega)$) and $\zeta(x, \omega)$ be an almost surely continuous stochastic process on $\mathds{R}^d$. Then the function
    $$
    f: \closed(\mathds{R}^d) \times \Omega \ni  (W, \omega) \longmapsto \sup_{x \in W} \zeta(x, \omega) \in \mathds{R}
    $$
    is measurable with respect to $\mathfrak{E}(\mathds{R}^d) \otimes \mathcal{T}$.
\end{lemma}

\begin{proof}

It is enough to show that $f^{-1}(]t, +\infty [) \in \mathfrak{E}(\mathds{R}^d) \otimes \mathcal{T}$ for any $t \in \mathds{Q}$ as those sets $]t, +\infty[$ generate the Borel $\sigma$-algebra in $\mathds{R}$. Let us fix some $t\in \mathds{Q}$.
Let us denote $\zeta_\omega := \zeta(\cdot, \omega)$, we have:
$$
f^{-1}(]t, +\infty[) = \bigcup_{\omega \in \Omega} \bigg( \{ F \in \closed(\mathds{R}^d),~ F \cap \zeta_\omega^{-1}(]t, +\infty[) \neq \emptyset  \} \times \{\omega\} \bigg).
$$
By \citep[Proposition $1.1.2$]{molchanov_theory_2017}, we have that the sets of the form $\{  F \in \closed(\mathds{R}^d),~ F \cap \mathcal{O} \neq \emptyset \}$ generate $\mathfrak{E}(\mathds{R}^d)$, with $\mathcal{O}$ running through open sets of $\mathds{R}^d$. Therefore the continuity of $\zeta$ and the countability of $\mathcal{Z}$ give us:
$$
f^{-1}(]t, +\infty[) \in \mathfrak{E}(\mathds{R}^d) \otimes \mathcal{T}.
$$
    
\end{proof}

If we want to get rid of the countability assumption on $\Omega$, we have to introduce some metric structure on it. This approach justifies the assumptions made on $\mathcal{Z}$ (that it is a sub-metric space of some $\mathds{R}^N$).

\begin{lemma}
    Assume that $\Omega$ is a Polish space with a dense countable subset $D$ and that $\zeta$ is continuous in both variables. Then the function: 
    $$
    f: \closed(\mathds{R}^d) \times \Omega \ni  (W, \omega) \longmapsto \sup_{x \in W} \zeta(x, \omega) \in \mathds{R},
    $$
    is measurable with respect to $\mathfrak{E}(\mathds{R}^d) \otimes \mathcal{B}_\Omega$, where $\mathcal{B}_\Omega$ is the Borel $\sigma$-algebra on $\Omega$.
\end{lemma}

\begin{proof}
As before, let $t \in \mathds{Q}$, for $X, \omega \in \closed(\mathds{R}^d) \times \Omega$ we have that
$$
\begin{aligned}
    (X,\omega) \in f^{-1}(]t, +\infty[) \iff \exists x \in X,~\exists \epsilon \in \mathds{Q}_{>0}, ~\exists \bar{d} \in D,~\forall d \in B(\bar{d}, \epsilon) \cap D, \zeta(d, x) > t,
\end{aligned}
$$
and therefore
$$
f^{-1}(]t, +\infty[) = \bigcup_{\bar{r}\in D} \bigcup_{\epsilon \in \mathds{Q}_{>0}} \bigg\{  \bigg( \bigcap_{d \in B(\bar{d}, \epsilon)} \{  F \in \closed(\mathds{R}^d),~ F \cap \zeta_\omega^{-1}(]t, +\infty[) \neq \emptyset \} \bigg)\times B(\bar{d}, \epsilon) \bigg\}.
$$
The results follows from the same arguments as in the proof of the previous lemma.

\end{proof}

\subsubsection{Construction of measurable coverings}

To end this technical discussion about random closed set we try to answer the following questions: are the covering numbers with respect to pseudo-metric $\rho_S$ measurable? Moreover, can we construct coverings that are well-defined random close sets themselves?

Recall that we defined a $\delta$-covering of some set $X$ as a minimal set of points $N_\delta$ of $X$ such that:
$$
X \subseteq \bigcup_{w\in N_\delta} \bar{B}^{\rho}_\delta(w)
$$
The fact that we ask the coverings to be in $X$ is for technical reasons and does not change the values of the dimensions.

We first need a technical lemma to ensure that it is equivalent to cover dense countable subsets:

\begin{lemma}[Closure property of coverings]
    \label{lemma:coverings_closure_property}
    Let $W$ be a closed set and $\mathcal{C}$ be a countable dense subset of $W$. Under Assumption \ref{bounded_continuous_assumption} we have that any covering of $\mathcal{C}$ is a covering of $W$ for pseudo-metric $\rho_S$. Moreover we have, for all $\delta > 0$:
    \begin{equation}
        |N_{2\delta}^{\rho_S}(\mathcal{C})| \leq |N_{\delta}^{\rho_S}(W)| \leq |N_{\delta}^{\rho_S}(\mathcal{C})|.
    \end{equation}
\end{lemma}

\begin{proof}

Let us consider some $\delta> 0$, a minimal $\delta$-cover $\{c_1,\dots,c_K\}$ of $\mathcal{C}$ and $w \in W$. By density, there exists a sequence $(\xi_n)_n$ in $\mathcal{C}$ such that $\xi_n \to w$. As $\{c_1,\dots,c_K\}$ is a finite cover of $\mathcal{C}$, we can assume without loss of generality that, for all $n$, $\xi_n \in \bar{B}^{\rho_S}_\delta(c_i)$ for some $i$. Therefore, by continuity we have $\rho_S(w,c_i) = \lim_{n \to \infty} \rho_S (w, \xi_n) \leq \delta$. Thus:
$$
|N_{\delta}^{\rho_S}(W)| \leq |N_{\delta}^{\rho_S}(\mathcal{C})|.
$$

Now, by the triangle inequality we have:
$$
|N_{2\delta}^{\rho_S}(\mathcal{C})| \leq |N_{\delta}^{\rho_S}(W)|
$$
 
\end{proof}

Let us first prove that we can construct measurable coverings in the case of fixed hypothesis sets. Indeed, this is essential to ensure the fact that the upper box-counting dimension $\upperbox^{\rho_S}$ induced by $\rho_S$ is a well-defined random variable, which is required for the high probability bounds in our results to make sense. This kind of measurability condition is often assumed by authors dealing with potentially random covering numbers \citep{simsekli_hausdorff_2021, camuto_fractal_2021}. In our case, we can prove this measurability under some condition.

Recall that we required that $\mathcal{Z}$ has a metric space structure, typically inherited by an inclusion in an Euclidean space $\mathds{R}_N$ and that its $\sigma$-algebra $\mathcal{F}$ is the corresponding Borel $\sigma$-algebra. With that in mind we prove the following theorem:

\begin{theorem}[Measurability of covering numbers in the case of fixed hypothesis set]
\label{th:covering_numbers_measurability_fixed_w}
    Let $\mathcal{W}$ be a closed set, $\mathcal{C}$ be a dense countable subset\footnote{It always exists for any closed set in $\mathds{R}^d$.} of $\mathcal{W}$ and $\delta>0$. Under Assumption \ref{bounded_continuous_assumption}, we have that the mapping between probability spaces
    $$
    (\mathcal{Z}^n, ~\mathcal{F}^{\otimes n}) \ni S \longmapsto |N_\delta^{\rho_S}(\mathcal{C})| \in (\mathds{N}_+,~ \mathcal{P}(\mathds{N}_+)),
    $$
    is a random variable, where $\mathcal{P}(A)$ denotes the subsets of a set $A$.
\end{theorem}

\begin{proof}

For any set $X$ let us denote by $\mathfrak{F}_{\leq k}(X)$ the set of finite subsets of $X$ with at most $k$ elements. 

We start by noting that thanks to the continuous loss assumption, we have that $S \mapsto \rho_S(w,w')$ is continuous for any $w,w' \in \mathds{R}^d$. Moreover, let us denote $\mathcal{C} := \{ w_k, ~k \in \mathds{N} \}$.

Thus, to show the measurability condition, it suffices to show that for any $M \in \mathds{N}_+$ we have: $\{ S \in \mathcal{Z}^n,~ |N_\delta^{\rho_S}(\mathcal{C})|  \leq M \} \in \mathcal{F}^{\otimes n}$. we can write

$$
|N_\delta^{\rho_S}(\mathcal{C})|  \leq M  \iff \exists F \in \mathfrak{F}_{\leq M}(\mathcal{C}) , ~\forall k \in \mathds{N}, ~ \mathcal{C} \subset \bigcup_{c \in F} \bar{B}^{\rho_S}_\delta(c).
$$

Therefore

\begin{equation}
\label{eq:covering_numbers_mes}
    \{ S \in \mathcal{Z}^n,~ |N_\delta^{\rho_S}(\mathcal{C})|  \leq M \} = \bigcup_{ F \in \mathfrak{F}_{\leq M}(\mathcal{C}) } \bigcap_{k \in \mathds{N}} \bigcup_{c \in F} \{ S,~\rho_S(c,w_k) \leq \delta \}.
\end{equation}

By continuity, it is clear that $ \{ S,~\rho_S(c,w_k) \leq \delta \} \in \mathcal{F}^{\otimes n}$, hence we have the result by countable unions and intersections.
    
\end{proof}

\begin{remark}
    \label{rq:dim_countable_limit}
    Given any positive sequence $\delta_k$, decreasing and converging to $0$, thanks to Lemma \ref{lemma:coverings_closure_property} the upper box-counting dimension can be written as
    \begin{equation}
    \label{eq:countable_upper_limit}
    \upperbox^{\rho_S} (\mathcal{W}) = \limsup_{k\to + \infty} \frac{\log|N_{\delta_k}^{\rho_S}(\mathcal{C})|}{\log(1/\delta_k)},
    \end{equation}
    which, by Theorem \ref{th:covering_numbers_measurability_fixed_w}, implies that $ \upperbox^{\rho_S} (\mathcal{W}) $ is a random variable as countable upper limit of random variables.
\end{remark}

We now come to the case of random hypothesis sets, we begin by introducing Castaing's representations, which are a fundamental tool to deal with random closed sets \citep[Theorem $1.3.3$ and Definition $1.3.6$]{molchanov_theory_2017}.

\begin{proposition}[Castaing's representations]
    Let $W$ be a random closed set in $\mathds{R}^d$, then there exists a countable family $(\xi_n)_{n\geq 1}$ of $\mathds{R}^d$-valued random variables whose closure is almost surely equal to $W$, namely:
    $$
    \overline{\{ \xi_n,~n\geq 1 \}} = W,~\text{almost surely.}
    $$
\end{proposition}

Equipped with this result, we can easily extend Theorem \ref{th:covering_numbers_measurability_fixed_w} to the measurability of the covering numbers associated to a Castaing's representation of the hypothesis set:

\begin{theorem}[Measurability of covering numbers in the case of random hypothesis set]
\label{th:covering_numbers_measurability_random_w}
    Let $W \subset \mathds{R}^d$ be a random closed set over a probability space $(\Omega, \mathcal{T})$ and $\delta>0$. Let us introduce a Castaing's representation $(\xi_n)_{n\geq 1}$ of $W$.

    Then, under Assumption \ref{bounded_continuous_assumption}, we have that the mapping between probability spaces
    $$
    (\mathcal{Z}^n, ~\mathcal{F}^{\otimes n}) \otimes (\Omega, \mathcal{T}) \ni (S,\omega) \longmapsto |N_\delta^{\rho_S}(\{ \xi_n(\omega),~n\geq 1 \})|\in (\mathds{N}_+,~ \mathcal{P}(\mathds{N}_+)),
    $$
    is a random variable, where $\mathcal{P}(A)$ denotes the subsets of a set $A$. In particular, the upper-box counting dimension $\upperbox^{\rho_S}$ is a random variable.
\end{theorem}

\begin{proof}

The proof follows exactly that of Theorem \ref{th:covering_numbers_measurability_fixed_w} except that now have a Castaing's representation $(\xi_n)_{n\geq 1}$ of $W$.

By the same proof than Equation \eqref{eq:covering_numbers_mes}, we have:

$$
    \{ (S, \omega),~ |N_\delta^{\rho_S}(\{ \xi_n(\omega),~n\geq 1 \})|  \leq M \} = \bigcup_{ I \in \mathfrak{F}_{\leq M} (\mathds{N}_+)} ~ \bigcap_{k \in \mathds{N}} \bigcup_{i \in I} \{ (S,\omega),~\rho_S(\xi_i(\omega),\xi_k(\omega)) \leq \delta \}.
$$

By continuity and composition of random variables, it is clear that $$ \{ (S,\omega),~\rho_S(\xi_i(\omega),\xi_k(\omega)) \leq \delta \} \in \mathcal{F}^{\otimes n} \otimes \mathcal{T},$$ hence we have the result by countable unions and intersections.

Therefore $\upperbox^{\rho_S}(W(\omega))$ is a random variable is a random variable as a direct consequence of Lemma \ref{lemma:coverings_closure_property}.
    
\end{proof}

Thanks to Theorem \ref{th:covering_numbers_measurability_random_w}, we are actually able to prove the much stronger result that we \emph{can} build measurable coverings.

\begin{theorem}[Measurable coverings]
\label{th:measurable_coverings}
    Let $W \subset \mathds{R}^d$ be a random closed set over a probability space $(\Omega, \mathcal{T}, \mathds{P})$ and $\delta>0$. Let $\mathfrak{F}(\mathds{N}_+)$ denote the set of finite subsets of $\mathds{N}_+$. Then, under Assumption \ref{bounded_continuous_assumption}, we can build a map:
    $$
    N_\delta: \mathcal{Z}^n \times \Omega \longrightarrow \mathfrak{F}(\mathds{R}^d) \subset \closed(\mathds{R}^d),
    $$
    which is measurable (with respect to the Effrös $\sigma$-algebra on the right hand-side) and such that for almost all $(S,\omega ) \in \mathcal{Z}^n \times \Omega$, $N_\delta(S,\omega)$ is a finite set which is (almost surely) a covering of $W(\omega)$ with respect to pseudo-metric $\rho_S$ and such that we have almost surely over $\mu_z^{\otimes n} \otimes \mathds{P}$:
    $$
    \upperbox^{\rho_S} (W(\omega)) = \limsup_{\delta \to 0} \frac{|N_\delta (S,\omega)|}{\log(1/\delta)}.
    $$
\end{theorem}

\begin{proof}
    Let us introduce a Castaing's representation $(\xi_k)_{k\geq 1}$ of $W$ and denote by $\mathfrak{F}_N(\mathds{N}_+)$ the set of finite subsets of $\mathds{N}_+$ with exactly $N$ elements. Again, as in Theorem \ref{th:covering_numbers_measurability_fixed_w}, the proof is based on the idea that, thanks to the continuity of the loss $\ell$ defining the pseudo-metric $\rho_S$, a cover of $\{ \xi_k, ~k\in \mathds{N}_+ \}$ covers $W$. Let us denote $\mathcal{C}(\omega) = \{ \xi_k(\omega), ~k\in \mathds{N}_+ \}$.

    As $\mathfrak{F}_N(\mathds{N}_+)$ is countable, for each $N \in \mathds{N}_+$, we introduce $(F^N_i)_{i\geq 1}$ an ordering of $\mathfrak{F}_N(\mathds{N}_+)$.

    Now for each $(S,\omega) \in \mathcal{Z}^n \times \Omega$, we define:
    $$
    \forall i \in \mathds{N}_+, ~F_i(S,\omega) := F_i^{|N_\delta^{\rho_S}(\mathcal{C}(\omega))|}.
    $$
    Let us now introduce the minimal index of a set of indices that can cover $W$:
    $$
    i_0(S,\omega) := \text{argmin} \bigg\{ i \in \mathds{N}_+,~\forall k \geq 1, ~\exists j \in F_i(S,\omega), ~\rho(\xi_j(\omega), \xi_i(\omega)) \leq \delta \bigg\}.
    $$
    Note that $i_0$ is finite because $\{(\ell(w, z_i)_{1\leq i \leq n}),~ w _in W\}$ is compactly contained, thanks to the boundedness assumption on $\ell$, i.e. the covering numbers are finite.
    
    We can therefore build the following `covering indices' function:
    $$
    \begin{aligned}
        \mathcal{I}_\delta: \mathcal{Z}^n \times \Omega \longrightarrow \mathfrak{F}(\mathds{N}_+),
    \end{aligned}
    $$
    defined by $ \mathcal{I}_\delta (S,\omega) = F_{i_0(S,\omega)}(S,\omega)$.

    Now we want to introduce an `evaluation functional', i.e. a mapping:
    $$
    \Xi : \Omega \times \mathfrak{F}(\mathds{N}_+) \longrightarrow \mathfrak{F}(\mathds{R}^d) \subset \closed(\mathds{R}^d),
    $$
    defined by $\Xi(\omega, I) = \{ \xi_i(\omega),~i\in I \}$. It is easy to see that $\Xi$ is measurable, indeed for any compact set $K \subset \mathds{R}^d$ we have:
    $$
    \{ (\omega, I),~ \Xi(\omega, I) \cap K \neq \emptyset \} = \bigcup_{F \in \mathfrak{F}(\mathds{N}_+)} \bigcup_{i \in F} \{ \xi_i \in K \} \times \{ I \},
    $$
    implying the measurability by countable unions and Definition \ref{def:effros}.

    The key point of the proof is that we construct the coverings as $N_\delta(S,\omega) = \Xi(\omega, \mathcal{I}_\delta (S,\omega))$, so that the measurability of $N_\delta$ reduces to that of $\mathcal{I}_\delta $. This is achieved by noting that for any non-empty $I \in \mathfrak{F}(\mathds{N}_+)$, such that $I = F^N_{i_1}$ for some $N,i_1 \geq 1$, we have, by leveraging the countable ordering of $\mathfrak{F}_N(\mathds{N}_+)$:

    $$
    \begin{aligned}
    \mathcal{I}_\delta^{-1}(\{ I \} ) = & \{ |N_\delta^{\rho_S}(\mathcal{C}(\omega))| = N \} \cap \bigg( \bigcap_{k=1}^{+\infty} \bigcup_{m\in I } \{ (S,\omega),~ \rho_S (\xi_k(\omega), \xi_m(\omega)) \leq \delta \} \bigg) \\ & \cap \bigg( \bigcap_{i<i_1} \bigcup_{k=1}^{+\infty} \bigcap_{m\in F_i^N } \{ (S,\omega),~ \rho_S (\xi_k(\omega), \xi_m(\omega)) > \delta \} \bigg).
    \end{aligned}
    $$
    By Theorem \ref{th:covering_numbers_measurability_random_w}, we have the measurability of $(S,\omega) \mapsto |N_\delta^{\rho_S}(W(\omega))|$, hence the measurability result follows by continuity of $\ell$ (and therefore $\rho_\cdot(\cdot,\cdot)$) and countable unions and intersections.

    Now, using Lemma \ref{lemma:coverings_closure_property}, $N_\delta(S,\omega)$ also defines a covering of $W$ and we have:

    \begin{equation}
    \upperbox^{\rho_S} (W(\omega)) = \limsup_{\delta \to 0} \frac{|N_\delta (S,\omega)|}{\log(1/\delta)}.
    \end{equation}

\end{proof}

Let us make the following important remark , which summarizes most of this subsection.

\begin{remark}

Theorem \ref{th:measurable_coverings} shows that we can construct measurable coverings of the random closed hypothesis set under pseudo-metric $\rho_S$. While those coverings may not be strictly speaking minimal, they yields the same upper-box counting dimensions, which is enough for all proofs in this work to hold. Note that this technical complication of not being minimal comes from the fact that we asked the minimal coverings of a set $F$ to be included in $F$, however this also removes further technical complications. If we do not impose this condition, our proof would imply that we can construct measurable minimal coverings.

\end{remark}

\textbf{From now on, we will always implicitly assume that the coverings we consider are measurable and induce correct upper box-counting dimension, the present subsection being a theoretical basis for this assumption.This is formalized by Assumption \ref{coverings_measurability_assumption} in the main part of the paper.}

\section{Postponed proofs}

\label{Postponed proofs}

\subsection{Proof of Theorem \ref{main_result_fixed_hypothesis_space}}
\label{subsection:proof_fixed_W}

This proof essentially uses classical arguments related to Rademacher complexity.

\begin{proof}

\textbf{Step $0$:} First of all, as $\mathcal{W}$ is closed, we can consider a dense countable subset $\mathcal{C}$\footnote{the fact that we cover a dense countable subset and not directly $\mathcal{W}$ here is just made to invoke the measurability result of Theorem \ref{th:covering_numbers_measurability_fixed_w}, it does not change anything to the proof.}. Thanks to the boundedness assumption, we can find finite coverings $N_r$ for each value of $r > 0$. The notation $N_r$ refers in this proof to the set of the centers of a covering of $\mathcal{C}$ by closed $r$-balls under the pseudo-metric $\rho_S$. Invoking results from Lemma \ref{lemma:coverings_closure_property}, those set $N_r$ are also $\delta$-coverings of $\mathcal{W}$ and induce the upper box-counting dimension of $\mathcal{W}$ under $\rho_S$, so that considering them does not change the dimension.

\textbf{Step $1$:} Let us set:
$$
G(S) := \sup_{w \in \mathcal{W} } \big( \mathcal{R}(w) - \er_S(w) \big).
$$
Invoking proposition \ref{rademacher_generalization} we have:

\begin{equation}
    \label{prop91_rademacher}
    G(S)  \leq 2 \rad(\ell(\mathcal{W}, S))+ 3  \sqrt{\frac{2B^2}{n} \log(1/\eta)}.
\end{equation}

\textbf{Step $2$:}

    Therefore we have everywhere for $S \in \mathcal{Z}^n$:
    \begin{equation}
    \label{eq:appendix_proof_fixed_W_limit}
    \upperbox^{\rho_S} (\mathcal{W}) := \limsup_{r \to 0} \frac{\log(\vert N_r \vert)}{\log(1/r)}.
    \end{equation}

    Thanks to Theorem \ref{th:covering_numbers_measurability_fixed_w} we have that $\log(|N_r|)$ is a random variable. Let us consider an arbitrary positive sequence $r_k$ decreasing and converging to $0$.  We have:
    \begin{equation}
    \label{eq:fixed_w_countable_limit}
    \upperbox^{\rho_S} (\mathcal{W}) := \limsup_{k \to \infty} \frac{\log(\vert N_{r_k} \vert)}{\log(1/r_k)}.
    \end{equation}
    
    Let $\gamma > 0$, by Egoroff's Theorem \ref{egoroff} there exist a set $\Omega_\gamma$ such that $\mu_z^{\otimes n} (\Omega_\gamma) \geq 1 - \gamma$, on which the above convergence is uniform. Therefore, if we fix $\epsilon > 0$, we have that there exists $K \in \mathds{N}$ such that
    $$
    \forall S \in \Omega_\gamma,~ \forall  k \geq K, ~\sup_{0<\delta<r_k}\frac{\log(\vert N_{\delta} \vert)}{\log(1/\delta)} \leq \epsilon + \upperbox^{\rho_S} (\mathcal{W}).
    $$
    Now, setting $\delta_{n,\gamma,\epsilon} := r_K$, we have that on $\Omega_\gamma$:
    \begin{equation}
    \label{eq:use_of_egoroff}
    \forall \delta \leq  \delta_{n,\gamma,\epsilon}, ~\log(| N_{\delta} | ) \leq (\epsilon + \upperbox^{\rho_S} (\mathcal{W})) \log(1/\delta).
    \end{equation}

    Now let us fix $S \in \Omega_\gamma$ and the associated cover $N_r$, for $(\sigma_i)$ Rademacher random variables independent of $S$ and $N_r$, taking two points $w,w'$ such that $\rho_S(w,w') \leq r$ we can use the triangle inequality and write:
    $$
    \frac{1}{n} \sum_{i=1}^n \sigma_i \ell(w,z_i) \leq r + \frac{1}{n} \sum_{i=1}^n \sigma_i \ell(w',z_i) .
    $$
    Therefore we have:
    $$
    \rad(\ell(\mathcal{W}, S)) \leq r + \mathds{E}_\sigma \bigg[ \max_{w \in N_r} \frac{1}{n} \sigma^T \ell(w, S) \bigg].
    $$

    As the Rademacher random variables are independent of the other random variables we have by Massart's lemma (lemma \ref{lemma:Massart}):
    $$
    \rad(\ell(\mathcal{W}, S)) \leq r + B \sqrt{\frac{2 \log(\vert N_r \vert)}{n}}.
    $$

Therefore if we take $\delta \leq \delta_\gamma$ we get that with probability at least $1 - \gamma$:
\begin{equation}
    \label{prop9_1_step3}
    \rad(\ell(\mathcal{W}, S)) \leq \delta + \mathds{E}_\sigma \bigg[ \max_{w \in N_r} \frac{1}{n} \sigma^T \ell(w, S) \bigg] \leq \delta + B \sqrt{\frac{2\log(1/\delta)}{n} (\epsilon + \upperbox^{\rho_S} (\mathcal{W}))} .
\end{equation}

Putting together equations \ref{prop91_rademacher} and \ref{prop9_1_step3} we get that with probability at least $1 - 2\eta - \gamma$, for $\delta \leq \delta_{n,\gamma, \epsilon}$:
\begin{equation}
\label{eq:appendix_proof_fixed_W_final}
G(S) \leq 2\delta + 2B \sqrt{\frac{4 (\epsilon + d(S)) \log(1/\delta) + 9 \log(1/\eta)}{n}}.
\end{equation}

\end{proof}

\begin{remark}
    An important remark can be made at this point. One can see that $\delta$ is still appearing in Equation \eqref{eq:appendix_proof_fixed_W_final}, this is due to the possible lack of uniformity in the limit defined in Equation \eqref{eq:appendix_proof_fixed_W_limit}. That way the quantity $\log(1/\delta_{n,\gamma, \epsilon})$ may be seen as a sort of speed of convergence of the upper box-counting dimension. Theorem \ref{main_result_fixed_hypothesis_space}, as well as our other main results (Theorems \ref{main_result_HP_bound_convering_MI} and \ref{main_result_hp_bound_with_coverings_stability}) may be made uniform in $n$ by further assumption of uniformity in $n$ on the convergence of the limit defining the upper box-counting dimension, i.e. Equation \eqref{eq:appendix_proof_fixed_W_limit}, meaning that in that case $\delta_{n,\gamma,\epsilon}$ will not depend on $n$. This would allow us to proceed as in the proof of Lemma $S1$ in \citep{simsekli_hausdorff_2021} and set $\delta = \delta_n := 1/\sqrt{n}$, at the cost of making the bound asymptotic in $n$.
\end{remark}

\subsection{Proof of Theorem \ref{main_result_HP_bound_convering_MI}}

Before going to the proof, let us make a few remarks on the introduced approximated level sets of the empirical risk.

to be able to develop a covering argument, we first cover the set $\mathcal{W}_{S,U}$ by using the pseudo-metric $\rho_S$ and rely on the following decomposition: for any $\delta >0$ and $w' \in N_\delta^{\rho_S}(\mathcal{W}_{S,U})$ we have that 
$$
    \mathcal{R}(w) - \hat{\mathcal{R}}_S(w) \leq \mathcal{R}\left(w^{\prime}\right) - \hat{\mathcal{R}}_S\left(w^{\prime}\right) + |\hat{\mathcal{R}}_S(w) 
     -\hat{\mathcal{R}}_S\left(w^{\prime}\right) | + 
     \left|\mathcal{R}(w)-\mathcal{R}\left(w^{\prime}\right)\right|.
$$
In the above inequality, the first term can be controlled by standard techniques, namely concentration inequalities and decoupling theorems presented in Section \ref{subsection:information_theoretic_quantities} as $w'$ lives in a finite set $N_\delta^{\rho_S}(\mathcal{W}_{S,U})$ and the second term is trivially less than $\delta$ by the definition of coverings. However, the last term cannot be bounded in an obvious way.
To overcome this issue we introduce `approximate level-sets' of the population risk, defined as follows\footnote{As $U$ is independent of $S$, we drop the dependence on it to ease the notation.} for some $K \in \mathds{N}_+$:
   \begin{equation}
   \label{eq:approximate_level_sets_appendix}
   R_S^j := \mathcal{W}_{S,U} \cap \mathcal{R}^{-1} \bigg( \bigg[\frac{jB}{K}, \frac{(j+1)B}{K} \bigg] \bigg),
   \end{equation}
   where $j = 0,\dots, K-1$ and $\mathcal{R}^{-1}$ denotes the inverse image of $\mathcal{R}$. The interval $\big[\frac{jB}{K}, \frac{(j+1)B}{K} \big]$ will be denoted $I_j$. Note that thanks to the

Let $N_{\delta,j}$ collect the centers of a minimal $\delta$-cover of $R_S^j$ 
relatively to $\rho_S$, the measurability condition on the coverings extend to the randomness of those sets $N_{\delta,j}$.

\begin{remark}
\label{rq:ugly_empty_level_sets}
    Without loss of generality we can always assume that those sets $R_S^j$ are non-empty. Indeed we can always add one deterministic point of $\mathcal{R}^{-1}(I_j)$ in each of the coverings $N_{\delta,j}$ one deterministic (always the same) element of $\mathcal{R}^{-1}(I_j)$. It won't make the mutual information term appearing in our result bigger (by the data-processing inequality) and it won't change the upper box-counting dimension because of its finite stability. Moreover if some of the sets $\mathcal{R}^{-1}(I_j)$ are empty then we just need to restrict ourselves to a deterministic subset of $[0, B]$. If we don't want to do this, another way, maybe cleaner, of handling the potential empty sets would be to use the convention $\max(\emptyset) = 0$ everywhere in the proof, then we should also adapt the definition of $\epsilon(N,I)$ below to replace $\log(KN)$ by $\max(0, \log(KN))$, where $\log(0)$ is set to $-\infty$. All those manipulations would essentially lead to the same results. 
\end{remark}

\textbf{Measurability of the coverings:} We proved in Section~\ref{subsection:technical_lemmas} that we can construct measurable coverings (as random sets), which are actually coverings of a dense countable subset (or a Castaing's representation) of $\mathcal{W}_{S,U}$. Therefore, without loss of generality and thanks to the continuity of the loss $\ell$, we can assume in all the remaining of this work that all the considered coverings are random sets, because either they can be constructed by Theorem \ref{th:measurable_coverings} or we can restrict ourselves to Castaing's representations of $\mathcal{W}_{S,U}$.

As can already be noted in Remark \ref{rq:ugly_empty_level_sets}, our approximate level set technique introduces quite a lot of technical difficulties and intricate terms. We believe that this proof technique is interesting but may not be a definitive answer to the problem at hand, improving it is a direction for future research.

\begin{proof}
We assume without loss of generality that that the loss takes values in $[0,B]$.

Let us fix some integer $K \in \mathds{N}_+$ and define $I_j = [\frac{jB}{K}, \frac{(j+1)B}{K}]$, such that:
$$
[0,B] = \bigcup_{j=0}^{K-1} I_j.
$$
Then, given $\mathcal{W}_S$ we define the set $R_S^j := \mathcal{W}_S \cap \mathcal{R}^{-1}(I_j)$.

We then introduce the random closed (finite) sets \footnote{Note that, as mentioned earlier, in this paper we always assume that minimal coverings are random sets.} $N_{\delta, j}$ corresponding to the centers of a minimal covering\footnote{Without loss of generality we can always assume that those sets are non-empty. Indeed we can always add one deterministic point of $\mathcal{R}^{-1}(I_j)$ in each of the coverings $N_{\delta,j}$ one deterministic (always the same) element of $\mathcal{R}^{-1}(I_j)$. It won't change the mutual information term in the final results (by the data-processing inequality) and it won't change the upper box-counting dimension because of its finite stability. Moreover if some of the sets $\mathcal{R}^{-1}(I_j)$ are empty then we just need to restrict ourselves to a deterministic subset of $[0, B]$. If we don't want to do this, another way, maybe cleaner, of handling the potential empty sets would be to use the convention $\max(\emptyset) = 0$ everywhere in the proof, then we should also adapt the definition of $\epsilon(N,I)$ below to replace $\log(KN)$ by $\max(0, \log(KN))$, where $\log(0)$ is set to $-\infty$.}  of $R_S^j$, such that $N_{\delta,j} \subset R_S^j$, for the pseudo-metric:
$$
\rho_S(w,w') := \frac{1}{n} \sum_{i=1}^n \vert \ell(w,z_i) - \ell(w', z_i) \vert.
$$

The first step is to write that almost surely:
$$
\sup_{w \in \mathcal{W}_S } \big( \mathcal{R}(w) - \er_S(w) \big) = \max_{0 \leq j \leq K-1} \sup_{w \in R_S^j} \big( \mathcal{R}(w) - \er_S(w) \big).
$$
Then, given $w,w' \in R_S^j$ such that $\rho_S(w,w') \leq \delta$ we have by the triangle inequality:
\begin{equation}
    \label{prop93_step_1}
\begin{aligned}
    \big( \mathcal{R}(w) - \er_S(w) \big) &\leq \big( \mathcal{R}(w') - \er_S(w') \big) + \rho_S(w,w') + \vert \mathcal{R}(w) - \mathcal{R}(w')\vert \\
    &\leq \big( \mathcal{R}(w') - \er_S(w') \big) + \delta+ \frac{B}{K}. 
\end{aligned}
\end{equation}
So that we get:
\begin{equation}
    \sup_{w \in \mathcal{W}_S } \big( \mathcal{R}(w) - \er_S(w) \big) \leq  \delta+ \frac{B}{K} + \max_{0 \leq j \leq K-1}  \max_{w \in N_{\delta,j}} \big( \mathcal{R}(w) - \er_S(w) \big).
\end{equation}

Now we fix some $\eta > 0$ and just introduce the random variable $\epsilon$ as a function of two variables $N$ and $I$:
    $$
    \epsilon(N, I) := \sqrt{\frac{2B^2}{n} \bigg( \log(1/\eta) + \log \big( KN\big)  + I\bigg)}.
    $$
We have by the decoupling lemma \ref{lemma:probability_decoupling_hodgkinson} along with Fubini's Theorem, Hoeffding inequality and a union bound:
    \begin{equation}
    \label{prop93_hoeffding}
    \begin{aligned}
        \mathds{P} \bigg( \max_{0 \leq j \leq K-1} \max_{w \in N_{\delta,j}} \big( \mathcal{R}(w) - & \er_S(w) \big) \geq \epsilon (\max_j \vert N_{\delta,j} \vert ,~ \max_j I_{\infty} (S, N_{\delta,j}) ) \bigg) \\
        &\leq \sum_{j=0}^{K-1 }  \mathds{P} \bigg(\max_{w \in N_{\delta,j}} \big( \mathcal{R}(w) - \er_S(w) \big) \geq \epsilon (|N_{\delta,j}|, I_{\infty} (S, N_{\delta,j}) ) \bigg) \\
        &\leq \sum_{j=0}^{K-1 } e^{ I_{\infty} (S, N_{\delta,j})} \mathds{P}_{N_{\delta,j}} \otimes \mathds{P}_S \bigg(\max_{w \in N_{\delta,j}} \big( \mathcal{R}(w) - \er_S(w) \big) \geq \epsilon (|N_{\delta,j}|, I_{\infty} (S, N^{j}_{\delta}) ) \bigg) \\
        &\leq \sum_{j=0}^{K-1 } e^{ I_{\infty} (S, N_{\delta,j})} \mathds{E}_{N_{\delta,j}} \bigg[ \mathds{P}_S \bigg(\max_{w \in N_{\delta,j}} \big( \mathcal{R}(w) - \er_S(w) \big) \geq \epsilon (|N_{\delta,j}|, I_{\infty} (S, N^{j}_{\delta}) ) \bigg) \bigg]\\
        &\leq \sum_{j=0}^{K-1 } e^{ I_{\infty} (S, N_{\delta,j})} \mathds{E}_{N_{\delta,j}} \bigg[ \sum_{w \in N_{\delta,j}} \mathds{P}_S \bigg( \big( \mathcal{R}(w) - \er_S(w) \big) \geq \epsilon (|N_{\delta,j}|, I_{\infty} (S, N^{j}_{\delta}) ) \bigg) \bigg]\\
        &\leq \sum_{j=0}^{K-1 } e^{ I_{\infty} (S, N_{\delta,j})} \mathds{E}_{N_{\delta,j}} \bigg[ |N_{\delta,j}| \exp \bigg\{ -\frac{n\epsilon (|N_{\delta,j}|, I_{\infty} (S, N^{j}_{\delta}) )^2}{2 B^2} \bigg\} \bigg] \\
        &\leq \sum_{j=0}^{K-1 } e^{ I_{\infty} (S, N_{\delta,j})} \mathds{E}_{N_{\delta,j}} \bigg[ \frac{\eta}{K } e^{ - I_{\infty} (S, N_{\delta,j})} \bigg] \\
        & = \eta.
    \end{aligned}
    \end{equation}

Now let us consider a random minimal $\delta$-cover of the whole (random) hypothesis set $\mathcal{W}_S$. Given $j \in \{ 0, \dots, K-1 \}$, we have in particular almost surely that:
$$
\mathcal{W}_S \cap R_j \subseteq \bigcup_{w \in N_\delta} B^{\rho_S}_\delta (w).
$$
Where $B^{\rho_S}_\delta (w)$ denotes the closed $\delta$-ball for metric $\rho_S$ centered in $w$. Therefore there exists a non-empty subset $\Tilde{N}_\delta \subseteq N_\delta$ such that for all $w \in \Tilde{N}_\delta$ we have $B^{\rho_S}_\delta (w) \cap R_j \neq \emptyset$. 

Therefore we can collect in some set $\Tilde{N}_{\delta,j}$ one element in each  $B^{\rho_S}_\delta (w) \cap R_j$ for $w \in \Tilde{N}_\delta$ and the triangular inequality gives us:
$$
R_S^j \subseteq \bigcup_{w \in \Tilde{N}_{\delta,j}} B^{\rho_S}_{2\delta} (w).
$$
This proves that almost surely $\forall j,~ \vert N_{\delta,j} \vert\leq \vert N_{\delta/2} \vert $, and thus:
\begin{equation}
    \label{prop93_covering_numbers_inequality}
   \max_{0 \leq j \leq K-1} \vert N_{\delta, j} \vert \leq \vert N_{\delta/2} \vert .
\end{equation}

We know that we have almost surely that:
$$
\limsup_{\delta \to 0} \frac{\log(\vert N_{\delta/2} \vert)}{\log(2/\delta)} = \upperbox^{\rho_S} (\mathcal{W}_S).
$$

Therefore let us fix $\gamma, \epsilon > 0$. Using Egoroff's Theorem we can say that there exists $\delta_{n,\gamma, \epsilon} > 0$ such that, with probability at least $1 - \gamma$, for all $\delta \leq \delta_{n,\gamma, \epsilon}$ we have:
$$
\log \big( \vert N_{\delta/2} \vert \big) \leq (\epsilon + \upperbox^{\rho_S} (\mathcal{W}_S)) \log(2/\delta).
$$

Therefore combining equations \ref{prop93_step_1}, \ref{prop93_hoeffding}, \ref{prop93_covering_numbers_inequality}, we get that with probability at least $1 - \gamma - \eta$, for all $\delta \leq \delta_{n,\gamma, \epsilon}$:
$$
\begin{aligned}
\sup_{w \in \mathcal{W}_S } \big( \mathcal{R}(w) - \er_S(w) \big) &\leq \delta + \frac{B}{K} + \sqrt{\frac{2B^2}{n} \bigg( \log(K/\eta) + \log \big(\max_j \vert N_{\delta,j} \vert\big) + \max_j I_{\infty} (S, N_{\delta,j})\bigg)} \\
&\leq\delta + \frac{B}{K} + \sqrt{\frac{2B^2}{n} \bigg( \log(K/\eta) + \log  \vert N_{\delta/2} \vert  + \max_j I_{\infty} (S, N_{\delta,j})\bigg)} \\
&\leq \delta + \frac{B}{K} + \sqrt{\frac{2B^2}{n} \bigg( \log(K/\eta) + \log(2/\delta) (\epsilon + \upperbox^{\rho_S} (\mathcal{W}_S)) + \max_j I_{\infty} (S, N_{\delta,j})\bigg)} .
\end{aligned}
$$

The choice of $K$ has not been done yet, considering the above equation the best choice is clearly: $K = K_n := \lfloor \sqrt{n} \rfloor$. Let us introduce the notation:
$$
I_{n,\delta} := \max_j I_{\infty} (S, N_{\delta,j}).
$$
This way we get that with probability at least $1 - \gamma - \eta$, for all $\delta \leq \delta_{n,\gamma, \epsilon}$:
\begin{equation}
    \label{prop93_final}
    \sup_{w \in \mathcal{W}_S } \big( \mathcal{R}(w) - \er_S(w) \big) \leq  \delta + \frac{B}{\sqrt{n} - 1}
 + \sqrt{2}B \sqrt{\frac{ \log(\sqrt{n}/\eta) + \log(2/\delta) (\epsilon+ \upperbox^{\rho_S} (\mathcal{W}_S)) +I_{n,\delta}} {n}}.
\end{equation}
Note that it is possible to set this value of $K$, which depends on $n$, at the end of the proof, because the previous limits do not depend on $K$.

\end{proof}

\subsection{Proof of Theorem \ref{main_result_hp_bound_with_coverings_stability}}

Here we present the proof of Theorem \ref{main_result_hp_bound_with_coverings_stability}. The proof proceeds in two steps and is based on what we will call a \emph{grouping technique}. The main idea is to divide the dataset $S\in \mathcal{Z}^n$ into $H$ groups $J_1,\dots,J_H$ of size $J$ with $J, H \in \mathds{N}_+$ and $JH = n$. In the end of the proof a particular choice is made. 

A minor technical difficulty appears when it is not actually possible two write $JH = n$ for a pertinent choice of $(J,H)$. Therefore we first present a result when the latter is possible and then derive two corollaries to deal with this technical issue, mostly based on the boundedness assumption. Theorem \ref{main_result_hp_bound_with_coverings_stability} will be the second corollary.

\begin{remark}
    For the sake of the proof we need to assume $\alpha \leq \frac{3}{2}$, which is just asking for a potentially weaker assumption, which is not a problem. Note that the value $\alpha \leq \frac{3}{2}$ will lead in Theorem \ref{main_result_hp_bound_with_coverings_stability} to a convergence rate in $n^{-1/2}$ which is optimal anyway.
\end{remark}

Let us start with the main result of this section:

\begin{proposition}[]
    \label{high_proba_bound_small_mutal_information}
    We make assumptions \ref{bounded_continuous_assumption}, \ref{coverings_measurability_assumption} and \ref{local_covering_stability_assumption} with the same notations than in Theorem \ref{main_result_hp_bound_with_coverings_stability}. We also take arbitrary $J, H \in \mathds{N}_+$ such that $JH = n$

    Then for all $n\geq 2^{\frac{3}{3 - 2\alpha}}$, with probability $1 - \gamma - \eta$, for all $\delta$ smaller than some $\delta_{\gamma, \epsilon, n} > 0$ we have:

    $$
     \sup_{w\in \mathcal{W}_{S,U}} |\risk(w) - \er_S(w)| \leq \delta + \frac{B}{\sqrt{n} - 1} + \frac{2J\beta}{n^\alpha} 
 +H\sqrt{\frac{JB^2}{2n^2} \bigg( \big( \epsilon +d(S,U) \big) \log(4/\delta) + \log(H\sqrt{n}/\eta) + I \bigg)}.
    $$
    
\end{proposition}

\begin{proof}

Let us first refine our notations for the coverings to make the proof clearer. Throughout this section, for any $S,S'$  we will denote $N_\delta(S,S',U)$ the centers of a covering of $\mathcal{W}_{S,U}$ by closed $\delta$-balls under pseudo-metric $d_{S'}$. As in the proof of Theorem \ref{main_result_HP_bound_convering_MI}, we introduce some approximate level sets $R_S^j$ for $j \in \{ 0, \dots, K-1 \}$. We then denote by $N_{\delta,j}(S,S',U)$ the centers of a covering of $R_S^j$ by closed $\delta$-balls under pseudo-metric $d_{S'}$. (note that the $R_S^j$ still depends on $U$ but the dependence has been dropped to ease the notations).

The remark we made in the proof of theorem \ref{main_result_HP_bound_convering_MI} about the assumptions that $R_S^j$ are non-empty without loss of generality still holds in the setting described hereafter.

The proof starts by introducing the "level-sets" of the population risk as in proof of Theorem \ref{main_result_HP_bound_convering_MI}. We define $R_S^j$ exactly in the same way.

The proof starts with the same statement:
$$
\sup_{w\in \mathcal{W}_{s,U}} |\risk(w) - \er_S(w)| \leq \max_{0 \leq j \leq K-1} \sup_{w\in R_S^j} |\risk(w) - \er_S(w)| .
$$
For all $j$, we (minimally) cover $R_S^j$ with $\delta$-covers for pseudo-metric $d_S$, such that the centers are in $R_S^j$. We collect those centers in $N_{\delta,j}(S,S,U)$.

This leads us to:
\begin{equation}
    \label{prop12_2_step_1}
    \sup_{w\in \mathcal{W}_{s,U}} |\risk(w) - \er_S(w)| \leq \delta + \frac{B}{K} + \max_{0 \leq j \leq K-1} \underbrace{\max_{w \in N_{\delta,j} (S,S,U) } |\risk(w) - \er_S(w)| }_{:= E_j}.
\end{equation}

Thanks to our stability assumption \ref{local_covering_stability_assumption}, we can say that for $\delta$ small enough there exists a random minimal covering such that for all $j \in \{0,\dots,K-1\}$ and all $k \in \{ 1, \dots, H \}$ the covering $N_{\delta,j} (S,S^{\backslash J_k},U) $ satisfies:
$$
\forall w \in N_{\delta,j} (S,S,U),~ \exists w' \in N_{\delta,j} (S,S^{\backslash J_k},U), ~\sup_{z\in \mathcal{Z}} |\ell(w,z) - \ell(w',z)| \leq \frac{\beta J}{n^\alpha},
$$
where the $J$ factor on the right hand side comes from the fact that our stability assumption can be seen as a Lipschitz assumption in term of the Hausdorff distance of the coverings with respect to the Hamming distance on the datasets.

Recall that we assume that all $N_{\delta,j}$ have coverings-metrics stability with common parameters $\beta, \alpha$.

As in the previous proposition, we split the index set $\{ 1,\dots,n \}$ into $H$ groups of size $J$, with $HJ = n$, which allows us to write (with a similar proof):
$$
\begin{aligned}
    E_j &= \max_{w \in N_{\delta,j} (S,S,U) } |\risk(w) - \er_S(w)| \\
    &\leq \max_{w \in N_{\delta,j} (S,S,U) }\sum_{k=1}^H \frac{1}{n} \bigg| 
\sum_{i \in J_k} \big( \ell(w,z_i) - \risk(w) \big) \bigg| \\
    &\leq  \sum_{k=1}^H \max_{w \in N_{\delta,j} (S,S,U) } \frac{1}{n} \bigg| 
\sum_{i \in J_k} \big( \ell(w,z_i) - \risk(w) \big) \bigg| \\
   &\leq \sum_{k=1}^H \bigg\{ \frac{2\beta J^2}{n^{1+\alpha}} +  \frac{1}{n} \max_{w \in N_{\delta,j} (S,S^{\backslash J_k},U) }  \bigg| \sum_{i \in J_k} \big( \ell(w,z_i) - \risk(w) \big) \bigg| \bigg\} \\
    &= \frac{2J\beta}{n^\alpha} +  \frac{1}{n}  \sum_{k=1}^H  \max_{w \in N_{\delta,j} (S,S^{\backslash J_k},U) }  \bigg|  \sum_{i \in J_k} \big( \ell(w,z_i) - \risk(w) \big) \bigg| .
\end{aligned}
$$

Putting this back into equation \eqref{prop12_2_step_1} we get:
\begin{equation}
    \label{prop12_2_step_2}
    \begin{aligned}
        \sup_{w\in \mathcal{W}_{s,U}} |\risk(w) - \er_S(w)| &\leq \delta + \frac{B}{K} + \frac{2J\beta}{n^\alpha} + \max_{0 \leq j \leq K-1}  \sum_{k=1}^H  \max_{w \in N_{\delta,j} (S,S^{\backslash J_k},U) } \frac{1}{n}  \bigg|  \sum_{i \in J_k} \big( \ell(w,z_i) - \risk(w) \big) \bigg|  \\
        &\leq \delta + \frac{B}{K} + \frac{2J\beta}{n^\alpha} + H \max_{0 \leq j\leq  K-1} ~ \max_{1 \leq k \leq H} ~ \underbrace{\max_{w \in N_{\delta,j} (S,S^{\backslash J_k},U) }  \frac{1}{n}\bigg|  \sum_{i \in J_k} \big( \ell(w,z_i) - \risk(w) \big) \bigg| }_{:= M_{j,k}(S,U)}.
    \end{aligned}
\end{equation}

Let $\epsilon$ be a random variable depending on $N_{\delta,j} (S,S^{\backslash J_k},U) $ only. We use a decoupling lemma (lemma $1$ in \cite{hodgkinson_generalization_2022}) along with Hoeffding's inequality to write:

\begin{equation}
    \label{prop12_2_step_3}
    \begin{aligned}
    \mathds{P} \big( M_{j,k}(S,U) \geq \epsilon \big) &\leq  e^{I_\infty (N_{\delta,j} (S,S^{\backslash J_k},U) , S_{J_k})} \mathds{P}_{N_{\delta,j} (S,S^{\backslash J_k},U) } \otimes \mathds{P}_{S_{J_k}}  \big( M_{j,k}(S,U) \geq \epsilon \big) \\
    &\leq  e^{I_\infty (N_{\delta,j} (S,S^{\backslash J_k},U) , S_{J_k})} \mathds{E}_{N_{\delta,j} (S,S^{\backslash J_k},U) } \bigg[ \mathds{P}_{S_{J_k}}  \big( M_{j,k}(S,U) \geq \epsilon \big) \bigg] \\
    &\leq  e^{I_\infty (N_{\delta,j} (S,S^{\backslash J_k},U) , S_{J_k})}\\ & \quad \times \mathds{E}_{N_{\delta,j} (S,S^{\backslash J_k},U) } \bigg[ \mathds{P}_{S_{J_k}}  \bigg( \bigcup_{w \in N_{\delta,j} (S,S^{\backslash J_k},U) } \bigg\{ \frac{1}{n}\bigg|  \sum_{i \in J_k} \big( \ell(w,z_i) - \risk(w) \big) \bigg| \geq \epsilon  \bigg\} \bigg) \bigg] \\
    &\leq e^{I_\infty (N_{\delta,j} (S,S^{\backslash J_k},U) , S_{J_k})} \mathds{E} \bigg[ |N_{\delta,j} (S,S^{\backslash J_k},U) | e^{-\frac{2 \epsilon^2 n^2}{JB^2}} \bigg].
    \end{aligned}
\end{equation}

The key point of the proof, and the reason for which we have introduced this strong stability assumption on the coverings is that we can now use the following Markov chain:
\begin{equation}
    \label{prop12_2_markov_chain}
    S_{J_k} \longrightarrow \mathcal{W}_{S,U} \longrightarrow N_{\delta,j} (S,S^{\backslash J_k},U) .
\end{equation}

Therefore, by the data processing inequality:

$$
I_\infty (N_{\delta,j} (S,S^{\backslash J_k},U) , J_{J_k}) \leq I_\infty ( \mathcal{W}_{S,U}, S_{J_k}).
$$

Now using the easier Markov chain:

$$
\mathcal{W}_{S,U} \longrightarrow S \longrightarrow S_{J_k},
$$

We have:
\begin{equation}
    \label{prop12_2_mutual_info_inequatliy}
    I_\infty (N_{\delta,j} (S,S^{\backslash J_k},U) , S_{J_k}) \leq I_\infty (S, \mathcal{W}_{S,U}).
\end{equation}

Note that the mutual information term appearing in equation \eqref{prop12_2_mutual_info_inequatliy} is the same than the one appearing in \cite{hodgkinson_generalization_2022}.

Thus:
$$
 \mathds{P} \big( M_{j,k}(S,U) \geq \epsilon \big) \leq e^{I_\infty (S, \mathcal{W}_{S,U})} \mathds{E} \bigg[ |N_{\delta,j} (S,S^{\backslash J_k},U) | e^{-\frac{2 \epsilon^2 n^2}{JB^2}} \bigg].
$$

Equipped with this result we can make an informed choice for the random variable $\epsilon$, for a fixed $\eta > 0$:
$$
\epsilon = \epsilon_{j,k} := \sqrt{\frac{JB^2}{2n^2} \bigg( \log | N_{\delta,j} (S,S^{\backslash J_k},U) | + \log(HK/\eta) + I_\infty (S, \mathcal{W}_{S,U}) \bigg)},
$$

Now we can apply an union bound to get:
$$
\begin{aligned}
 \mathds{P} \big( \max_{0 \leq j \leq K-1} ~ \max_{1 \leq k \leq H} ~ M_{j,k}(S,U) \geq \max_{0 \leq j \leq K-1} ~ \max_{1 \leq k \leq H} ~  \epsilon_{j,k} \big) &\leq \sum_{j=0}^{K-1} \sum_{k = 1}^H \mathds{P} \big( M_{j,k}(S,U) \geq \max_{0 \leq j \leq K-1} ~ \max_{1 \leq k \leq H} ~  \epsilon_{j,k} \big) \\
 &\leq \sum_{j=0}^{K-1} \sum_{k = 1}^H \mathds{P} \big( M_{j,k}(S,U) \geq  \epsilon_{j,k} \big) \\
 &= \eta.
\end{aligned}
$$

Now let us have a closer look at those covering numbers $| N_{\delta,j} (S,S^{\backslash J_k},U) |$. Note that we have:
$$
\forall w,w' \in \mathds{R}^d,~ d_{S^{\backslash J_k}} (w,w') \leq \frac{n}{n-J} d_S(w,w'),
$$
And therefore $| N_{\delta,j} (S,S^{\backslash J_k},U) | \leq | N_{\frac{\delta(n-J)}{n},j } (S,S,U) |$. 

Moreover, using the same reasoning than in the proof of Theorem \ref{main_result_HP_bound_convering_MI}, we know that we have $| N_{\delta,j} (S,S^{\backslash J_k},U) | \leq | N_{\delta/2} (S,S^{\backslash J_k},U) |$.

Thus:
$$
| N_{\delta,j} (S,S^{\backslash J_k},U) | \leq | N_{\frac{\delta(n-J)}{2n} } (S,S,U) |.
$$

As before, we will want to solve the trade-off in the values of $H$ and $J$ by setting $J = n^\lambda$ for some $\lambda \in (0,1)$ (this time we do not allow the value $\lambda = 1$, which will be justified later when we find the actual value of $\lambda$). A very simple calculation gives us:
$$
\frac{\delta (n - J)}{2n} = \frac{\delta}{2} \bigg( 1 - \frac{1}{n^{1 - \lambda}} \bigg).
$$

Therefore we can say that if $n\geq 2^{\frac{1}{1 - \lambda}}$, then $\frac{\delta (n - J)}{2n} \geq \delta/4$ and therefore:
\begin{equation}
\label{prop12_2_covering_numbers inequality}
| N_{\delta,j} (S,S^{\backslash J_k},U) | \leq | N_{\frac{\delta}{4} } (S,S,U) |.
\end{equation}

We know that:

$$
\upperbox^{d_S} (\mathcal{W}_{S,U}) = \limsup_{\delta \to 0} \frac{ | N_{\frac{\delta}{4} } (S,S,U) |}{\log(4/\delta)}.
$$

If we fix $\epsilon, \gamma > 0$, we can apply Egoroff's Theorem to write that with probability $1 - \gamma$, we have for $\delta$ small enough:
$$
| N_{\frac{\delta}{4} } (S,S,U) | \leq \big( \epsilon + \upperbox^{d_S} (\mathcal{W}_{S,U}) \big) \log(4/\delta).
$$

Therefore, we can say that with probability $1 - \eta - \gamma$, we have for $\delta$ small enough:
\begin{equation}
\label{prop12_2_bound_before_tradeoff}
\begin{aligned}
 \sup_{w\in \mathcal{W}_{s,U}} |\risk(w) - \er_S(w)| \leq \delta &+ \frac{B}{K} + \frac{2J\beta}{n^\alpha}  \\
 &+H\sqrt{\frac{JB^2}{2n^2} \bigg( \big( \epsilon + \upperbox^{d_S} (\mathcal{W}_{S,U}) \big) \log(4/\delta) + \log(HK/\eta) + I_\infty (S, \mathcal{W}_{S,U}) \bigg)}.
\end{aligned}
\end{equation}

Setting $K = \lfloor \sqrt{n} \rfloor$ and noting that $1 - \alpha/3 \leq 1$ in the above equation gives us the result.

\end{proof}

\begin{corollary}
\label{tradeoff_simple_case}
    With the exact same setting than in proposition \ref{high_proba_bound_small_mutal_information}, if we assume in addition that $n^{\alpha/3} \in \mathds{N}_+$, then for all $n\geq 2^{\frac{3}{3 - 2\alpha}}$, with probability $1 - \gamma - \eta$, for all $\delta$ smaller than some $\delta_{\gamma, \epsilon, n} > 0$ we have:
    $$
     \sup_{w\in \mathcal{W}_{s,U}} |\risk(w) - \er_S(w)| \leq \delta + \frac{B + 2\beta}{n^{\alpha/3}}  +B\sqrt{\frac{ \log(1/\eta) + \big( 1 - \frac{\alpha}{3} \big)\log(n) + I + \big( \epsilon + d(S,U)\big) \log(4/\delta)}{2n^{\frac{2\alpha}{3}}}}.
    $$
    
\end{corollary}

\begin{proof}
We want to write J in the form $J = n^\lambda$ with some $\lambda > 0$. We see that there is a trade-off to be solved in the values of $(J,H)$ if we want both all terms in equation \eqref{prop12_2_bound_before_tradeoff} to have the same order of magnitude in $n$, which leads to $H\sqrt{J}/n = J/n^\alpha$. Therefore we want to have $1/\sqrt{J} = J/n^\alpha$ and $\lambda/2 = \alpha - \lambda$, which implies the following important formula:

\begin{equation}
    \label{prop12_2_lambda}
    \lambda = \frac{2\alpha}{3}.
\end{equation}

Finally, we are left again with choosing the value of $K$, an obvious choice is $K = n^{\alpha/3} \in \mathds{N}_+$ to get the same order of magnitude. Thus we get the final result: for $n\geq 2^{\frac{3}{3 - 2\alpha}}$, with probability $1 - \gamma - \eta$, for all $\delta$ smaller than some $\delta_{\gamma, \epsilon, n} > 0$ we have:

\begin{equation}
    \label{prop12_2_final}
    \begin{aligned}
 \sup_{w\in \mathcal{W}_{s,U}} |\risk(w) - \er_S(w)| \leq \delta + \frac{B + 2\beta}{n^{\alpha/3}}  +B \bigg\{ \frac{ \log(1/\eta) + \big( 
 1 - \frac{\alpha}{3} \big)\log(n) + I + \big( \epsilon + d(S,U) \big) \log(4/\delta)}{2n^{\frac{2\alpha}{3}}} \bigg\}^{\frac{1}{2}}.
 \end{aligned}
\end{equation}

\end{proof}

\begin{remark}
    The asymptoticity in $\delta$ defined by $\delta_{\gamma, \epsilon, n}$ above accounts for the asymptoticity coming both from the stability assumption (definition \ref{coverings_stability}) and the convergence of the limit defining the upper box-counting dimension.
\end{remark}

Now we prove Theorem \ref{main_result_hp_bound_with_coverings_stability} which is based on the same idea than the previous corollary, but when $n^{\alpha/3} \notin \mathds{N}$.

\begin{theorem}
\label{main_result_imbalenced}
under the same assumptions and notations than proposition \ref{high_proba_bound_small_mutal_information}. We have that for $n \geq C(\alpha) := \max \{ 2^{\frac{3}{2\alpha}}, 2^{1+\frac{3}{3 - 2\alpha}} \} $, with probability $1 - \gamma - \eta$, for all $\delta$ smaller than some $\delta_{\gamma, \epsilon, n} > 0$ we have:

\begin{equation}
 \sup_{w\in \mathcal{W}_{s,U}} |\risk(w) - \er_S(w)| \leq \delta + \frac{3B + 2\beta}{n^{\alpha/3}}  +B\sqrt{\frac{ \log(1/\eta) + \big( 
 1 - \frac{\alpha}{3} \big)\log(n) + I +  \big( \epsilon + d(S,U) \big) \log(4/\delta)}{2n^{\frac{2\alpha}{3}}}}.
\end{equation}
\end{theorem}

\begin{proof}
    We define $J := \lfloor n^{2\alpha/3 \rfloor}$, $J := \lfloor n^{1 - 2\alpha/3} \rfloor$ and $\Tilde{n} := JH$. We obviously have $\Tilde{n} \leq n$.
    
    Using the boundedness assumption we have:
    \begin{equation}
        \label{imbalanced_step1}
        | \er_S(w) - \mathcal{R}(w)| \leq \frac{n - \Tilde{n}}{n} B + \frac{\Tilde{n}}{n} \bigg| \frac{1}{\Tilde{n}}\sum_{i=1}^{\Tilde{n}} \ell(w,z_i) - \mathcal{R}(w) \bigg|.
    \end{equation}

    For the first term we write:
    $$
    \frac{n - \Tilde{n}}{n} B \leq \frac{n - \big( n^{2\alpha/3} - 1 \big) \big( n^{1 - 2\alpha/3} - 1 \big)}{n} = \frac{n^{2\alpha/3} + n^{\alpha/3} -1}{n} \leq \frac{2B}{n^{\alpha/3}}.
    $$

    The idea is to apply the proof of Theorem \ref{high_proba_bound_small_mutal_information} to the last term of equation \eqref{imbalanced_step1}, replacing $d_{S_n}$ with $d_{S_{\Tilde{n}}}$. For clarity we still denote $S = (z_1,\dots,z_n)$ and $S_{\Tilde{n}} = (z_1,\dots, z_{\Tilde{n}})$
    
    There are several terms we need to consider:
    \newline
    \textbf{The mutual information term:} The two data processing inequality we apply to prove equation \eqref{prop12_2_mutual_info_inequatliy} still apply so we can still write $I_\infty (S, \mathcal{W}_{S,U})$ in the bound.
    \newline
    \textbf{Dimension term:} Let us denote by $d(S,S', U)$ the upper-box dimension of $\mathcal{W}_{S,U} $ for pseudo-metric $d_{S'}$. Using the same reasoning than equation \eqref{prop12_2_covering_numbers inequality}, we have:
    $$
    |N_\delta (S,S_{\Tilde{n}}, U)| \leq |N_{\delta\frac{\Tilde{n}}{n}}(S,S,U)|.
    $$
    We have:
    $$
    \delta\frac{\Tilde{n}}{n} \geq \delta  \frac{\big( n^{2\alpha/3} - 1 \big) \big( n^{1 - 2\alpha/3} - 1 \big)}{n} \geq \delta \bigg( 1 - \frac{1}{n^{2\alpha/3}} \bigg).
    $$
    And therefore, once we have $n \geq 2^{\frac{3}{2\alpha}}$ we have:
    $$
    |N_\delta (S,S_{\Tilde{n}}, U)| \leq |N_{\frac{\delta}{2}}(S,S,U)|,
    $$
    which implies:
    $$
    d(S,S_{\Tilde{n}, U} \leq d(S,S,U).
    $$
    \newline
    \textbf{Terms in $n$}: Now we look at equation \eqref{prop12_2_bound_before_tradeoff}, where we have $4$ types of term in $n$ which are of the form:
    \begin{itemize}
        \item $1/K$,
        \item $H\sqrt{J}/n$,
        \item $\sqrt{\log(HK)} H\sqrt{J}/n$,
        \item $J/n^\alpha$.
    \end{itemize}
    We do not forget that we also have to multiply those terms by the factor $\Tilde{n}/n$ coming from equation \eqref{imbalanced_step1}. Setting $K := \lfloor 1 + \sqrt{J} \rfloor$ we get successively:
    $$
    \frac{\Tilde{n}}{n} \frac{1}{K} \leq \frac{1}{n^{\alpha/3}}, \quad \frac{\Tilde{n}}{n} H\sqrt{J}/n \leq \frac{1}{n^{\alpha/3}}, \quad \frac{\Tilde{n}}{n} J/n^\alpha \leq \frac{1}{n^{\alpha/3}}.
    $$
    For the logarithmic term we have:
    $$
    \log(HK) \leq \log(2 \sqrt{J} n^{1 - 2\alpha/3}) \leq \log(2 n^{1 - \alpha/3}).
    $$
    Moreover, if $n \geq 2^{\frac{3}{2\alpha}}$ we have:
    $$
    \Tilde{n} \geq \big( n^{2\alpha/3} - 1 \big) \big( n^{1 - 2\alpha/3} \big) \geq n/2.
    $$
    Therefore the condition $\Tilde{n} \geq 2^{\frac{3}{3 - 2\alpha}}$ is implied by $n/2 \geq  2^{\frac{3}{3 - 2\alpha}}$. So now the condition on $n$ becomes:
    \begin{equation}
        \label{imbalanced_n_condition}
        n \geq C(\alpha) := \max \{ 2^{\frac{3}{2\alpha}}, 2^{1+\frac{3}{3 - 2\alpha}} \} .
    \end{equation}
    
    Putting all of this together, we get that for $n\geq C(\alpha)$ (defined in equation \eqref{imbalanced_n_condition}), with probability $1 - \gamma - \eta$, for all $\delta$ smaller than some $\delta_{\gamma, \epsilon, n} > 0$ we have:

\begin{equation}
    \label{imbalanced_final}
    \begin{aligned}
 \sup_{w\in \mathcal{W}_{s,U}} |\risk(w) - \er_S(w)| \leq \delta + \frac{3B + 2\beta}{n^{\alpha/3}}  +B\sqrt{\frac{ \log(1/\eta) + \big( 
 1 - \frac{\alpha}{3} \big)\log(n) + I +  \big( \epsilon + d(S,U) \big) \log(4/\delta)}{2n^{\frac{2\alpha}{3}}}}.
\end{aligned}
\end{equation}

\end{proof}

\subsection{Proof of Theorem \ref{dim_equality_pseudo_metric_spaces}}

Let $(X, \rho)$ be a pseudo-metric space, we introduce the equivalence relation: 
$$
x \sim y \iff \rho(x,y) = 0.
$$
We call metric identification of $X$ the quotient of $X$ by this equivalence relation. The canonical projection on the quotient will be denoted as:
$$
\pi : X \longrightarrow X/\sim.
$$
$\rho$ induces a metric on $X/\sim$ that we will denote $\rho^\star = \pi _\star \rho$.

We prove that upper box-counting dimension and persistent homology dimension are invariant by this identification operation. Let us recall that we always consider the covers are made from closed $\delta$-balls, even though equivalent definitions exist.

\begin{lemma}[Upper-box dimension with pseudo metric]
\label{lemma:upc_pseudo_metric_space}

\begin{equation}
    \label{pseudo_upper_box}
    \upperbox (X) = \upperbox (X/\sim).
\end{equation}
\end{lemma}
Let $N_{\delta}^d(F)$ denote the minimum number of \textbf{closed $\delta$-balls} coverings of $F$ for the (pseudo)-metric $d$.
\begin{proof}
    Let $F \subset X$, bounded. 
    Let $\{ x_1,\dots, x_n \}$ be the centers of a closed $\delta$-balls covering of $F$ for metric $\rho$. We have:
    $$
    \forall x, x' \in B(x_i, \delta),~ \rho^\star (\pi(x), \pi(y')) = \rho(x, x') \leq \delta.
    $$
    Therefore $\pi(B(x_i, \delta)) \subset B(\pi(x_i), \delta)$, therefore $N_{\delta}^{\rho}(F) \geq N_{\delta}^{\rho^\star}(\pi(F))$.
    
    On the other hand, if $\{ y_1, ..., y_n \}$ are the centers of a covering of $\bar{F} \subset X/\sim$, a similar reasoning shows that the $\pi^{-1}(B(y_i, \delta))$ give a covering of $\pi^{-1}(F)$ with (set included in) $\delta$-balls.
\end{proof}

The result is also quite obvious for the persistent homology dimension, even though it is a bit more complicated to write it. For more details on persistent homology please refer to \cite{boissonat_geometrical_2018, memoli_primer_2019, schweinhart_persistent_2019}.

\begin{lemma}[Persistent homology dimension in pseudo metric spaces]

\label{pseudo_ph_dim_lemma}

    \begin{equation}
    \label{pseudo_ph_dim}
    \dimph (X) = \dimph(X/\sim).
    \end{equation}

\end{lemma}

Intuitively, the proof of this result is as follows: When constructing the VR filtration in a pseudo-metric space, points within $0$ pseudo-distance will only add pairs of the form $(0,0)$ in their persistence homology of degree $0$, because they are created with the same value of the distance parameter $\delta$ in construction of the VR filtration.

\begin{proof}

    Let $K$ be a simplicial complex based on a finite point set $T \in X$. Let us denote by $\Tilde{K} := \pi(K)$ the image of $K$ by the canonical projection $\pi : X \longrightarrow X/\sim$, defined by its value on the simplices:
    \begin{equation}
    \label{eq:pi_on_simplices}
    \pi([a_0,\dots,a_s]) := [\pi(a_0), \dots, \pi(a_s)].
    \end{equation}
    We also introduce a \emph{section} of $\pi$, i.e. an injective application $s:X/\sim \longrightarrow X$, such that $\pi \circ s = \text{Id}_{X/\sim}$.
    Clearly, $\tilde{K}$ is still a simplicial complex.The map $\pi$ does not preserve the dimension of the simplices, as $[\pi(a_0), \dots, \pi(a_s)]$ is seen as a set, and two $a_i$ can have the same image, but $\pi$ always reduces the dimension.

    Note that $\Tilde{K}$ and $s(\Tilde{K})$ clearly define simplicial complex, but that $s(\Tilde{K})$ can only be seen as a sub-complex of $K$. Therefore, we define $s: \Tilde{K} \longrightarrow K$ analogously to Equation \eqref{eq:pi_on_simplices}. Actually, by injectivity of $s$, this allows us to identify $\Tilde{K}$ with a sub-complex of $K$.

    Thus, both $\pi$ and $s$ linear maps on the space of $k$-chains:
    $$
    \pi : C_k(K) \longrightarrow C_k(\Tilde{K}),\quad s: C_k(\Tilde{K}) \longrightarrow C_k(K),
    $$ 
    which both commute with the boundary operator, indeed, for any simplex $[a_0,\dots,a_s]$ and $\epsilon_i \in \mathds{K}$:
    $$
    \begin{aligned}
        \pi \circ \partial ([a_0,\dots,a_s]) &= \pi \bigg( \sum_{i=0}^s \epsilon_i[a_0, \dots,a_{i-1}, a_{i+1}, \dots, a_s] \bigg) \\
        &= \sum_{i=0}^s \epsilon_i [\pi(a_0), \dots,\pi(a_{i-1}), \pi(a_{i+1}), \dots, \pi(a_s)] \\
        &= \partial \circ \pi ([a_0,\dots,a_s]),
    \end{aligned}
    $$
    with the exact same computation for $s$, so that the following diagram commutes:
    $$
     \xymatrix{
    C_1(K) \ar[r]^{\partial} \ar@/^/[d]^\pi  & C_0(K) \ar[r]^\partial \ar@/^/[d]^\pi & \{ 0 \} \ar@<2pt>[d] \\
    C_1(\Tilde{K}) \ar[r]^\partial \ar@/^/[u]^s  & C_0(\Tilde{K})  \ar[r]^\partial   \ar@/^/[u]^s & \{ 0 \} \ar@<2pt>[u]
  }
    $$

Therefore, $\pi$ and $s$ induces linear maps between the homology groups, making the following diagram commute:

$$
     \xymatrix{
    C_0(K) \ar@/^/[r]^\pi \ar[d]  & C_0(\Tilde{K}) \ar@/^/[l]^s \ar[d] \\
    H_0(K) \ar@/^/[r]^{\bar{\pi}}   & H_0(\Tilde{K})   \ar@/^/[l]^{\bar{s}}
  }
    $$

Now let us consider $P = \{ x_1,\dots, x_n \}$ a finite set in $(X,\rho)$ and denote accordingly $\Tilde{P} := \pi(P)$. Let us introduce a Vietoris-Rips filtration of $P$ denoted by:
$$
\emptyset \rightarrow K^{\delta_0, 1} \rightarrow \dots \rightarrow K^{\delta_0, \alpha_0} \rightarrow K^{\delta_1, 1} \rightarrow \dots \rightarrow K^{\delta_c, \alpha_C} = K,
$$
where $0\leq \delta_1 < \dots < \delta_C$ are the `time-distance' indices of the filtration and for the same value of $\delta$ the simplices are ordered by their dimension and arbitrarily if they also have the same dimension. Obviously $\delta_0 = 0$.

As $\pi : P \longrightarrow \Tilde{P}$ preserves distances, it is clear that, up to allowing certain complexes to appear several times in a row, the nested sequence $(\Tilde{K}^{i,j})_{(0\leq i\leq C, 1\leq j \leq \alpha_i)}$ is a Vietoris-Rips filtration for $\Tilde{P}$.

Let us fix some $i \in {0,\dots, C}$ and $j \in {1, \dots,\alpha_i}$ such that either $i\leq 1$ or $j = \alpha_0$. This way we have: 
$$
\forall a,b \in P,~\pi(a) = \pi(b) \implies [a,b] \in K^{i,j} , 
$$
by definition of the VR filtration (all simplices within $\delta_0 = 0$ $\rho$-distance have been added in the filtration). Therefore, if $\pi(a) = \pi(b)$, as $\partial [a,b] = [a] + [b]$, we have that $\overline{[a]} = \overline{[b]}$ in $H_0(K^{i,j})$. As be definition of $s$, for any $a \in P$ we have $\pi \circ s \circ \pi (a) = \pi(a)$, we have the following identity (the bars denote classes in homology groups):
$$
\bar{s} \circ \bar{\pi} ([a]) = \overline{s \circ \pi ([a])} = \overline{[a]}.
$$

Therefore, as also $\pi \circ s = \text{Id}$, we have that $\bar{s}$ and $ \bar{\pi}$ are inverse of one another, so that we have an isomorphism $H_0(K^{i,j}) \cong H_0(\Tilde{K}^{i,j})$ and the following diagram:

\newcommand{\eq}[1][r]
   {\ar@<-3pt>@{-}[#1]
    \ar@<-1pt>@{}[#1]|<{}="gauche"
    \ar@<+0pt>@{}[#1]|-{}="milieu"
    \ar@<+1pt>@{}[#1]|>{}="droite"
    \ar@/^2pt/@{-}"gauche";"milieu"
    \ar@/_2pt/@{-}"milieu";"droite"}

$$
 \xymatrix{
    H_0(K^{0,1}) \ar[r] \ar[d] & \dots \ar[d] \ar[r] & H_0(K^{0,\alpha_0 - 1})\ar[r] \ar[d]& H_0(K^{0,\alpha_0 })\ar[r] \eq[d]& \dots \eq[d] \ar[r] & H_0(K^{\delta_C,\alpha_C})  \eq[d] \\
    H_0(\Tilde{K}^{0,1}) \ar[r]  & \dots \ar[r] & H_0(\Tilde{K}^{0,\alpha_0 - 1}) \ar[r] & H_0(\Tilde{K}^{0,\alpha_0 })\ar[r] & \dots \ar[r] & H_0(\Tilde{K}^{\delta_C,\alpha_C}) 
  }
$$

As already mentioned, persistent homology of degree $0$ is characterized by the multi-set of `death times' $\delta_i$. All death before $K^{0, \alpha_0 - 1}$ are $0$ so they do not add anything the weighted life-sum of Equation \eqref{eq:weighted_alpha_sum}. After $K^{0, \alpha_0 - 1}$, the isomorphisms in the diagram show that the basis will evolve exactly in the same way so the death times will be the same, therefore the weighted sum are the same in both spaces for any $P$. Therefore, by definition, we have the equality between the persistent homology dimension.

\end{proof}

Combination of Equation \eqref{eq:dim_equality_metrics_spaces}, lemma \ref{lemma:upc_pseudo_metric_space} and Lemma \ref{pseudo_ph_dim_lemma} immediately gives the proof of Theorem \ref{dim_equality_pseudo_metric_spaces}.

\subsection{Proof of Theorem \ref{thm:lower_bound}}
In this subsection, we show how we can leverage very classical tools from high dimensional probability to give one first step toward proving lower bounds, even though the obtained lower bound may look a bit disappointing.
We combine two tools, namely Gaussian complexity and Sudakov's theorem. 

\begin{definition}[Gaussian complexity]

Given a set $A \subset \mathds{R}^n$, and $g_1,\dots,g_n \sim \mathcal{N}(0,1)$ independent, the Gaussian complexity of $A$ is defined by:
$$
\Gamma(A) := \frac{1}{n} \mathds{E}_g \bigg[ \sup_{a\in A} \sum_{i=1}^n g_i a_i  \bigg]
$$ 
\end{definition}

As before we will denote, for $S \in \mathcal{Z}^n$:

$$
\Gamma(\ell(\mathcal{W}, S)) := \frac{1}{n} \mathds{E}_g \bigg[ \sup_{w\in \mathcal{W}} \sum_{i=1}^n g_i \ell(w,z_i)  \bigg]
$$ 

We have the following lower bound of Rademacher complexity

\begin{lemma}
    \label{lemma:rad_gauss_lower_bound}
    We have:
    $$
    \rad(A) \geq \frac{1}{2\sqrt{\log(n)}} \Gamma(A)
    $$
\end{lemma}

\begin{proof}

For Rademacher random variables $\sigma_1, \dots, \sigma_n$, let us define the following function, for $\alpha = (\alpha_1, \dots, \alpha_n) \in [0,1]^n $:
$$
f(\alpha) := \mathds{E}_{\sigma} \bigg[ \sup_{a \in A} \sum_{i=1}^n \sigma_i \alpha_i a_i \bigg].
$$
it is easy to see that $f$ is convex and continuous on the compact set $[0,1]^n$. Therefore we know that $f$ attains its maximum for some $\alpha^0 \in [0,1^n.]$. We denote the constant one vector by $\mathds{1}_n \in \mathds{R}^n$. For some $0\leq \lambda\leq 1$, let us write $\alpha^0 = \lambda \alpha + (1 - \lambda) \mathds{1}^n$, for some $\alpha$. We have, by convexity:
$$
f(\alpha^0) \leq \lambda f(\alpha) + (1 - \lambda)f(\mathds{1}^n) \leq \lambda f(\alpha^0) + (1 - \lambda) \rad(A),
$$
which implies that:
\begin{equation}
    \label{eq:lemma_13_1_step1}
     \mathds{E}_{\sigma} \bigg[ \sup_{a \in A} \sum_{i=1}^n \sigma_i \alpha_i a_i \bigg] \leq \rad(A)
\end{equation}

Now let $g_1,\dots,g_n \sim \mathcal{N}(0,1)$ be independent normal random variables. Let $g_\infty := \max_i(|g_i|)$. It is possible to write the following decomposition: $\forall i,~ g_i = |g_i| \sigma_i$ where the $\sigma_i$ are Rademacher random variables \textbf{independent of $|g_i|$}.

As $g_\infty>0$ almost surely, we have:

$$
\begin{aligned}
\Gamma(A) &:= \frac{1}{n} \mathds{E}_g \bigg[ \sup_{a\in A} \sum_{i=1}^n g_i a_i  \bigg] \\ 
&\leq \frac{1}{n} \mathds{E}_{g_\infty} \bigg[ g_\infty  \mathds{E}_{\sigma} \bigg[ \sup_{a\in A} \sum_{i=1}^n \sigma_i \frac{|g_i|}{g_\infty} a_i \bigg]  \bigg], \quad \text{(because $\rad$ is non negative)} \\ 
&\leq  \frac{1}{n} \mathds{E}_{g_\infty} \bigg[ g_\infty  \mathds{E}_{\sigma} \bigg[ \sup_{a\in A} \sum_{i=1}^n \sigma_i  a_i \bigg]  \bigg], \quad \text{(by Equation \eqref{eq:lemma_13_1_step1})} \\
&=  \frac{1}{n} \mathds{E}_{g_\infty} [ g_\infty ]\rad(A), \quad \text{(Fubini's theorem)} .
\end{aligned}
$$
We conclude by using that $\mathds{E}[g_\infty] \leq 2 \sqrt{\log(n)}$.

\end{proof}

\begin{remark}
    It is also possible to prove that $\rad(A) \leq \sqrt{\frac{\pi}{2}} \Gamma(A)$
\end{remark}

The key ingredient for the lower bound is Sudakov's theorem, see \citep[Section $7$]{vershynin_high-dimensional_2020}:

\begin{theorem}[Sudakov's theorem]
    \label{thm:sudakov}
    Let $(X_t)_{t\in T}$ be a mean zero gaussian process, then for any $\epsilon > 0$ we have:
    $$
    \mathds{E} \bigg[ \sup_{t\in T} X_t \bigg] \geq C \delta \sqrt{N_\delta (T,d)},
    $$
    where $C$ is an absolute constant and $N_\epsilon(T,d)$ the covering number of $T$ for the following pseudo-metric:
    $$
    d(t,s)^2 := \mathds{E}[(X_t - X_s)^2]
    $$
\end{theorem}

To prove our lower bound, we first need a lower bound of the expected worst case generalization error in terms of Rademacher complexity:

\begin{proposition}{Desymmetrization inequality}{}
    \label{prop:desymmetrization}
    Assume that the loss $\ell$ is in the interval $[0,B]$ (this does not make us lose generality). Then we have:
    $$
    \mathds{E} \bigg[\sup_{w \in \mathcal{W}} \big| \mathcal{R}(w) - \er_S(w) \big| \bigg] \geq \frac{1}{2} \mathds{E}\big[ \rad(\ell(\mathcal{W}, S)) \big] - B \sqrt{\frac{\log(2)}{2n}}.
    $$
\end{proposition}

While this result is classical, we present a proof for the sake of completeness, and to exhibit the absolute constants that we get in our case.

\begin{proof}

Similarly to the symmetrization inequality, we write, with $(z_i')_i$ an independent copy of $(z_i)_i$ and $(\sigma_i)_i$ independent Rademacher random variables:

$$
\begin{aligned}
    \mathds{E}\big[ \rad(\ell(\mathcal{W}, S)) \big] 
        &\leq \mathds{E} \bigg[ \frac{1}{n} \sup_{w \in \mathcal{W}} \sum_{i=1}^n \sigma_i \big( \ell(w,z_i) - \mathcal{R}(w) \big)\bigg] + \mathds{E} \bigg[ \frac{1}{n} \sup_{w \in \mathcal{W}} \sum_{i=1}^n \sigma_i  \mathcal{R}(w) \bigg] \\
        &\leq \mathds{E} \bigg[ \frac{1}{n} \sup_{w \in \mathcal{W}} \sum_{i=1}^n \sigma_i \big( \ell(w,z_i) - \ell(w, z_i') \big)\bigg] + B \mathds{E} \bigg[ \frac{1}{n} \bigg| \sum_{i=1}^n \sigma_i \bigg| \bigg], \quad \text{(Jensen's inequality)} \\
        &\leq \mathds{E} \bigg[ \frac{1}{n} \sup_{w \in \mathcal{W}} \sum_{i=1}^n \big( \ell(w,z_i) - \ell(w, z_i') \big)\bigg] + B \mathds{E} \bigg[ \frac{1}{n} \bigg| \sum_{i=1}^n \sigma_i \bigg| \bigg], \quad \text{(Symmetrization argument)} \\
        &\leq 2 \mathds{E} \bigg[\sup_{w \in \mathcal{W}} \big| \mathcal{R}(w) - \er_S(w) \big| \bigg] + B \mathds{E} \bigg[ \frac{1}{n} \bigg| \sum_{i=1}^n \sigma_i \bigg| \bigg], \quad \text{(Triangle inequality)} \\
        &\leq 2 \mathds{E} \bigg[\sup_{w \in \mathcal{W}} \big| \mathcal{R}(w) - \er_S(w) \big| \bigg] + B \sqrt{\frac{2 \log(2)}{n}}, \quad \text{(Simple case of Massart's lemma)},
\end{aligned}
$$
hence the result.
    
\end{proof}

Then we can prove the following result:

\begin{proposition}{Lower bound in term of covering numbers}{}
    Assume that the loss $\ell$ is bounded by $B>0$, then there is an absolute constant $c>0$ (the one coming from Sudakov theorem) such that with probability at least $1 - \zeta$, for all $\delta > 0$ we have
    $$
    \sup_{w \in \mathcal{W}} \big| \mathcal{R}(w) - \er_S(w) \big|  \geq \frac{c}{4}\sqrt{\frac{\delta^2 \log|N_\delta^{\rho_S}(\mathcal{W})|}{n \log(n)}} - B \sqrt{\frac{\log(2) + 9 \log(1/\zeta)}{n}},
    $$
    where $\rho_S$ is the data-dependent metric already used before in this project (based on an $L^1$ empirical mean).
\end{proposition}

\begin{proof}
    Using the same reasoning, based on Mc-Diarmid's inequality, than for proving the upper bound, we write successively that with probability at least $1 - \zeta$:
    $$
    \begin{aligned}
         \sup_{w \in \mathcal{W}} \big| \mathcal{R}(w) - \er_S(w) \big|  &\geq \mathds{E} \bigg[   \sup_{w \in \mathcal{W}} \big| \mathcal{R}(w) - \er_S(w) \big| \bigg] - B \sqrt{\frac{2 \log(1/\zeta)}{n}}, \quad \text{(Mc-Diarmid's inequality)} \\
        &\geq  \frac{1}{2} \mathds{E}\big[ \rad(\ell(\mathcal{W}, S)) \big] - B \sqrt{\frac{\log(2)}{2n}} -  \sqrt{\frac{2 \log(1/\zeta)}{n}}\\
        &\geq  \frac{1}{2} \rad(\ell(\mathcal{W}, S))- B \sqrt{\frac{\log(2)}{2n}} -  \frac{3}{2}\sqrt{\frac{2 \log(1/\zeta)}{n}} , \quad \text{(Mc-Diarmid's inequality)} \\
        &\geq \frac{1}{4\sqrt{\log(n)}}  \Gamma(\ell(\mathcal{W}, S)) - B \sqrt{\frac{\log(2)}{2n}}  -   3\sqrt{\frac{ \log(1/\zeta)}{2n}}\\
    \end{aligned}
    $$

    Now we note that:  
    $$
    \Gamma(\ell(\mathcal{W}, S)) := \frac{1}{n} \mathds{E}_g \bigg[ \sup_{w\in \mathcal{W}} \sum_{i=1}^n g_i \ell(w,z_i)  \bigg],
    $$
    and introduce the following gaussian process:
    $$
    \forall w \in \mathcal{W},~ X_w := \frac{1}{\sqrt{n}} \sum_{i=1}^n g_i \ell(w,z_i).
    $$
    The $L^2$ distance induced by this gaussian process on $\mathcal{W}$ can be computed by:
    $$
    \begin{aligned}
        d(w,w')^2 &= \frac{1}{n} \mathds{E} \bigg[ \bigg(\sum_{i=1}^n g_i (\ell(w,z_i) - \ell(w', z_i)) \bigg)^2\bigg] \\
        &= \frac{1}{n} \sum_{i=1}^n (\ell(w,z_i) - \ell(w', z_i))^2 \\
        &\geq \rho_S(w,w')^2, \quad \text{(Cauchy-Schwarz's inequality)}
    \end{aligned}
    $$
    where 
    $$
    \rho_S(w,w'):= \frac{1}{n}\sum_{i=1}^n |\ell(w,z_i) - \ell(w', z_i)|
    $$
    is the data-dependent pseudo-metric we used previously in this work. The result then follows by applying Sudakov's theorem.
    
\end{proof}

Using this proposition, we can prove the following result:

\begin{theorem}[Lower bound with data-dependent fractal dimension]
    Assume that the loss $\ell$ is bounded by $B>0$ and that almost surely we have $\lowerbox^{\rho_S}(\mathcal{W}) > 0$. Then, for all $\gamma, \zeta > 0$ there is an absolute constant $c>0$ and some $\delta_{n,\gamma,\zeta} > 0$ such that, with probability at least $1 - \zeta - \gamma$, for all $\delta \leq \delta_{n,\gamma,\zeta}$ we have:
    $$
    \sup_{w \in \mathcal{W}} \big| \mathcal{R}(w) - \er_S(w) \big|  \geq \frac{c}{4}\sqrt{\frac{\delta^2 \log(1/\delta)d(S)}{2n \log(n)}} - B \sqrt{\frac{\log(2) + 9 \log(1/\zeta)}{n}}.
    $$
\end{theorem}

\begin{remark}
    As many of our results, the more interesting part of this result is the underlying covering numbers bound, the annoying asymptoticity in $\delta$ being introduced when we go from the covering numbers to the data-dependent fractal dimensions.
\end{remark}

\begin{proof}
    Let us fix $
    \gamma, \zeta \in (0,1)$.
    Using the definition of the lower box-counting dimension and the fact that $\lowerbox^{\rho_S}(\mathcal{W})>0$ almost surely, we can write:
    $$
    \liminf_{\delta \to 0} \frac{\log |N_\delta^{\rho_S}(\mathcal{W})|}{\lowerbox^{\rho_S}(\mathcal{W}) \log(1/\delta)} = 1,
    $$
    we can invoke Egoroff's theorem, as in previous proofs, to argue that there exists $\Omega_\gamma \in \mathcal{F}^{\otimes n}$, such that $\mu_z^{\otimes n}(\Omega_\gamma) \geq 1 - \gamma$, on which the above convergence is uniform. As $\lowerbox^{\rho_S}(\mathcal{W})>0$ almost surely, we can assume without loss of generality that this is also the case on $\Omega_\gamma$.
    
    This implies that, on $\Omega_\gamma$, for $\delta$ smaller than some $\delta_{n,\gamma,\zeta}$, we have:
    $$
    \log |N_\delta^{\rho_S}(\mathcal{W})| \geq \frac{1}{2} \log(1/\delta)\lowerbox^{\rho_S}(\mathcal{W}) .
    $$
    Then, the result immediately follows from the previous proposition.
\end{proof}

\subsection{Lipschitz case}
\label{sec:Lipschitz_case}

As mentioned in the introduction, several authors \citep{simsekli_hausdorff_2021, camuto_fractal_2021, hodgkinson_generalization_2022} have proven worst-case generalization bounds involving the Hausdorff dimension of the hypothesis set, computed based on the Euclidean distance. Their is in particular based on a `Lipschitz loss $\ell$' assumption. It is therefore natural to ask whether we can find a similar result, i.e. involving the Euclidean based dimension, from our results.

Therefore, in this section, we assume that the function $(w,z) \longmapsto \ell(w,z)$ is $L$-Lipschitz in $w$, uniformly with respect to $z$, with $L>0$ a constant.

To simplify, we demonstrate the case of a fixed hypothesis set $\mathcal{W} \subset \mathds{R}^d$, more general cases being derived in a similar fashion.

Then we can bound the pseudo-metric $\rho_S$ by ($\Vert \cdot \Vert$ is the Euclidean norm):
$$
\begin{aligned}
    \forall w,w' \in \mathds{R}^d,~ \rho_S(w,w') &= \frac{1}{n} \sum_{i=1}^n |\ell(w,z_i) - \ell(w', z_i)| \\
    &\leq \frac{1}{n} \sum_{i=1}^n L \Vert w - w' \Vert \\
    &= L \Vert w - w' \Vert .
\end{aligned}
$$

From this observation, denoting $N^e$ the coverings associated to the Euclidean metric, we deduce that:
\begin{equation}
    \label{eq:lipschitz_coverings}
    N_\delta^{\rho_S} (\mathcal{W}) \leq N_{\delta/L}^e (\mathcal{W}).
\end{equation}
The proof of Theorem \ref{main_result_fixed_hypothesis_space} therefore leads us to write, instead of Equation \eqref{eq:appendix_proof_fixed_W_final}, that, with probability at least $1 - 2\eta$:
$$
\sup_{w\in\mathcal{W}} \big( \mathcal{R}(w) - \er_S(w) \big) \leq 2 \delta + 2B \sqrt{\frac{2 \log|N_{\delta/L}^e (\mathcal{W})|}{n}} + 3B \sqrt{\frac{2 \log(1/\eta)}{n}}.
$$
Then we can use proof techniques similar to that in \citep{simsekli_hausdorff_2021} in the case of a fixed hypothesis set. More precisely, let us fix some $\epsilon>0$, using the definition of upper box-counting dimension along with Egoroff's theorem, we have that there exists $\Omega_\gamma \in \mathcal{F}^{\otimes n}$, such that $\mu_z^{\otimes n} (\Omega_\gamma) \geq 1 - \gamma$, on which, for $\delta$ smaller than some $\delta_{\gamma, \epsilon}$ (which is independent of $n$, because the metric does not depend on the data anymore) , we have:
$$
\log|N_{\delta/L}^e (\mathcal{W})| \leq \big( \epsilon + \upperbox^e(\mathcal{W}) \big)\log(L/\delta).
$$
Because $\delta_{\gamma, \epsilon}$ does not depend on $n$, it is possible to set $\epsilon = \upperbox^e(\mathcal{W})$ and set:
$$
\delta = \delta_n := \frac{2}{\sqrt{n}}.
$$
Therefore, we have that, with probability $1 - 2\eta - \gamma$ for $n$ big enough:
$$
\sup_{w\in\mathcal{W}} \big( \mathcal{R}(w) - \er_S(w) \big) \leq \frac{4}{\sqrt{n}} + 2B \sqrt{\frac{2 \upperbox^e(\mathcal{W}) \log(L\sqrt{n})}{n}} + 3B \sqrt{\frac{2 \log(1/\eta)}{n}}.
$$
Thus, we recover a result analogous to \citep{simsekli_hausdorff_2021}, up to potentially absolute constants (coming from the fact that the proof technique is different). Their bound can therefore be seen as a particular case of our result.

\section{Additional experimental details}
\label{setcion:Additional experimental details}

\subsection{Granulated Kendall's coefficients}
\label{subsection:kendall_coefficients}

\emph{Kendall's coefficient}, initially introduced in \cite{kendall_new_1938}, is a well-known statistics to assess the co-monoticity of two observations, or rank correlation. It is usually denoted with letter $\tau$.

If we consider $((g_i, d_i)_{1\leq i \leq n})$ a sequence of observation of two random random elements, in our case the generalization error $g$ and the intrinsic dimension $d$. In our setting it is very likely that both $(g_i)$ and $(d_i)$ will have pairwise distinct elements and that ties would therefore have little impact on the analysis. Therefore we will assume it in our presentation to make it easier. To compute Kendall's $\tau$ coefficient, denoted $\tau((g_i)_i, (d_i)_i)$, we look at all the possible pairs of couples $(g_i, d_i)$ and count $1$ if they are ordered the same way and $-1$ otherwise. The coefficients is then normalized by the total number of pairs which is $\binom{n}{2}$. Therefore an analytical formula is:
\begin{equation}
    \label{kendall_tau}
    \tau((g_i)_i, (d_i)_i) = \frac{1}{\binom{n}{2}} \sum_{i<j} \text{sign} (g_i - g_j) \text{sign} (d_i - d_j)
\end{equation}

However, as highlighted in \cite{jiang_fantastic_2019}, vanilla Kendall's $\tau$ may fail to capture any notion of causality in the correlation. Indeed, in our experiments we make vary several hyperparameters (e.g. learning rate $L$ and batch size $B$), we want to somehow measure whether the observed correlation is due to the influence of a hyperparameter on both the generalization error and the persistent homology dimension computation. 

To overcome this issue, we follow the approach of \cite{jiang_fantastic_2019}, whose authors introduced a notion of \emph{granulated Kendall's coefficient}. Let $\Theta_L$ an $\Theta_B$ denote the (finite) set in which our two hyperparameters vary. We first compute $\tau$ coefficients when fixing (all but) one hyperparameter, and then average those coefficients to get the granulated Kendall's coefficients:

\begin{equation}
    \label{granulated_kendall}
    \psi_\eta := \frac{1}{|\Theta_B|} \sum_{b \in \Theta_B} \tau \big((g(\eta,b), d(\eta, b))_{\eta \in \Theta_L} \big), \quad \psi_B := \frac{1}{|\Theta_L|} \sum_{b \in \Theta_L} \tau \big((g(\eta,b), d(\eta, b))_{b \in \Theta_B} \big),
\end{equation}

Where $g(\eta, b)$ and $d(\eta, b)$ denote the generalization and dimension obtained with learning-rate $\eta$ and batch size $b$. We can then average those coefficients to get one numerical measure:
\begin{equation}
    \label{average_granulated_kendall}
    \boldsymbol{\Psi} := \frac{\psi_1 + \psi_2}{2}
\end{equation}

\begin{remark}
    Of course this analysis extends to more than $2$ hyperparameters, but most of our experiments used only learning-rate and batch size.
\end{remark}

We created Python scripts to compute those granulated Kendall's coefficients for all the results presented in this work.

Our analysis also report Spearman's rank correlation coefficient \cite{kendall_advanced_1973}, denoted $\rho$, which is another widely used correlation statistics.

\subsection{Hyperparameters and experimental setting}
\label{subsection:hyperparameters_details}

Here we present some additional experimental details concerning the experiments of the main part of the paper. Note that all experiments were realized using the same random seed while we were making vary the hyperparameters (e.g. learning rate and batch size). For each experiment both hyperparameters vary in a set of $6$ values, making a total of $36$ points if all experiment converge.

All Fully Connected Networks (FCN) have standard ReLU activation.

\textbf{Classification experiments:} We trained FCN-$5$ and FCN-$7$ networks of width $200$ (for each inner layer) on the full training set of MNIST images until we reach $100\%$ accuracy. Learning rate vary in the set $[5.10^{-3}, 10^{-1}]$ and batch size vary in $[32,256]$. 

The stopping criterion in those experiments is reaching $100\%$ accuracy, given that the model is evaluated on all data points every $10000$ iterations. To compute the PH dimension the last $5000$ iterations were considered (this number essentially comes from computational and time constraints). The model was evaluated on each data point for all those $5000$ iterations, producing a point cloud in $\mathds{R}^{5000 \times n}$ Persistent homology was computed on $20$ subset of the generated point cloud with sizes varying in $[1000, 5000]$ in order to apply the method from \cite{birdal_intrinsic_2021}.

Additional classification experiments, presented in Figures \ref{fig:classification}, \ref{fig:mnist_alexnet_appendix} and \ref{fig:mnist_lenet_appendix} and Tables \ref{table:kendall_alexnet_appendix} and \ref{table:kendall_conv_mnist_appendix} involve AlexNet and LeNet networks trained on both MNIST and CIFAR-$10$ dataset within the same ranges of hyperparameters as described above.

\textbf{Regression experiments on California Housing Dataset:} We trained FCN-$5$ and FCN-$7$ of width $200$ (for each inner layer) on a training set corresponding to a random subset of $80\%$ of the $20640$ points of the California Housing Dataset, using the remaining $20\%$ for validation. Learning rate vary in the set $[1.10^{-3}, 10^{-2}]$ and batch size vary in $[32,200]$. 

The stopping criterion for regression experiments is the following: We periodically (every $2000$ iterations in practice) evaluate the empirical risk on the whole training set and stop the training when the relative difference difference between two evaluations becomes smaller than some proportion, set to $0.5\%$ in those experiments. Note that this choice may affect the results. Indeed if we wait to long before stopping the training in a regression experiment, it is possible that the geometry of the point cloud becomes trivial, so we need to ensure convergence while stopping training when the losses $\ell$ are still "moving enough" to get interesting fractal geometry, both for our dimension and the one of \cite{birdal_intrinsic_2021}.

To compute the PH dimension the last $5000$ iterations were considered. The model was evaluated on each data point for all those $5000$ iterations, producing a point cloud in $\mathds{R}^{5000 \times n}$ Persistent homology was computed on $20$ subset of the generated point cloud with sizes varying in $[1000, 5000]$ in order to apply the method from \cite{birdal_intrinsic_2021}.

\textbf{Robustness experiment:} For the robustness experiment presented in figure \ref{fig:robustness}, we used the exact same hyperparameters and random seed than in experiment on MNIST and California Housing Dataset as above. For proportion $\eta$ varying in $[2\%, 10\%, 20\%, \dots, 90\%, 99\%]$ we randomly select a subset $T$ of the dataset $S$ such that $|T|/|S| = \eta$ and compute the PH dimensions corresponding to pseudo-metric $\rho_T$, presented in equation \eqref{losses_based_pseudo_metric}. Note that the PH dimension computation involves sampling different subsets of the last iterates (see above), of course this sampling has been done with the same random seed for all values of $\eta$ so that the observe difference in the dimensional value can only come from the selection of subset $T \subset S$.

\section{Additional experimental results}
\label{section:additional experimental reults}

\subsection{More details on the experiments presented in Section \ref{section:experiments}}
\label{sec:additional_res_from_main_part}

As mentioned above, for the experiments on MNIST and California Housing Dataset we performed $360$ trainings with various seeds, learning rates and batch sizes. This allowed us to compute various statistics, namely granulated Kendall's coefficients $\psi_{\text{lr}}$ and $\psi_{\text{bs}}$ for learning rate and batch size respectively, Average Kendall's coefficient $\boldsymbol{\Psi}$, Kendall's tau $\tau$ and Spearman's rho $\rho$, which are all indicators of correlation. Tables \ref{table:kendall_chd_appendix} and \ref{table:kendall_mnist_appendix} contain all those statistics (same data than the tables in the main part of the paper but with additional coefficients displayed, for space issues). The variation of the seed allows for displaying standard deviation of all those coefficients.

\begin{table}[!h]
\caption{Correlation coefficients on CHD}
\label{table:kendall_chd_appendix}
\vskip 0.15in
\begin{center}
\begin{small}
\begin{sc}
\begin{tabular}{@{} l l l l l l l @{}} 
\toprule
{Model} & {Dim.} & {$\rho$} & {$\psi_{\text{lr}}$} & {$\psi_{\text{bs}}$} & {$\boldsymbol{\Psi}$ } & {$\tau$} \\ 
\midrule
FCN-$5$ & $\dimph^{\text{eucl}}$ & $0.77 \pm 0.08$ & $0.62 \pm 0.11$ & $0.46 \pm 0.14$ & $0.54 \pm 0.11$ & $0.59 \pm 0.07$    \\
FCN-$5$ & $\dimph^{\rho_S}$ & $\mathbf{0.87 \pm 0.05}$ & $\mathbf{0.75 \pm 0.10}$ & $\mathbf{0.61 \pm 0.13}$ & $\mathbf{0.68 \pm 0.10}$  & $\mathbf{0.71 \pm 0.09}$ \\
\midrule
FCN-$7$ & $\dimph^{\text{eucl}}$ & $0.40 \pm 0.09$ & $0.07 \pm 0.13$ & $0.25 \pm 0.11$ &  $0.16 \pm 0.08$ & $0.28 \pm 0.07$   \\
FCN-$7$ & $\dimph^{\rho_S}$ & $\mathbf{0.77 \pm 0.08}$ & $\mathbf{0.63 \pm 0.05}$ & $\mathbf{0.58 \pm 0.10}$ & $\mathbf{0.62 \pm 0.06}$  & $\mathbf{0.77 \pm 0.08}$ \\ 
\bottomrule
\end{tabular}
\end{sc}
\end{small}
\end{center}
\vskip -0.1in
\end{table}

\begin{table}[!h]
\caption{Correlation coefficients on MNIST}
\label{table:kendall_mnist_appendix}
\vskip 0.15in
\begin{center}
\begin{small}
\begin{sc}
\begin{tabular}{@{} l l l l l l l @{}} 
\toprule
{Model} & {Dim.} & {$\rho$} & {$\psi_{\text{lr}}$} & {$\psi_{\text{bs}}$} & {$\boldsymbol{\Psi}$ } & {$\tau$} \\ 
\midrule
FCN-$5$ & $\dimph^{\text{eucl}}$ & $0.62 \pm 0.10$ & $0.78 \pm 0.07$ & $\mathbf{0.80 \pm 0.10}$ &  $0.78 \pm 0.08$ & $0.47 \pm 0.07$   \\
FCN-$5$ & $\dimph^{\rho_S}$ & $\mathbf{0.73 \pm 0.07}$ & $\mathbf{0.84 \pm 0.06}$ & $0.78 \pm 0.10$ & $\mathbf{0.81 \pm 0.07}$  & $\mathbf{0.56 \pm 0.06}$ \\ 
\midrule
FCN-$7$ & $\dimph^{\text{eucl}}$ & $0.80 \pm 0.04$ & $0.92 \pm 0.07$ & $\mathbf{0.85 \pm 0.11}$ & $0.88 \pm 0.04$ & $0.62 \pm 0.04$    \\
FCN-$7$ & $\dimph^{\rho_S}$ & $\mathbf{0.89 \pm 0.02}$ & $\mathbf{0.96 \pm 0.05}$ & $0.84 \pm 0.05$ & $\mathbf{0.90 \pm 0.04}$  & $\mathbf{0.73 \pm 0.03}$ 
\\
\bottomrule
\end{tabular}
\end{sc}
\end{small}
\end{center}
\vskip -0.1in
\end{table}

 In Table \ref{table:kendall_alexnet_appendix} we also report the full metrics on one experiment on AlexNet trained on CIFAR-$10$.

\begin{table}[!h]
\caption{Correlation coefficients with AlexNet on CIFAR-$10$}
\label{table:kendall_alexnet_appendix}
\vskip 0.15in
\begin{center}
\begin{small}
\begin{sc}
\begin{tabular}{@{} l l l l l l l @{}} 
\toprule
{Model} & {Dim.} & {$\rho$} & {$\psi_{\text{lr}}$} & {$\psi_{\text{bs}}$} & {$\boldsymbol{\Psi}$ } & {$\tau$} \\ 
\midrule
AlexNet & $\dimph^{\text{eucl}}$ & $0.86$ &  $0.78$ & $\mathbf{0.84}$ & $0.81$ & $0.68$    \\
AlexNet & $\dimph^{\rho_S}$ & $\mathbf{0.93}$ & $\mathbf{0.87}$ & $0.81$ & $\mathbf{0.84}$ & $\mathbf{0.78}$ \\
\bottomrule
\end{tabular}
\end{sc}
\end{small}
\end{center}
\vskip -0.1in
\end{table}

\begin{table}[!h]
\caption{Correlation coefficients with convolutional models on MNIST}
\label{table:kendall_conv_mnist_appendix}
\vskip 0.15in
\begin{center}
\begin{small}
\begin{sc}
\begin{tabular}{@{} l l l l l l l @{}} 
\toprule
{Model} & {Dim.} & {$\rho$} & {$\psi_{\text{lr}}$} & {$\psi_{\text{bs}}$} & {$\boldsymbol{\Psi}$ } & {$\tau$} \\ 
\midrule
AlexNet & $\dimph^{\text{eucl}}$ & $0.85$ &  $\mathbf{0.78}$ & $\mathbf{0.77}$ & $\mathbf{0.77}$ & $0.67$    \\
AlexNet & $\dimph^{\rho_S}$ & $\mathbf{0.88}$ & $\mathbf{0.78}$ & $\mathbf{0.77}$ & $\mathbf{0.77}$ & $\mathbf{0.70}$ \\
\midrule
LeNet & $\dimph^{\text{eucl}}$ & $0.74$ &  $0.78$ & $\mathbf{0.77}$ & $0.78$ & $0.57$    \\
LeNet & $\dimph^{\rho_S}$ & $\mathbf{0.80}$ & $\mathbf{0.80}$ & $\mathbf{0.77}$ & $\mathbf{0.79}$ & $\mathbf{0.62}$ \\
\bottomrule
\end{tabular}
\end{sc}
\end{small}
\end{center}
\vskip -0.1in
\end{table}

In Figures \ref{fig:mnist_5_360_loss} and \ref{fig:mnist_7_360_loss} we plot the values of $\dimph^{\rho_S}$ against the actual loss gap (computed based on the cross entropy loss). While this has probably little practical interest compared to the plots shown in the main part of the paper, it highlights the fact that the correlation is indeed still there. As before, we note that low batch sizes and high learning rates yields better results, but that the correlation is very good for middle range values of those hyperparameters. As in the regression experiment, we observe on figure \ref{fig:mnist_5_360_loss} and \ref{fig:mnist_7_360_loss} that a bigger network gives better empirical correlation between the data-dependent dimension and the generalization error. Another interesting observation is that there seems to be more noise in the coefficients with respect to the loss gap than with respect to the accuracy gap. In most all experiments, again, the proposed dimension in close or better than the one proposed in \cite{birdal_intrinsic_2021}.

\begin{figure}[!h]
   \centering  
   \includegraphics[width=0.7\linewidth]{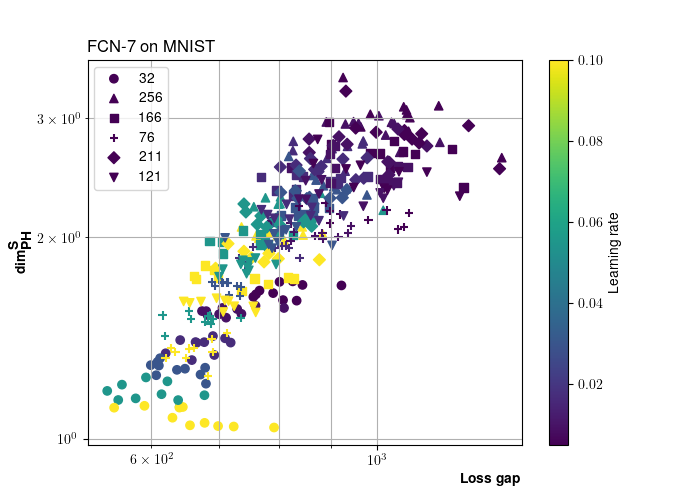}
    \caption{Plots of $\dimph^{\rho_S}$ against the loss gap (as opposed to the accuracy gap) for a FCN-$7$ trained on MNIST dataset.}
    \label{fig:mnist_7_360_loss}
\end{figure}

\begin{figure}[!h]
   \centering 
   \includegraphics[width=0.7\linewidth]{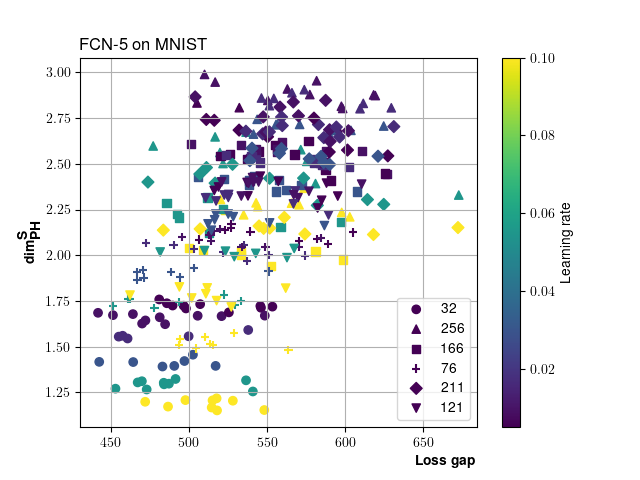}
    \caption{Plots of $\dimph^{\rho_S}$ against the loss gap (as opposed to the accuracy gap) for a FCN-$5$ trained on MNIST dataset.}
    \label{fig:mnist_5_360_loss}
\end{figure}

On Table \ref{table:kendall_mnist_loss_appendix} we report the correlations coefficients $(\rho, \psi_{\text{lr}}, \psi_{\text{bs}}, \boldsymbol{{\Psi}}, \tau)$ between our data-dependent intrinsic dimension and the actual loss gap in the same classification experiments than in Figures \ref{fig:mnist_5_360_loss} and \ref{fig:mnist_7_360_loss}.

\begin{table}[!h]
\caption{Correlation coefficients on MNIST, with respect to loss gap}
\label{table:kendall_mnist_loss_appendix}
\vskip 0.15in
\begin{center}
\begin{small}
\begin{sc}
\begin{tabular}{@{} l l l l l l l @{}} 
\toprule
{Model} & {Dim.} & {$\rho$} & {$\psi_{\text{lr}}$} & {$\psi_{\text{bs}}$} & {$\boldsymbol{\Psi}$ } & {$\tau$} \\ 
\midrule
FCN-$5$ & $\dimph^{\text{eucl}}$ & $\mathbf{0.76 \pm 0.06}$ & $\mathbf{0.33 \pm 0.18}$ & $\mathbf{0.75 \pm 0.09}$ &  $\mathbf{0.54 \pm 0.11}$ & $\mathbf{0.58 \pm 0.05}$   \\
FCN-$5$ & $\dimph^{\rho_S}$ & $0.73 \pm 0.09$ & $0.30 \pm 0.20$ & $\mathbf{0.75 \pm 0.09}$ & $0.52 \pm 0.12$  & $0.57 \pm 0.07$ \\ 
\midrule
FCN-$7$ & $\dimph^{\text{eucl}}$ & $0.86 \pm 0.05$ & $0.77 \pm 0.12$ & $\mathbf{0.80 \pm 0.08}$ & $0.79 \pm 0.06$ & $0.69 \pm 0.06$    \\
FCN-$7$ & $\dimph^{\rho_S}$ & $\mathbf{0.90 \pm 0.03}$ & $\mathbf{0.80 \pm 0.10}$ & $0.79 \pm 0.06$ & $\mathbf{0.80 \pm 0.06}$  & $\mathbf{0.75 \pm 0.05}$ 
\\
\bottomrule
\end{tabular}
\end{sc}
\end{small}
\end{center}
\vskip -0.1in
\end{table}

Figures \ref{fig:mnist_alexnet_appendix} and \ref{fig:mnist_lenet_appendix}, as well as Table \ref{table:kendall_conv_mnist_appendix} show experimental results obtained by training Convolutional Neural Networks (CNN) on the MNIST dataset, namely AlexNet \cite{krizhevsky_imagenet_2017} and LeNet \cite{lecun_gradient-based_1998} networks. It further highlights the pertinence of our intrinsic dimension, which is well correlated with the accuracy gap in those experiments.

\begin{figure}[!h]
   \centering  
   \includegraphics[width=0.7\linewidth]{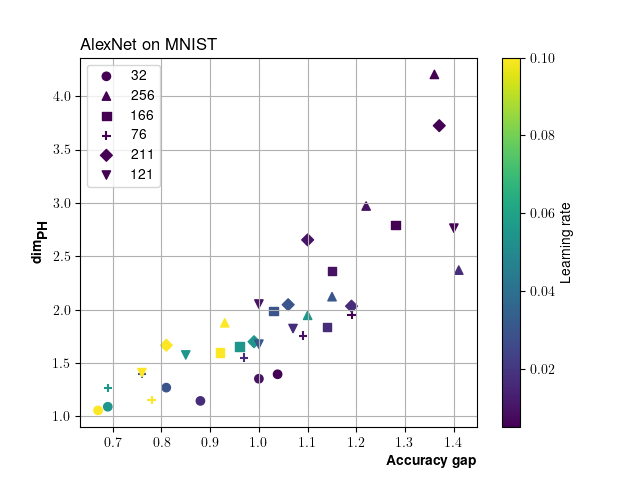}
    \caption{Plots of $\dimph^{\rho_S}$ against the accuracy gap for an AlexNet trained on MNIST dataset.}
    \label{fig:mnist_alexnet_appendix}
\end{figure}

\begin{figure}[!h]
   \centering  
   \includegraphics[width=0.7\linewidth]{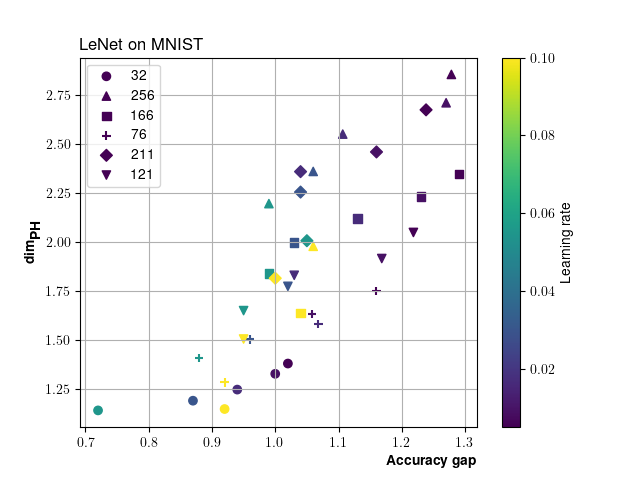}
    \caption{Plots of $\dimph^{\rho_S}$ against the accuracy gap for a LeNet trained on MNIST dataset.}
    \label{fig:mnist_lenet_appendix}
\end{figure}

\subsection{Evaluation of the computable part of the bounds}

Recent studies showed that, even when the overall scale of a generalization bound does not match the scale of the actual generalization error, the bound can still be useful (e.g., for hyperparameter selection) if it correlates well with the generalization error. 

Our study of the correlation between our data-dependent intrinsic dimension the generalization error is inspired by \citet{jiang_fantastic_2019}, who pointed out that correlation can be an efficient measure of the performance of different complexity measures. Moreover, in our study, we use the Granulated Kendall's coefficients, that \citet{jiang_fantastic_2019} introduced, to better capture the causal relationship in the correlation between the fractal dimension and the generalization.

That being said, one could argue that it would be better to plot the full bound and compare it to the generalization error. However, as we are aiming to compute the bounds in Theorems \ref{main_result_HP_bound_convering_MI} and \ref{main_result_hp_bound_with_coverings_stability}, this is a notoriously difficult, and often impossible task due to the presence of the mutual information (MI) terms. Hence the full computation of the bound is unfortunately not possible. Yet, we would like to underline that this has been the case for most information theoretic bounds, and fractal geometric bounds \citep{simsekli_hausdorff_2021, birdal_intrinsic_2021, hodgkinson_generalization_2022}.

As an intermediate solution towards this direction, we propose the following experimental setting, which will be added to the next version of the paper. Since our MI terms, especially $I_\infty(S,\mathcal{W}_{S,U})$ appearing in Theorem $3.8$, can be seen as similar as terms appearing in previous fractal geometric works \citep{simsekli_hausdorff_2021, birdal_intrinsic_2021, hodgkinson_generalization_2022}, we can aim at plotting the remaining terms of the bound and still provide a meaningful experiment.

Hence, in an attempt to provide further experimental results, 
we can compare the generalization error to the `computable' part of the bound, whose main term is of the form:
$$
\delta + \frac{B}{\sqrt{n} - 1} + \sqrt{2}B\sqrt{\frac{ \dim_{PH}(\mathcal{W}_{S,U})\log(1/\delta) + \log(\sqrt{n}/\eta)}{n}}.
$$
This expression can be approximated thanks to the persistent homology tools described in the paper. We will include these experiments in the paper, where we will compute the full bounds of the prior art in the same way (which will require estimating the Lipschitz constant).

One particular point of attention is that the loss functions we consider for the experiments are in practice not bounded. Despite this fact, to allow for the experimentation to take place, we set the value of the constant $B$ to the maximum loss reached on one data-point in the whole trajectory, over all experiments. For a fully connected networks trained on MNIST, we get Figures \ref{fig:mnist5_bound} and \ref{fig:mnist7_bound}.

\begin{figure}[!h]
   \centering  
   \includegraphics[width=0.7\linewidth]{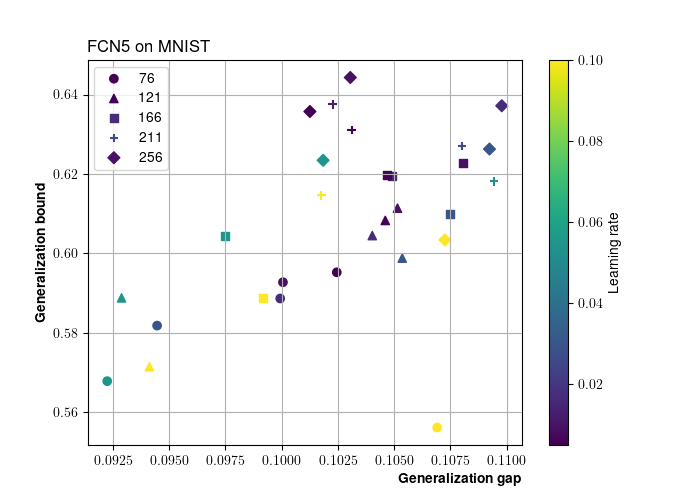}
    \caption{Plot of the generalization error against the computable part of our bounds for a $5$-layer fully-connected network trained on MNIST}
    \label{fig:mnist5_bound}
\end{figure}

\begin{figure}[!h]
   \centering  
   \includegraphics[width=0.7\linewidth]{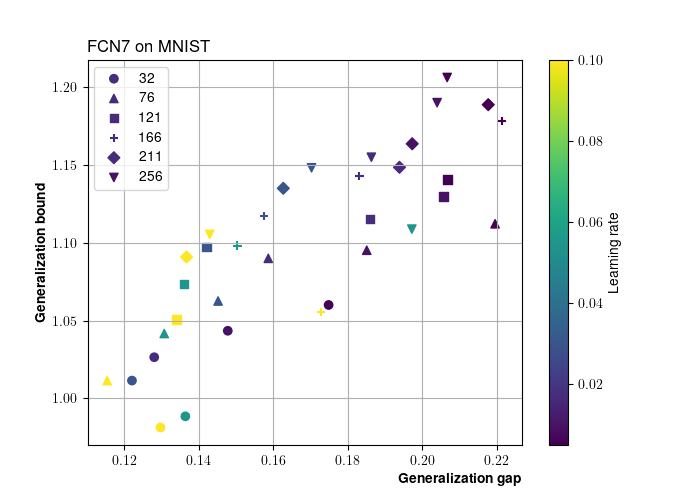}
    \caption{Plot of the generalization error against the computable part of our bounds for a $7$-layer fully-connected network trained on MNIST}
    \label{fig:mnist7_bound}
\end{figure}

The figures report the value of the above formula compared to the actual generalization error. The correlation with the generalization is quite well. However, an offset can be observed, corresponding to the scaling to what appears to be a rather small absolute constant. We believe that this is due to the fact that the way we estimate $B$ is a bad estimation of the potential sub-Gaussian character of the loss (indeed, the statement of Theorem $3.5$ can be easily extended for sub-Gaussian losses). As the bound scales linearly with this quantity, we see that we can easily get pretty close to the true generalization gap.

\subsection{Experiments with bigger models and datasets}
Most experiment presented in Sections \ref{section:experiments} and \ref{sec:additional_res_from_main_part} are made on relatively small datasets and/or neural network models. For the sake of completeness, we present here similar experiments computed with a Resnet-$18$ model on CIFAR$10$ and CIFAR$100$ datasets. 

Note that the main difficulty to perform such experiments is that the computation of $\dimph^{\rho_S}(\mathcal{W}_{S,U})$ requires the evaluation of the model on \emph{every} training data point, and the corresponding distance matrix. However, to be able to make this experiment in a reasonable amount of time, and according to our computational resources, we leveraged the ideas from Section \ref{section:experiments}, regarding the robustness analysis, and used only a subset of the dataset for the computation of $\dimph^{\rho_S}(\mathcal{W}_{S,U})$ (while the whole dataset is obviously used for training). Moreover, note that one advantage of the proposed data-dependent intrinsic dimension is that it requires much less memory to be computed than the one proposed in \citep{birdal_intrinsic_2021}. Indeed, to compute this last one, we would need to store all the weights of the network, for a few thousand iterations.

\textbf{Hyperparameters details} Both experiments (on CIFAR$10$ and CIFAR$100$) where realized on a $4\times4$ grid of hyperparameters with the batch size varying from $32$ to $256$ and the learning rate varying from $0.1$ to $0.001$. In both experiments, the persistent homology dimension has been computed using the last $2000$ iterates and $5\%$ of the dataset for the computation of the associated distance matrix. For the experiment on CIFAR$10$, the network was trained until $100\%$ accuracy before computing the persistent homology dimensions. For the CIFAR$100$ experiment, because we were not able to train the model until $100\%$ accuracy, we stopped the training after $100000$ iterations.

\begin{figure}[!h]
   \centering  
   \includegraphics[width=0.7\linewidth]{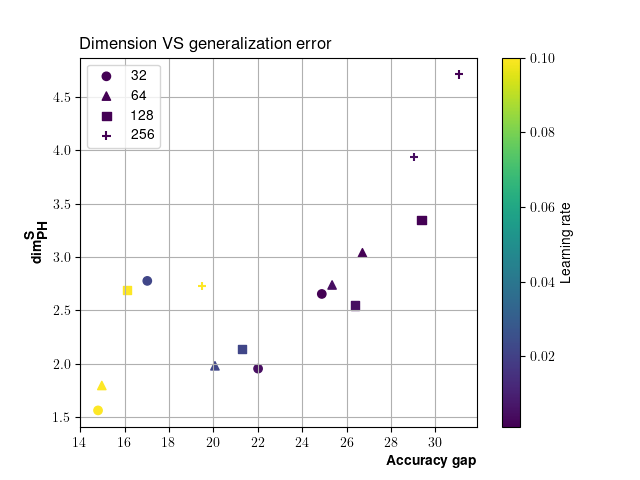}
    \caption{Plots of $\dimph^{\rho_S}$ against the accuracy gap for Resnet-$18$ network, trained on CIFAR$10$.}
    \label{fig:resnet18_cifar10}
\end{figure}

Figure \ref{fig:resnet18_cifar10} displays our results for the CIFAR$10$ experiments. As in previous experiments, we observe a very satisfying correlation for lower learning rates and high batch sizes. 

Results on the CIFAR$100$ dataset are shown on Figures \ref{fig:resnet18_cifar100} and \ref{fig:resnet18_cifar100_cnoverged_points}. Something interesting is observed here; as the network didn't reach, in our experiment, an accuracy close to $100\%$ for all hyperparameter settings, we observe two regimes regarding the correlation between the accuracy gap and the data-dependent persistent homology dimension:
\begin{itemize}
    \item For experiments achieving very good training accuracy, the correlation is excellent, as shown in Figure \ref{fig:resnet18_cifar100_cnoverged_points}.
    \item Experiments with less training accuracy look like `out of distribution ' experiments, this particular behavior is illustrated on Figure \ref{fig:resnet18_cifar100}.
\end{itemize}
This shows that the fractal behavior seems to be particularly pertinent in the `permanent regime' of training, i.e. when the distribution of the parameters reaches a stable distribution.

\begin{figure}[!h]
   \centering  
   \includegraphics[width=0.7\linewidth]{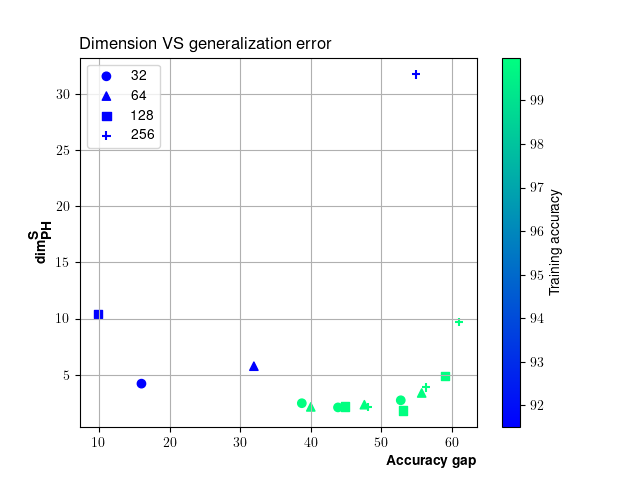}
    \caption{Plots of $\dimph^{\rho_S}$ against the accuracy gap for Resnet-$18$ network, trained on CIFAR$100$, the training accuracy is shown to highlight the importance of convergence for the correlation.}
    \label{fig:resnet18_cifar100}
\end{figure}

\begin{figure}[!h]
   \centering  
   \includegraphics[width=0.7\linewidth]{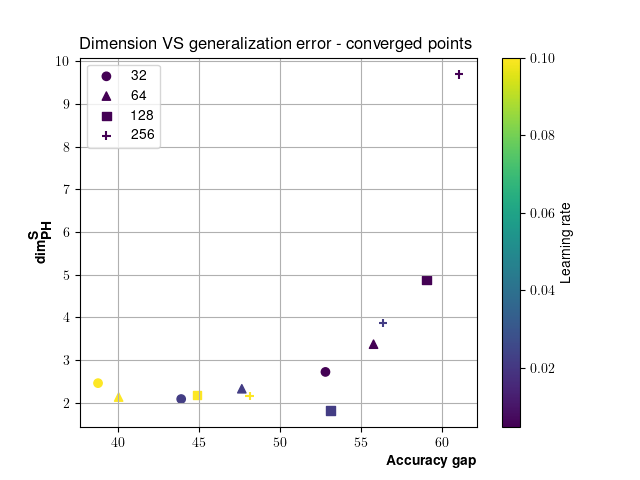}
    \caption{Plots of $\dimph^{\rho_S}$ against the accuracy gap for Resnet-$18$ network, trained on CIFAR$100$, displaying only points for which $98\%$ accuracy has been reached during training.}
    \label{fig:resnet18_cifar100_cnoverged_points}
\end{figure}

\end{document}